\documentclass{kais}
\pdfoutput=1
\usepackage{comment}     
\newcommand{\comments}[1]{}
\newcommand{\rem}[1]{}
\newcommand{\REM}[1]{}

\def\useHyperref{1}
\if\useHyperref1
	\usepackage[
				pdfpagemode=UseNone, 	
				pdfstartview=FitH,    
				]{hyperref}    
	
	\usepackage[all]{hypcap}
\fi

\usepackage[pdftex]{graphicx}
\graphicspath{{Images/jpg/120DPI/},{Images/pdf/},{Images/pdf/Expr811/},{Images/pdf/Expr816/},{Images/jpg/},{}}
\def\usePDF{1}
\if\usePDF1
	\DeclareGraphicsExtensions{.pdf,.jpg}
\else
	\DeclareGraphicsExtensions{.jpg}
\fi

\usepackage[caption=false,font=footnotesize]{subfig}

	\usepackage{wrapfig} 
	
    \usepackage[square,sort,comma,numbers]{natbib}

\usepackage{fixltx2e}

\usepackage{url} 

\usepackage[linesnumbered,algoruled,vlined]{algorithm2e} 

\usepackage{amsmath}

\usepackage{float}
\restylefloat{table}
\restylefloat{figure}

\def \CCS  {co-cluster}  
\def \CC   {{\CCS}ing}   
\def \UCCS {Co-cluster}  
\def \UCC  {{\UCCS}ing}  

\def \LC  {lagged {\CCS}}    
\def \ULC {Lagged {\CCS}}    
\def \LCP {{\LC}ing problem} 
\def \LCA {\textbf{LC}}      

\def \FLC  {fuzzy {\LC}}    
\def \UFLC {Fuzzy {\LC}}    
\def \FLCP {{\FLC}ing problem} 
\def \FLCANB {FLC}      
\def \FLCA {\textbf{FLC}}      

\def \NPC {NP-complete}      
\def \NPH {NP-hard}      

\def \Dp {p}  

\def \MI {\beta}  
\def \Ds {s}  
\def \DS {S}  

\def \MJ {\gamma}  

\def \MF  {F}  
\def \EF  {f}  

\def \SW {\varepsilon}  
\def \SWTF {\SW_{_{T,\MF}}}

\def \ExprArtificialProbArtifact                {{\EXPERIMENT} I}
\def \ExprArtificialDiscrSetSize                {{\EXPERIMENT} II}
\def \ExprArtificialDiscrProbabilities          {{\EXPERIMENT} III}
\def \ExprArtificialLowDiscrSizeMoreIterations  {{\EXPERIMENT} IV}
\def \ExprArtificialRunTimeHitRate              {{\EXPERIMENT} V}
\def \ExprArtificialErrFuzzEffect               {{\EXPERIMENT} VI}

\def \LOOPS {N}  

\def \ST {spatio-temporal}  
\def \FIGURE {Fig.} 
\def \SECTION {Section} 
\def \SUBSECTION {Subsection} 
\def \DEFINITION {Def.} 
\def \EQUATION {Equation} 
\def \LEMMA {Lemma} 
\def \ALGORITHM {Algorithm} 
\def \THEOREM {Theorem} 
\def \LINE {line} 
\def \LLINE {Line} 
\def \EXPERIMENT {Expt.} 
\def \FOOTNOTE {Footnote} 
\def \OBSERVATION {Observation} 
\def \REMARK {Remark} 
\def \COROLLARY {Corollary} 

\newtheorem{definition}{Definition}
\newtheorem{theorem}{Theorem}
\newtheorem{corollary}{Corollary}
\newtheorem{lemma}{Lemma}

\newtheorem{observation}{Observation}
\newtheorem{remark}{Remark}

\newcommand{\nin}{\noindent}

\def\QEDclosed{\hspace*{1em}\mbox{\rule[-0.1ex]{1.7ex}{1.7ex}}} 

\received{May 28, 2012}
\revised{Mar 25, 2014}
\accepted{May 03, 2014\\The final publication is available at Springer via http://dx.doi.org/10.1007/s10115-014-0758-7}
	
\pubyear{2014}
	\pagerange{\pageref{firstpage}--\pageref{lastpage}}
	\volume{xxx}

\begin{document}
	\label{firstpage}

\title{Co-clustering of Fuzzy Lagged Data}

	\author[E. Shaham et al]{
	Eran~Shaham$^1$, David~Sarne$^1$ and~Boaz~Ben-Moshe$^2$\\
	$^1$Department of Computer Science, Bar-Ilan University, Ramat-Gan, 52900 Israel\\
	$^2$Department of Computer Science, Ariel University, Ariel, 44837 Israel\\
	Email: erans@macs.biu.ac.il, sarned@macs.biu.ac.il, benmo@ariel.ac.il
	}	

\maketitle



\begin{abstract}
The paper focuses on mining patterns that are characterized by a  \emph{fuzzy lagged} relationship between the data objects forming them.
Such a regulatory mechanism is quite common in real-life settings.
It appears in a variety of fields:
finance, gene expression, neuroscience,
crowds and collective movements, are but a limited list of examples.
Mining such patterns not only helps in understanding the relationship between objects in the domain, but assists in forecasting their future behavior.
For most interesting variants of this problem, finding an optimal \FLC\ is an \NPC\ problem.
We present a polynomial-time Monte-Carlo approximation algorithm for mining {\FLC}s.
We prove that for any data matrix, the algorithm mines a \FLC\ with fixed probability,
which encompasses the optimal {\FLC}
by a maximum 2 ratio columns overhead and completely no rows overhead.
Moreover, the algorithm handles noise, anti-correlations, missing values and overlapping patterns.
The algorithm was extensively evaluated using both artificial and real-life datasets.
The results not only corroborate the ability of the algorithm to efficiently mine relevant and accurate {\FLC}s, 
but also illustrate the importance of including fuzziness in the lagged-pattern model.
\end{abstract} 



\section{Introduction} \label{sec:Introduction}

\def \FOne   {{F$_1$}}

A by-product of modern life is the ever growing trace of digital data;
these might be pictures uploaded to the web, 
cellular trajectories collected by mobile providers,  
or the earth's climate monitored by buoys, balloons and satellites. 
The feature common to such data is its temporal nature.
Mining these data can facilitate uncovering the hidden regulatory mechanisms governing the data objects.

\def \HEIGHT    {4.15cm} 
\def \WIDTH	    {6.7cm} 

Early mining techniques used the key concept of clustering to look for patterns formed by a subset of the objects over all attributes, or vice versa  \cite{jain1999data,berkhin2006survey}.
Following seminal work by Cheng and Church \cite{cheng2000biclustering} in the area of gene expression using microarray technology, substantial focus has been placed in recent years on \CC\ \cite{madeira2004bab,jiang2004cag}.
\UCC\ extends clustering by aiming to identify a {\em subset} of objects that exhibit similar behavior across a {\em subset} of attributes, or vice versa.
Very few \CC\ studies have considered the problem of mining patterns that have a {\em lagged} correlation between a subset of the objects over a subset of the attributes
\cite{shaham2011sc,yin2007mining,wang2010efficiently}.
For example, consider the problem of identifying a flock of pigeons from among a large collection of flight trajectories (that is, mining a coordinated movement of a subset of objects across a subset of time attributes) \cite{nagy2010hierarchical}.
The flock's spatial coordinated flight, where each member follows the leader with some {\em lag} (delay), is a lagged pattern comprising the flock members' tempo-spatial locations (trajectories).
The underlying assumption in these works is that the lagged correlation, if it exists, is fixed (i.e., with no noise whatsoever).

In real-life settings, however, lagged patterns are typically \textbf{noisy}.
For example, consider the flock's coordinated flight described above.
Overall the flock maintains a general lagged flight formation (each member follows the leader with some lag). Yet, a closer look will reveal that each member deviates from that lag to some extent (due to wind changes, threats, physical strength, etc).
The flock's flight pattern can, however, be captured by a \CCS\ comprising fuzzy lags. 
\begin{figure}
  \centering
  \includegraphics[clip=true,trim=34 438 32 39, height=\HEIGHT, width=\WIDTH]{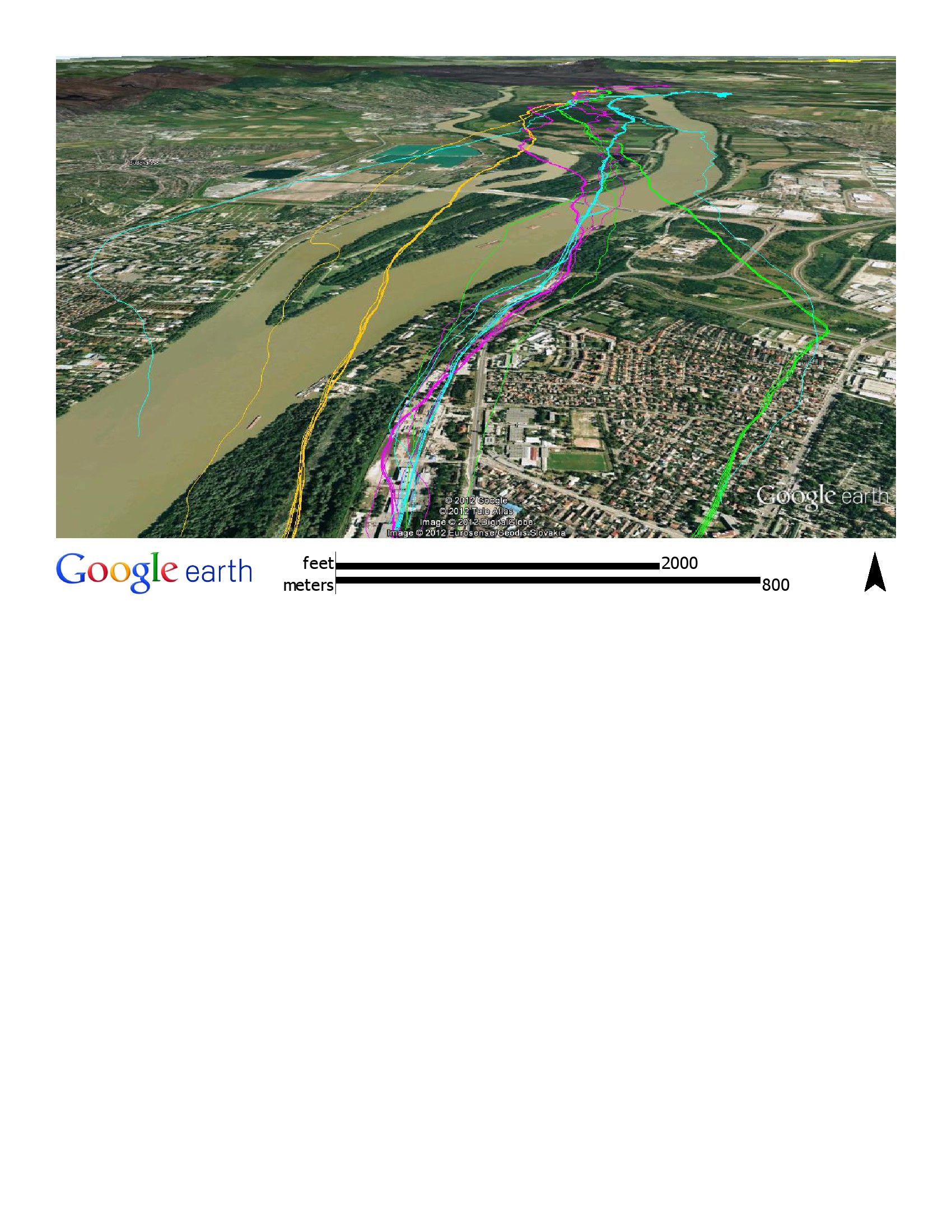} 
  \caption{Snapshot of the flight of flocks of pigeon.}
  \label{intro:examples:pigeon-hf-interleaving}
\end{figure}
\FIGURE~\ref{intro:examples:pigeon-hf-interleaving} presents such real-life flight trajectories, where each line represents a pigeon's trajectory (pigeons belonging to the same flock are denoted by the same color).
The presence of interleaving trajectories presents a serious challenge to mining algorithms (e.g., density-based algorithms \cite{ester1996density}), as well as to humans (see \SUBSECTION~\ref{subsec:Expr:birds}).
We denote a {\LC} which includes a fuzzy correlation between a subset of objects over a subset of lagged attributes as a \emph{\textbf{{\FLC}}}.
The problem, as later proved, is {\NPC} for most interesting cases.

Similar fuzzy lagged behavior can be observed during the mining of a group of people that coordinate their movements within a crowd (e.g., a group of terrorists trying to move from point A to point B). The group would maintain a general lagged formation where each member follows the leader with some lag.
However, due to obstacles, temporary loss of eye contact, and other difficulties, the group's members would probably be compelled to deviate from that fixed lag.
Additional motivation for studying {\FLC}s comes from the field of medicine within the context of disease relationships and causality. 
Given a dataset where an object is a disease, an attribute is an age, and an entry of the matrix is the number of occurrences of a disease in an age (the number of occurrences can be obtained from medical articles, hospital records, etc),
the causality of diseases would be captured by a lagged \CCS.
However, the lag is expected to be of a fuzzy nature due to change in medical treatment, difference in disease development, inaccuracy of the dataset, etc.
Mining such patterns can assist not only in the early detection of diseases, but also in providing better preventive treatment.

\def \HEIGHT    {4.4cm} 
\def \WIDTH	    {3.8cm} 

\begin{figure*}
  	\centering	
  	
  	\subfloat[Dataset]
	  {	
			\includegraphics[clip=true,trim=86 359 82 67, height=\HEIGHT]{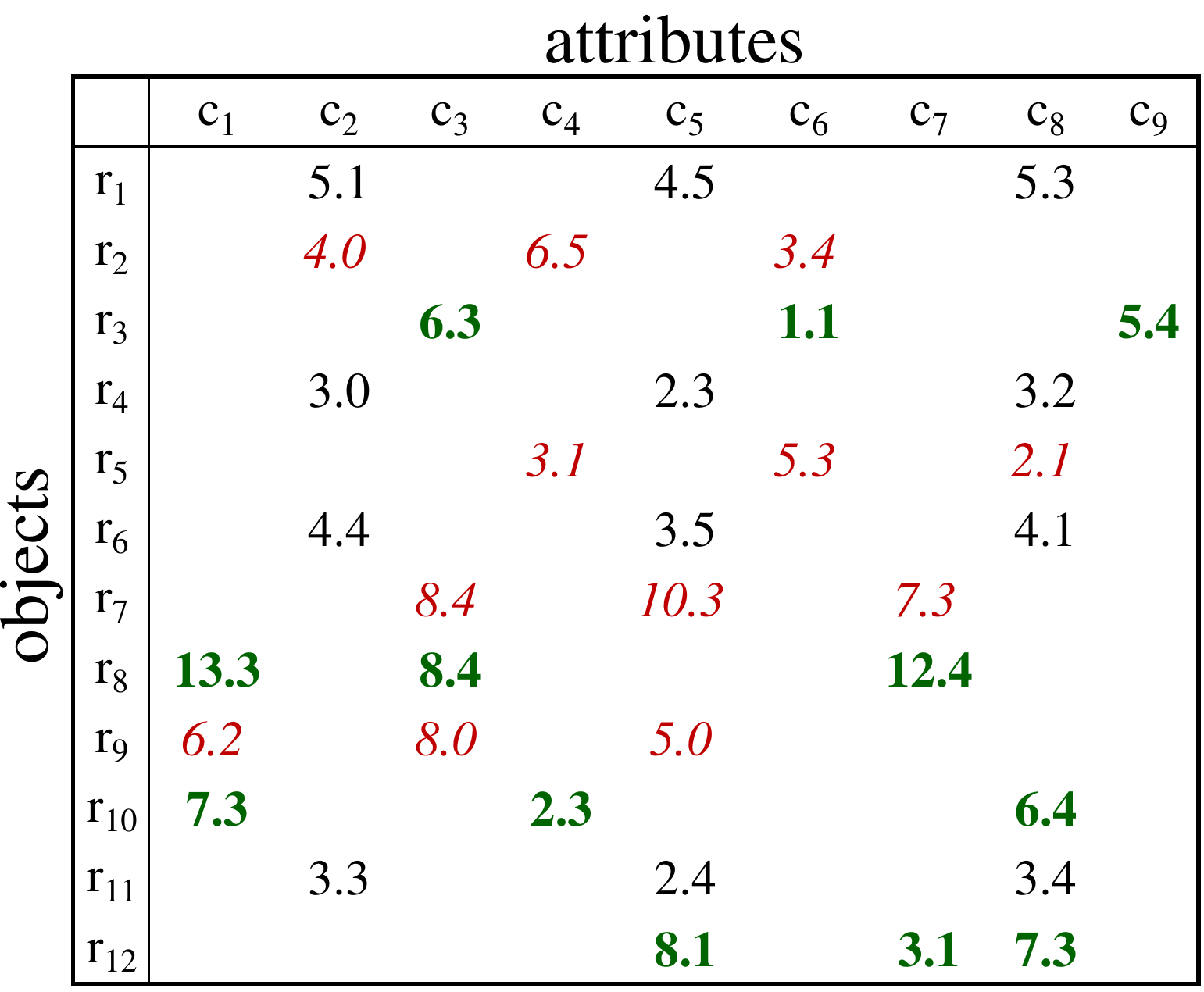}
			\label{intro:dataset}
	  }	
 	  \subfloat[Fuzzy lagged dataset]
	  {
			\includegraphics[clip=true,trim=78 376 96 68, height=\HEIGHT]{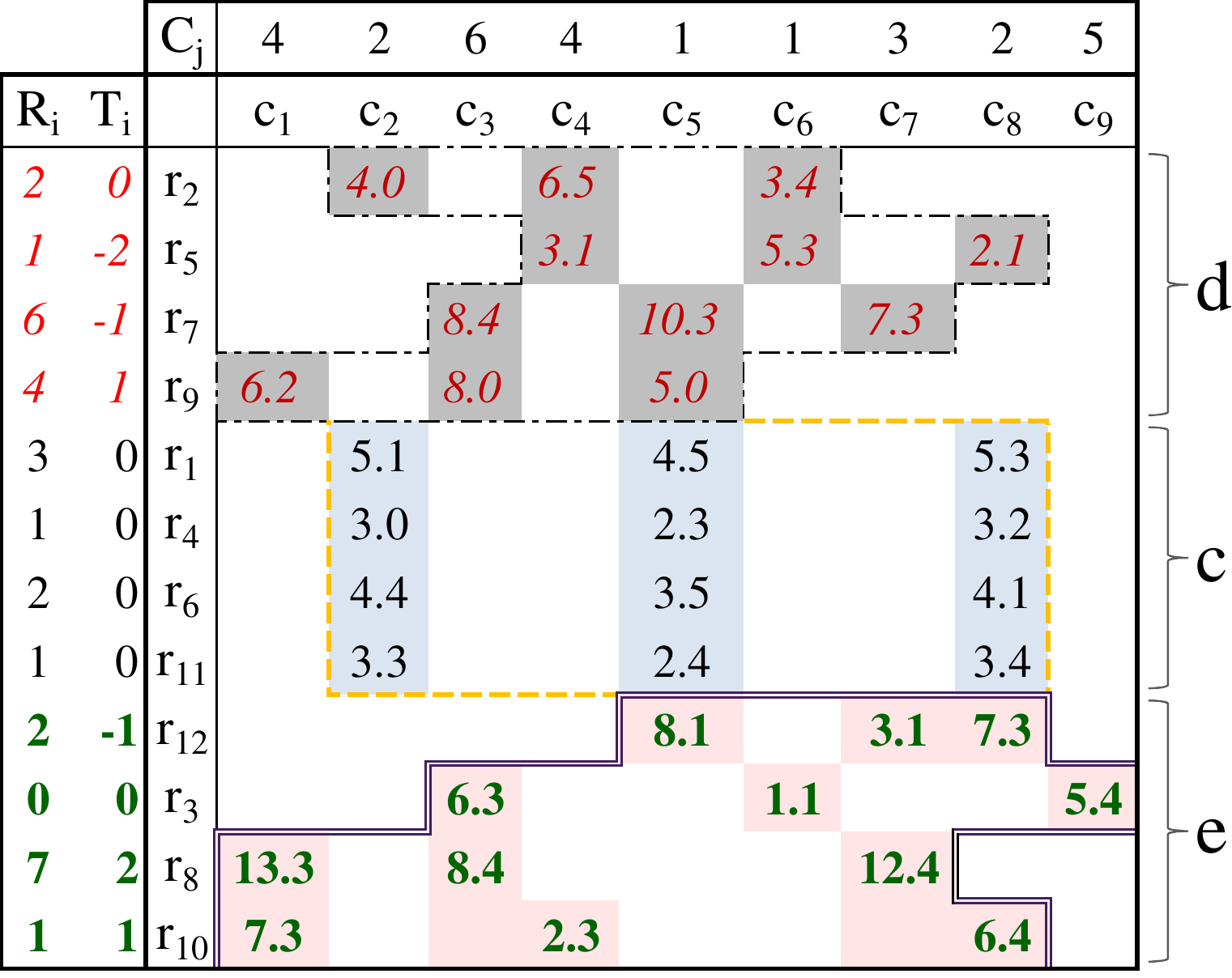}
			\label{intro:fuzzy-lagged-dataset}
	  }

	  \centering	
	  \subfloat[\UCCS]
	  {
			\includegraphics[clip=true,trim=52 75 62 85, width=\WIDTH]{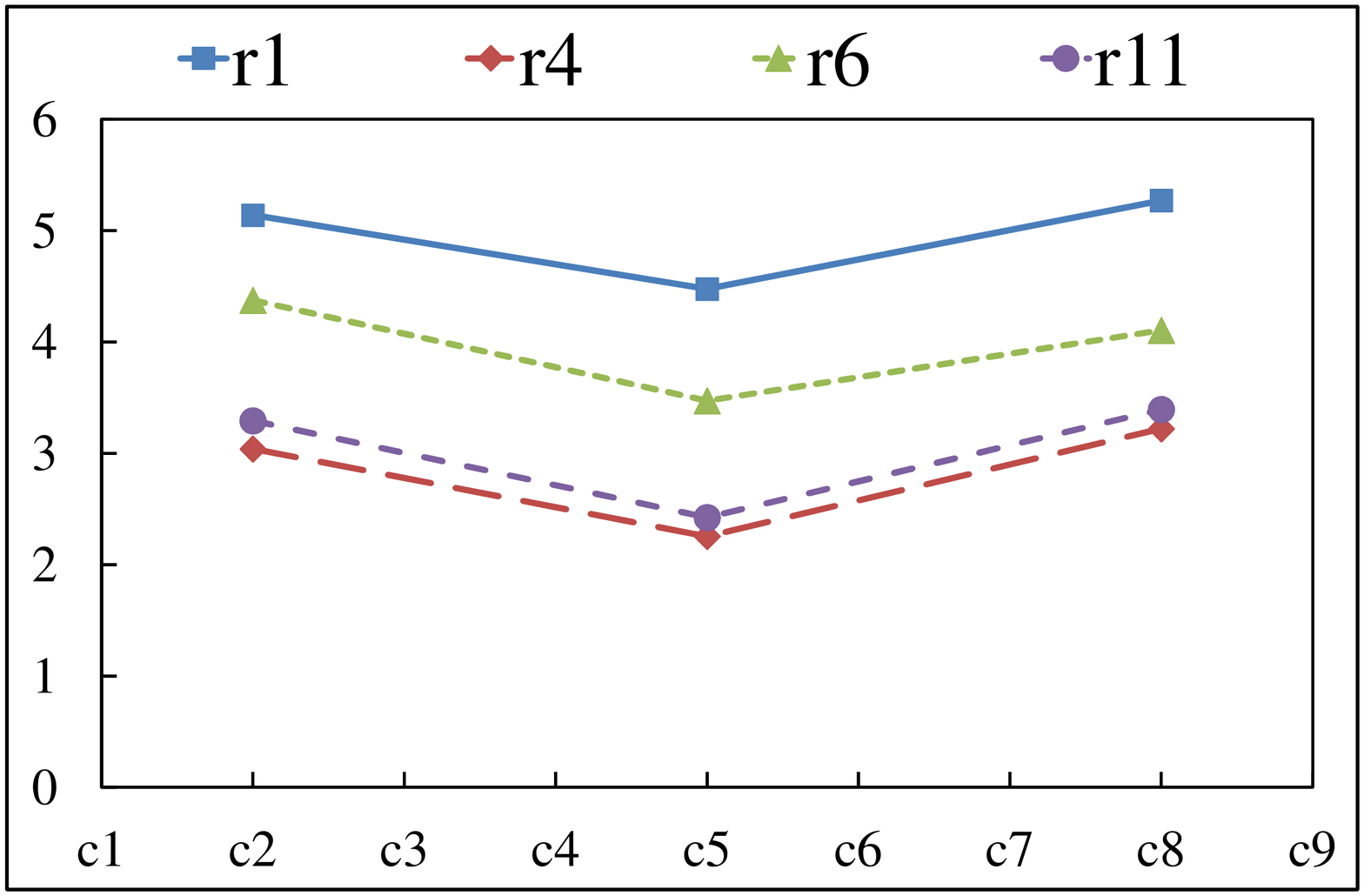}
			\label{intro:coCluster}
	  }
	  \subfloat[\ULC]
	  {
			\includegraphics[clip=true,trim=52 75 62 85, width=\WIDTH]{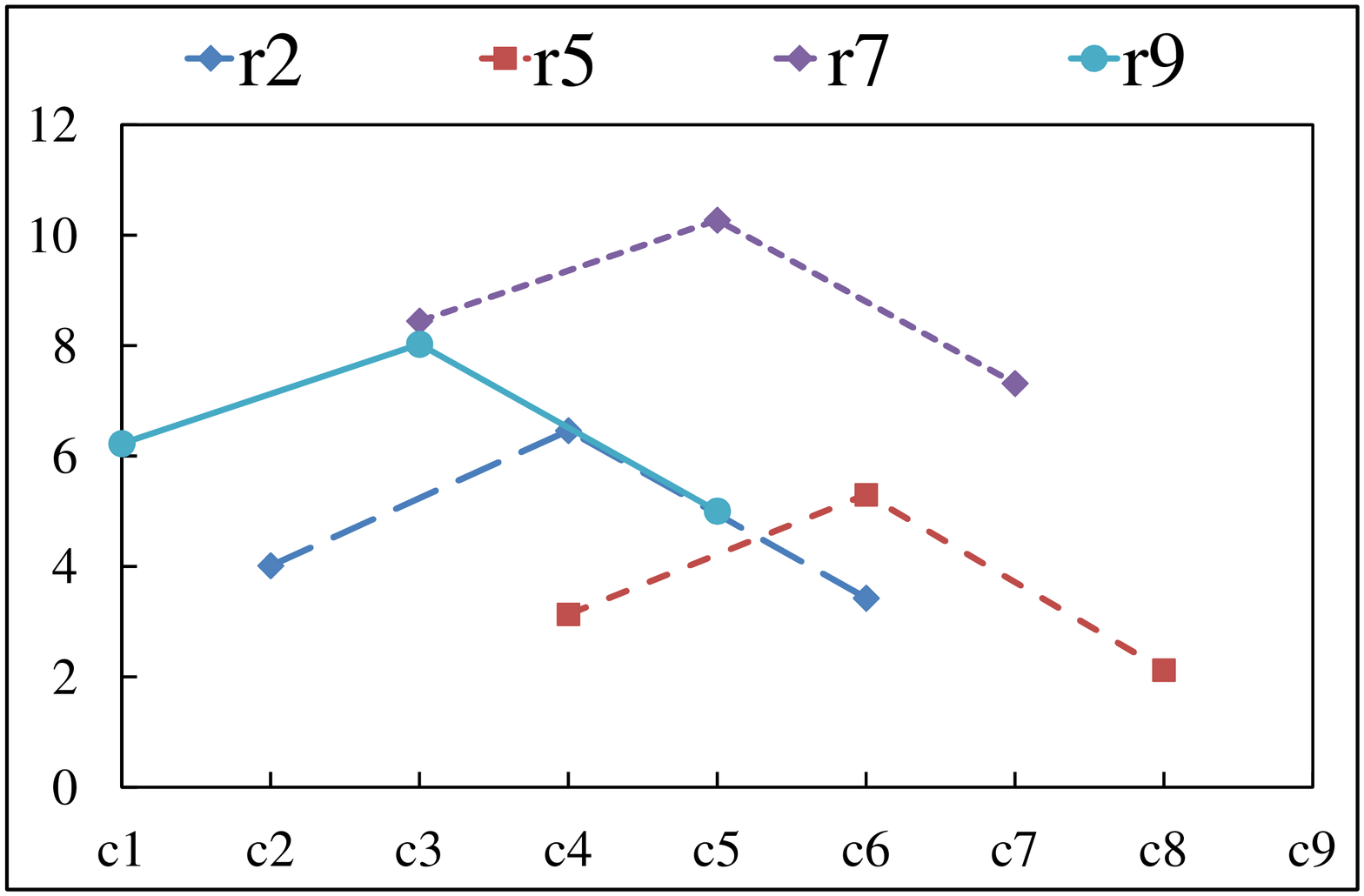}
			\label{intro:lagged-cluster}
	  }
	  \subfloat[\UFLC]
	  {
			\includegraphics[clip=true,trim=52 77 62 85, width=\WIDTH]{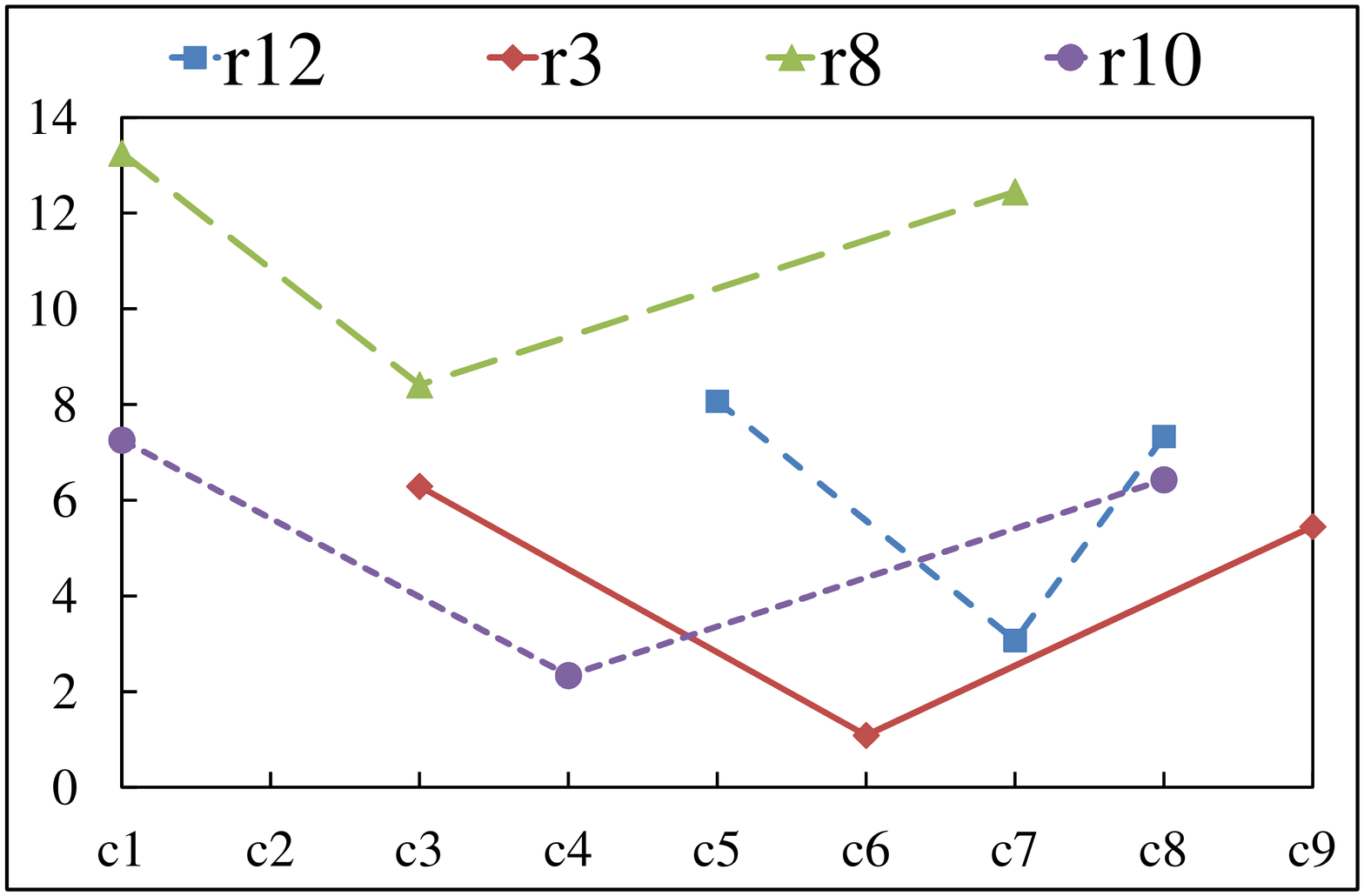}
			\label{intro:fuzzy-lagged-cluster}
	  }
	\caption{Example of a fuzzy lagged dataset (based on \cite{yin2007mining}).}
	\label{intro:examples}
\end{figure*}

\FIGURE~\ref{intro:examples} presents an example of a fuzzy lagged dataset and various clusters within it.
\FIGURE~\ref{intro:dataset} depicts an example of a matrix dataset (for simplicity, certain cells have been left blank). \FIGURE~\ref{intro:fuzzy-lagged-dataset} represents the same matrix after row permutation.
Three clusters emerge, as follows.
%
\FIGURE~\ref{intro:coCluster} (middle part of matrix \ref{intro:fuzzy-lagged-dataset}):
a {\em \CCS} with neither lag nor fuzziness.
The value of a cluster entry, $A_{i,j}$,
may deviates from being expressed as the sum of the column profile,
$R_i$ (=\{3, 1, 2, 1\} for row $i$=\{1, 4, 6, 11\}, respectively),
and the row profile,
$C_j$ (=\{2, 1, 2\} for column $j$=\{2, 5, 8\}, respectively),
by a maximum allowed error $\SW\leq0.5$.
That is: $|(R_i + C_j) - A_{i,j}| \leq \SW = 0.5$.\footnote{
Throughout the example we use the notations of $R_i$ and $C_j$ of the additive model
which are an alternative representation to the notations of $G_i$ and $H_j$ of the multiplicative model.
See more details in the formal
model representation that follows, and in particular the
definitions in {\EQUATION}s~\ref{model:multiplicative}--\ref{model:additive}.
}
Intuitively, the {\em column profile} indicates the regulation strength of the object, while the {\em row profile} indicates the regulatory intensity of the attribute.
For example,
the matrix entry of row $r_4$ and column $c_8$ is 3.2,
which deviates from the expected value of:
$R_i$+$C_j$=$R_4$+$C_8$=1+2=3, by an error of 0.2.

%
\FIGURE~\ref{intro:lagged-cluster} (upper part of matrix \ref{intro:fuzzy-lagged-dataset}) exemplifies a {\em \LC}, with no fuzziness.
Here, the value of a cluster entry, $A_{i,j}$,
may deviate from being expressed as $R_i + C_{j+T_i}$ by a maximum error of 0.5.
That is: $|(R_i + C_{j+T_i}) - A_{i,j}| \leq \SW = 0.5$.
For example,
the matrix entry of row $r_7$ and column $c_3$ is 8.4,
which deviate from the expected value of: ($T_7$=--1) $R_i$+$C_{j+T_i}$=$R_7$+$C_{3+T_7}$=$R_7$+$C_{3-1}$=$R_7$+$C_2$=6+2=8, by an error of 0.4.
%
\FIGURE~\ref{intro:fuzzy-lagged-cluster} (lower part of matrix \ref{intro:fuzzy-lagged-dataset}) exemplifies a {\em \FLC}.
Here, the value of a cluster entry, $A_{i,j}$,
not only may vertically deviate from being expressed as $R_i+C_{j+T_i}$ by a maximum error of 0.5,
but also may horizontally deviate from $C_{j+T_i}$ by a maximum fuzziness, $\MF$, of two.
That is: $|(R_i + C_{j+T_i+\EF_{i,j}}) - A_{i,j}| \leq \SW = 0.5$, for some $\max_{i,j}\{\EF_{i,j}\}\leq 2$.
For example,
the matrix entry of row $r_3$ has a {\em zero} lag ($T_3$=0) and {\em zero} fuzziness over the columns $c_3, c_6$ and $c_9$, i.e., $\EF_{3,j} = \{0, 0, 0\}$ (to ease readability, \FIGURE~\ref{intro:examples} does not present the fuzziness values).
\textbf{Relative} to $r_3$, object $r_8$ has a lag of $T_8$=2
and fuzziness of $\EF_{8,j} = \{0, 1, 0\}$ (relative to $r_3$ columns);
object $r_{10}$ has a lag of $T_{10}$=1 and fuzziness of $\EF_{10,j} = \{1, 1, 0\}$,
and object $r_{12}$ has a lag of $T_{12}$=--1 and fuzziness of $\EF_{12,j} = \{-1, 0, 2\}$.
For example, the matrix entry of row $r_{12}$ and column $c_8$ is 7.3,
which deviate from the expected value of:
($T_{12}$=$-1$, $\EF_{12,c_8}$=2)  
$R_i$+$C_{j+T_i+\EF_{i,j}}$=$R_{12}$+$C_{8+T_{12}+\EF_{{12},8}}$=$R_{12}$+$C_{8-1+2}$=$R_{12}$+$C_9$=2+5=7, by an error of 0.3.

The main contribution of the paper is in introducing a polynomial time approximation algorithm for mining \textit{{\FLC}s}, hereafter denoted as the \textbf{\FLCA} algorithm.
To the best of our knowledge, this is the first attempt to develop such an algorithm.
The input of the \FLCA\ algorithm is a real number matrix (where rows represent objects and columns represent attributes), a maximum error value and a maximum fuzziness degree.
%
The algorithm uses a Monte-Carlo strategy to guarantee, with fixed probability, the mining of a {\FLC} which encompasses the optimal {\FLC} by a maximum 2 ratio columns overhead and completely no rows overhead.
This guarantee holds for any monotonically increasing objective function defined over the cluster dimensions.
Many of the inherent shortcomings common to non-fuzzy, non-lagged data
\cite{tanay2005bas,berkhin2006survey}
are handled by the \FLCA\ algorithm, including noise (due to human or machine inaccuracies); missing values (e.g., equipment malfunction); anti-correlations (down-regulation, to adopt gene expression terminology) and overlapping patterns.
%
The algorithm and its properties were extensively evaluated using both artificial and real-life datasets.
The results not only corroborate the algorithm's ability to efficiently mine relevant and accurate {\FLC}s,
but also illustrate the importance of including fuzziness in the lagged-pattern model.
With this inclusion, a significant improvement is achieved in both \emph{coverage} and \emph{\FOne} measures in comparison to using the regular (non-fuzzy) lagged co-clustering model.
Moreover, the {\FLCA} algorithm presented classification capabilities which were superior to the ones presented by both the non-fuzzy lagged model and those of human subjects.

The remainder of the paper is organized as follows.
\SECTION~\ref{sec:Model} formally introduces the model
and shows that most interesting variants of the problem are \NPC.
In \SECTION~\ref{sec:Algorithm} we present the algorithm followed by a run-time analysis,
proof of the probabilistic guarantee to efficiently mine relevant {\FLC}s and extensions to the algorithm.
\SECTION~\ref{sec:Experiments} presents the experiments that were conducted and their results.
In \SECTION~\ref{sec:Related Work} we review related work.
We conclude with a discussion and suggested directions for future research in \SECTION~\ref{sec:Conclusions}.


\section{Model} \label{sec:Model}

\def \RvsC {{\psi}}
\def \TF {\mu} 


A \LC\ of a real number matrix is a tuple $(I,T,J)$,
representing a submatrix determined by
a subset of the columns $J$ 
over a subset of the rows $I$ 
with their corresponding lags $T$ ($|T|$=$|I|$) \cite{wang2010efficiently} (see example in \FIGURE~\ref{intro:lagged-cluster}).
The {\FLC}ing model augments the \LC\ definition, enabling fuzziness in the lagged pattern.
\begin{definition} \label{model:flc}
  A \FLC\ of an $m \times n$ real number matrix $X$
  is a tuple $(I,T,J,\MF)$, where
  $J$ is a subset of the columns, 
  $I$ is a subset of the rows 
  with their corresponding lags $T$, 
  aligned to some fuzzy lagged mechanism
  by a maximal fuzziness degree of $\MF$
  (see example in \FIGURE~\ref{intro:fuzzy-lagged-cluster}).
  The fuzziness reflects the ability of a column to deviate from its lagged location, by a maximum of $\MF$ columns.
\end{definition}
%

%
A fuzzy lagged regulatory mechanism holds if
for all $j\in J$,
each pair of rows $i_1,i_2 \in I$, their corresponding lags $T_{i_1},T_{i_2}$ and fuzziness $\EF_{i_1,j},\EF_{i_2,j}$,
the proportion between the matrix entries is some constant, $C_{i_1,i_2}$, dependent only on the rows $i_1,i_2$ and independent of the columns $J$:\footnote{
Based on the standard \CC\ model definition, according to which
$\forall j \in J$, $X_{i_1,j}/X_{i_2, j}=C_{i_1,i_2}$ \cite{melkman2004sc,cheng2000biclustering}
and the {\LC}ing model definition, according to which
$\forall j \in J$, $X_{i_1,j+T_{i_1}}/X_{i_2, j+T_{i_2}}=C_{i_1,i_2}$ \cite{shaham2011sc,wang2010efficiently}.
}
$$X_{i_1, j+T_{i_1}+\EF_{i_1,j}}/X_{i_2, j+T_{i_2}+\EF_{i_2,j}}=C_{i_1,i_2}.$$ 
Let
$G_i$ indicate the regulation strength of object $i$;
$T_i$ indicate the influencing-lag of object $i$;
$H_j$ indicate the regulatory intensity of attribute $j$;
and $\EF_{i,j}$ indicate the fuzzy alignment of object $i$ to attribute $j$.
Thus, the submatrix elements of a {\FLC} should comply with the relation:
$X_{i,j} \approx G_i H_{j+T_i+\EF_{i,j}}$ for all $i\in I$ and $j\in J$.
We use the non-{\FLC}ing \textit{relative error} criteria \cite{wang2010efficiently,shaham2011sc}
to express the deviation of $X_{i,j}$ from the approximation of $G_i H_{j+T_i+\EF_{i,j}}$.
Thus, our aim is to mine large submatrices which follow a fuzzy lagged regulatory mechanism, with a relative error below a pre-defined threshold:
\begin{equation}
  \frac{1}{\eta} \leq \frac{G_i H_{j+T_i+\EF_{i,j}}}{X_{i,j}} \leq {\eta}, \ \forall \ i \in I,\ j \in J.
  \label{model:multiplicative}
\end{equation}

To ease analysis, we move from a multiplicative model to an additive model.
We do so by applying a logarithm transformation, setting
$A_{i,j}$ = $\log(X_{i,j})$,
$R_i$ = $\log(G_i)$,
$C_{j+T_i+\EF_{i,j}}$ = $\log(H_{j+T_i+\EF_{i,j}})$ 
and $\SW$ = $\log(\eta)$.
Therefore, our problem turns into finding $R_i$, $T_i$, $C_j$ and $\EF_{i,j}$, such that for all $i,j$:\footnote{\label{model:inverse-correlations}
	For an \textbf{\textit{anti}} fuzzy lagged correlations, i.e.,
	$X_{i,j} \approx G_i / H_{j+T_i+\EF_{i,j}}$, one should apply:
    \newline
	$-\SW \leq R_i - C_{j+T_i+\EF_{i,j}} - A_{i,j} \leq \SW.$  \label{model:additive-inverse}
}
\begin{equation}
    -\SW \leq R_i + C_{j+T_i+\EF_{i,j}} - A_{i,j} \leq \SW.
    \label{model:additive}
\end{equation}

The optimality of a submatrix depends on the objective function $\TF(I,J)$ being used.
Examples of such common functions are:
area: $\TF(I,J) = |I| \cdot |J|$;
perimeter: $\TF(I,J)=|I|+|J|$;
and $\TF(I,J)=|I| / \RvsC^{|J|}, 0 < \RvsC < 1$ \cite{procopiuc2002mca,melkman2004sc}, which favors the inclusion of one column over the exclusion of a relatively large amount of rows. Such preferment of columns over rows appears mostly in biologically-oriented datasets where $m \gg n $ \cite{jiang2004cag}.
Nevertheless, for many fuzzy lagged datasets, assumptions relating to the number of rows vs. the number of columns is usually futile, i.e., a temporal dataset will usually contain thousands of time readings or, in an on-line version, an infinite stream of columns.
Consequently, we allow the use of any monotonically growing objective function $\TF(I,J)$.
Thus, our problem turns into mining an \textit{optimal size} submatrix with a relative error below some given threshold.

\begin{definition} \label{model:def-error}
The \emph{error} of a submatrix $A$, defined by
a subset $J$ of the columns, a subset $I$ of the rows and their corresponding lags $T$ is:
	\begin{equation}
	   \SWTF(I,J)=\min_{R,C} \max_{i\in I,j\in J} | R_i + C_{j+T_i+\EF_{i,j}} - A_{i,j} |.
	\end{equation}
\end{definition}
The error reflects the maximum deviation of a {\FLC}'s entry, from being expressed as $R_i + C_{j+T_i+\EF_{i,j}}$.

At this point, we have all that is required to formally define a \FLC.
As mining small clusters, e.g., $[2 \times 2]$, may not be of interest,
we further extend the model to enable the user to specify the desired minimum dimensions:
(1) minimum number of rows, expressed as a fraction of $m$, denoted $\MI$;
and (2) minimum number of columns, expressed as a fraction of $n$, denoted $\MJ$.

\begin{definition} \label{model:def-flc}
Let $A$ be a matrix of size $m \times  n$, $\MF\geq0$ and
$0 <  \MI, \MJ \leq 1$ constants independent of the matrix dimensions.
A {\FLC} of a matrix $A$ with an error $w \geq 0$ is a tuple $(I,T,J,\MF)$ with
$J$ a subset of the columns,
$I$ a subset of the rows with their corresponding lags $T$,
which satisfies the following:
\begin{itemize}
	\item 
  	\textbf{Size:}
  	   The number of rows is $2 \leq \MI m \leq |I|$
  	   and the number of columns is $2 \leq \MJ n \leq |J|$.

	\item
  	\textbf{Fuzziness:}
       $-\MF \leq \EF_{i,j} \leq \MF$, for all $i\in I$ and $j\in J$.

	\item
  	\textbf{Error:}
		   $\SWTF(I,J) \leq w$.
		   i.e., for all $i \in I$ and $j \in J$ there exists $R_i$, $T_i$ and $C_j$, such that
		   $|R_i+C_{j+T_i+\EF_{i,j}}-A_{i,j}|\leq w$.
		   $R_i,\ i\in I$ will be called a column profile,
		   $T_i,\ i\in I$ will be called a lagged column profile 
		   and $C_j,\ j\in J$ will be called a row profile.
	
\end{itemize}
\end{definition}

As a consequence,
lagging row $i$ by $T_i$ and shifting it by $R_i$,
will place each column $j \in J$, aligned with its fuzziness $\EF_{i,j}$,
within a maximal error of $w$ of the \textit{row profile}.
The specific case of $\EF_{i,j}=0$ for all $i\in I$ and $j\in J$, is equivalent to the non-{\FLC} definition given in the previous chapter.

\subsection{Hardness Results} \label{subsection:NP-completeness}

The complexity of the \FLCP\ depends
on the nature of the cluster being mined, which is reflected by the objective function $\TF$ being used.
Former literature has shown that many such non-fuzzy and non-lagged instances are
\NPC~\cite{lonardi2006fbr,cheng2000biclustering,shaham2011sc,peeters2003meb}.

\pagebreak[3]
\begin{observation} \label{NP-completeness:FLCP-LCP-CCC}
	Any hardness or inapproximability, resulting either from the non-fuzzy or the non-lagged problem,
	implies the same result for the fuzzy lagged problem.
\end{observation}
\begin{proof}
	The \FLCP\ extends the \LCP\ ($\EF_{i,j}$=0, $\forall i\in I, j\in J$),
	which in turn extends the \CC\ problem ($T_i$=0, $\forall i\in I$).
	Thus, any valid instance of the non-fuzzy or the non-lagged problems can be seen as an instance
	of the fuzzy lagged problem.
	By negation, a polynomial time algorithm for the \FLCP\
	would allow the \LCP\ or the \CC\ problem to be solved optimally in polynomial time -- contradiction.
\end{proof}

The following observation demonstrates a polynomial reduction between a fuzzy lagged instance and a non-fuzzy, non-lagged instance.
\begin{observation} \label{NP-completeness:convert-FLC-2-regular-matrix}
  Let $A$ be a fuzzy lagged matrix of
  size $[m \times n]$
  and for all ${i\in I}$, $| \{\EF_{i,j} : \forall j \in J,\ \EF_{i,j} \neq 0 \}| \leq \log(mn)$.
	The matrix $A$
  can be presented as a non-fuzzy, non-lagged matrix $A'$,
  of size $[2mn{(2\MF+1)}^{\log(mn)} \times (3n+2\MF)]$.	
\end{observation}	
\begin{proof}	
	We duplicate each row $i\in A$, to represent all possible lags and fuzziness for that row.
	Each row $i$ $(1\leq i\leq m)$ can have $2n$ possible lags $(-n \leq$ lag $\leq n)$
	and ${(2\MF+1)}^{\log(mn)}$ possible fuzziness
	(a maximum of $\log(mn)$ columns, each with a possible fuzziness
	assignment of: $-\MF \leq$ fuzziness $\leq \MF$),
	resulting in $(3n+2\MF)$ columns.
	Null entries resulting from such alignments,
	i.e., lag and fuzziness, are marked as missing values.
	The resulting non-fuzzy, non-lagged matrix $A'$, is therefore of size
	$[m(2n){(2\MF+1)}^{\log(mn)} \times (3n+2\MF)]={\cal O}((mn)^c)$
	for some $c={\cal O}(\log(\MF))$.		
	The result complies with a matrix of size
	$[2mn \times 3n]$ for the specific case of $\MF$=0 \cite{shaham2011sc}.
\end{proof}

\begin{corollary} \label{NP-completeness:proof}
	Let $A$ be a fuzzy lagged matrix.
	The problem of finding the largest square
	\FLC\ $(I,T,$ $J,\MF)$	$(|I|$=$|J|)$ in $A$
	is \NPC.	
\end{corollary}
\begin{proof}		
	Following \OBSERVATION~\ref{NP-completeness:FLCP-LCP-CCC}, the {\FLCP} is \NPH.
	Yet, verifying a submatrix of $A$ to be a \FLC\ can be done in polynomial time by examining whether each
	entry holds the inequality of
	$-\SW \leq R_i + C_{j+T_i+\EF_{i,j}} - A_{i,j} \leq \SW$.	
	Therefore the problem is \NPC.
\end{proof}

The following \NPC\ approximations are worth mentioning \cite{shaham2011sc}:
approximating the size of the largest combinatorial square \CCS\ with an approximation factor of $n^{1-\epsilon}$;
approximating the size of the minimal sequential cluster-set for the {\CC} problem within a constant factor (Max-SNP-Hard);
and, approximating the minimal set of combinatorial squares (\CCS\ set)	with an approximation factor of $n^{1-\epsilon}$.


\section{The \textbf{\FLCA} Algorithm} \label{sec:Algorithm}

We now present the {\FLCA} algorithm.
This section also includes a proof for the algorithm's guarantee to mine with fixed probability, in a polynomial number of iterations, a {\FLC} that encompasses an optimal {\FLC}.
In addition, we supply a run-time analysis and several extensions.


\subsection{The Algorithm} \label{subsection:algorithm:Algorithm}

The input of the algorithm is:
a matrix $A$ of real numbers;
a maximum allowed error value $w$;
a maximum allowed fuzziness degree $\MF$;
a minimum fraction of the rows $\MI$;
and, a minimum fraction of the columns $\MJ$.
The algorithm itself uses a projected clustering approach.
This common technique for mining {\CCS}s \cite{lonardi2006fbr,procopiuc2002mca}
uses iterative random projection (i.e., a Monte-Carlo strategy) to obtain the cluster's seed.
It later grows the seed into a cluster.
The output using this method is guaranteed, with \emph{fixed probability},
to contain {\FLC}s that comply with the specified $\MI$, $\MJ$, $\MF$, and encompass the \emph{optimal} \FLC.
Each mined cluster precisely obtains the rows and lags of the optimal cluster,
with a maximum 2 ratio of its columns (i.e., a maximum addition of $J$ columns) and a maximum 2 ratio of its error.

\ALGORITHM~\ref{algorithm:FLC} presents the \textbf{\FLCA} algorithm.
Generally, the algorithm can be divided into four stages, as follows.
(1) Seeding ({\LINE}s~\ref{algorithm:AlgDiscRow}-\ref{algorithm:AlgDiscColSet}):
a random selection of a row and a set of columns to serve as seeds.
(2) Addition of rows ({\LINE}s~\ref{algorithm:AlgRowPhaseStart}-\ref{algorithm:AlgRowPhaseEnd}):
we search for rows that reside within an error $w$ of the row and column profiles (see \DEFINITION~\ref{model:def-flc}).
Unfortunately, these profiles are unknown.
It may happen that the seed lies within the edge of the cluster.
In such cases, rows situated on the other edge of the cluster would be within an error of $2w$.
%
A naive exhaustive search is computationally not feasible, as there is an exponential number of combinations.
To reduce this complexity, we use a sliding window technique. This technique enables a polynomial complexity.
The window slides on the \emph{sorted set} of events: $(A_{i,j+f} - A_{\Dp,s})$, where $i\in m$, $j\in n$, $|f| \leq \MF$ and $s\in \DS$.
In order to achieve an error of $2w$,
we set the width of the sliding window to $4w$, which results in:
$\SWTF(\{i,p\}, J) = (\max_{j \in J} (A_{i,j+f} - A_{\Dp,s}) - \min_{j \in J} (A_{i,j+f} - A_{\Dp,s}))/2 = (4w)/2 = 2w$ (see \REMARK~\ref{proofs:theorem:remark-calc-profile}, below).
(3) Addition of columns ({\LINE}s~\ref{algorithm:AlgColPhaseStart}-\ref{algorithm:AlgColPhaseEnd}):
this is similar to the previous stage, but accumulating only columns that comply with the accumulated rows.
(4) Polynomial repetition of the above steps (\LINE~\ref{algorithm:AlgLoop}) providing a guarantee to mine an encompassed optimal \FLC.

\IncMargin{1em} 
\begin{algorithm*}
\small  
\caption{\textbf{\FLCA} algorithm} \label{algorithm:FLC}
\DontPrintSemicolon
  %
	
	%
	\SetKwBlock{StartLoop}      {\textbf{loop} $\LOOPS$ times}{}{}
	\SetKwBlock{SlideWindowRow} {slide a $4w$ width window on $\{e_{s,j,f} \:|\: e_{s,j,f} = A_{i,j+f} -  A_{\Dp,s},\: \forall j\in n,\: \forall s \in \DS,\: -\MF \leq f \leq \MF \}$  \label{algorithm:AlgRowPhaseSlidingWindow} }{}{}
  \SetKwBlock{SlideWindowCol} {slide a $4w$ width window on $\{e_i\:|\: e_i = A_{i,j+T_i+f} - A_{i,\Ds+T_i},\: \forall i\in I,\: -\MF \leq f \leq \MF  \}$  \label{algorithm:AlgColPhaseSlidingWindow} }{}{}
	\SetKwInput{Initialization}{Initialization}
	\KwIn{$A$, an $m \times n$ matrix of real numbers;
		   	$w$, the maximum acceptable error;
		   	$\MF$ the maximum degree of fuzziness;
		   	$\MI$, the minimum fraction of rows;
		   	and $\MJ$, the minimum fraction of columns.
		   }
		
	\KwOut{A collection of {\FLC}s $(I,T,J,\MF)$ whose error does not exceed $2w$.}

	\Initialization{Setting $\LOOPS$ and $|\DS|$ is thoroughly discussed in the following section.}

	\BlankLine
	
	\StartLoop { \label{algorithm:AlgLoop}

	  \BlankLine
	  // {Initialization Phase} 
	
		randomly choose a discriminating row $\Dp:\ 1\leq \Dp \leq m$ \; \label{algorithm:AlgDiscRow}
		randomly choose a discriminating set of columns $\DS:\ \DS \subseteq n$ \; \label{algorithm:AlgDiscColSet}
		\BlankLine
		$I \leftarrow \{\Dp\}$ \;
		$J \leftarrow \DS$ \;

		\BlankLine
		// \textit{Row Addition Phase} 
	
		\ForEach {row $i: \ 1 \leq i \leq m$} {		 	\label{algorithm:AlgRowPhaseStart}	
			\SlideWindowRow { 																	
				\If (\\ // \textit{found a common lag $t$ for all $s\in\!S$})
				{$(\forall s\in S\ \exists t,\ \!t+\!s\!=\!j\!+\!f \;  \wedge \;  \exists e_{s,j,f}\!\in\!window)$} {				 
				\label{algorithm:AlgRowPhaseCheck}
					add $(i,t)$ to $(I,T)$ \; 								\label{algorithm:AlgRowPhaseEnd} 
				}
			}
		} 																						 

		\BlankLine
		// \textit{Column Addition Phase}  
		
	  randomly choose a discriminating column $\Ds \in \DS$\;	\label{algorithm:AlgColPhaseStart} \label{algorithm:AlgDiscCol}

		\ForEach {column $j: \ 1 \leq j \leq n$} { 							
			\SlideWindowCol {
				\If {$(\forall i\!\in\!I \ \exists e_i\!\in\!window)$} { 	\label{algorithm:AlgColPhaseCheck}
					add $j$ to $J$ \; 																 \label{algorithm:AlgColPhaseEnd}
				}
			}
		} 																			

		\BlankLine
		// \textit{Validation of Dimensions} 
		
	 	\If {$|I| < \MI m$ or $|J| < \MJ n$} {
	 		discard $(I,T,J,\MF)$ \;
	 	}
	}
	\textbf{return} a collection of valid $(I,T,J,\MF)$ \;

\end{algorithm*}
\DecMargin{1em} 

The {\FLCA} algorithm augments the (\textit{non-fuzzy}) {\LC}ing {\LCA} miner \cite{shaham2011sc}
to mine \emph{fuzzy} {\LC}s.
In addition, its improved design suggests a substantial improvement in run-time
(in comparison to the {\LCA} miner) when mining non-fuzzy (i.e., $\MF$=0) clusters, from a run-time of
${\cal O}((m n)^{2-\log\MJ})$ \cite[\SECTION~6]{shaham2011sc}, to
${\cal O}((mn)^{1-\log\MJ} \: {\log}^2 (mn))$
(see \SUBSECTION~\ref{subsection:algorithm:Run-Time}).

The nature of the {\FLCA} algorithm suggests that it is sensitive to the error being set.
This key parameter needs to be carefully set in order to mine meaningful clusters.
Setting it too high might result in many artifact clusters, while setting it too low might preclude valid clusters.
To choose an appropriate error value, one can adopt any of the methods suggested
for the non-fuzzy {\LC}ing model \cite{shaham2011sc}.

Innately embedded within the algorithm are many desirable properties such as:
(1) the ability to handle noise by allowing the \FLC\ to deviate from the model (see \DEFINITION~\ref{model:def-flc}) by some pre-specified error. We accomplish this by using a window of width $4w$ as described above;
(2) the ability to mine overlapping clusters by utilizing the Monte-Carlo strategy, which grows independent seeds into clusters on each repetitive run;
(3) the ability to overcome missing values by calculating the coherence of a {\FLC} on the non-missing values of the submatrix \cite{yang2003ebe,melkman2004sc};
and (4) anti-correlation (see \FOOTNOTE~\ref{model:inverse-correlations}).
When both correlated and anti-correlated patterns may appear in the same {\FLC}, one can exercise one of the following solutions: (i) \emph{duplicate} each row of the input matrix to contain the anti-values of the row, i.e., for each row $i\in m$, add to the input matrix a new row containing the values of: $-A_{i,j},\ j\in n$;
or (ii) the algorithm's row addition phase ({\LINE}s~\ref{algorithm:AlgRowPhaseStart}-\ref{algorithm:AlgRowPhaseEnd}) should be modified into a two-pass sliding window.
The first pass (similar to the current \LINE~\ref{algorithm:AlgRowPhaseSlidingWindow}) is over events of the type: \\
  $\{e_{s,j,f} \:|\: e_{s,j,f} = A_{\Dp,s} - A_{i,j+f},\: \forall j\in n,\: \forall s \in \DS,\: -\MF \leq f \leq \MF \},$ \\
while the second pass is over events of the type: \\
  $\{e_{s,j,f} \:|\: e_{s,j,f} = A_{\Dp,s} + A_{i,j+f},\: \forall j\in n,\: \forall s \in \DS,\: -\MF \leq f \leq \MF \}.$ \\
The intuition behind the second pass is that an anti-correlated value is basically the value of $(-A_{i,j})$. Therefore, the first sliding window pass, which includes events of $A_{\Dp,s} - A_{i,j+f}$, should now be repeated over events of $A_{\Dp,s} - (-A_{i,j+f})$, which equals to $A_{\Dp,s} + A_{i,j+f}$.


\subsection{Run-time} \label{subsection:algorithm:Run-Time} 

The row addition phase ({\LINE}s~\ref{algorithm:AlgRowPhaseStart}-\ref{algorithm:AlgRowPhaseEnd})
handles, for each of the $m$ rows, a sliding window of ${\cal O}(n\cdot|\DS|\cdot\MF)$ events.
Therefore, its run-time is
${\cal O}(m \cdot {n|\DS|\MF} \cdot \log(n|\DS|\MF))$.
In the same manner, the column addition phase ({\LINE}s~\ref{algorithm:AlgColPhaseStart}-\ref{algorithm:AlgColPhaseEnd})  handles, for each of the $n$ columns, a sliding window of ${\cal O}(m\MF)$ events.
Therefore, its run-time is ${\cal O}(n \cdot {m\MF} \cdot \log(m\MF))$.
Thus, the inner for-loops run-time is:
${\cal O}(mn \: \log(mn) \: \log(n) \: \MF)$.

The total number of iterations is bounded by \THEOREM~\ref{proofs:theorem:rc}
to $\LOOPS = {\cal O} (1/ \MI \MJ^{|\DS|})$.
Thus, for the constants $\MI$ and $\MJ$
independent of the matrix dimensions (see \DEFINITION~\ref{model:def-flc}),
and the discriminating set
$|\DS| = {\cal O}(\log (mn))$ (see \THEOREM~\ref{proofs:theorem:d3}),
the {\FLCA}'s total run-time is \emph{polynomial} in the matrix size:
${\cal O}((mn)^{1-\log\MJ} \: {\log}^2 (mn) \: \MF)$
which in many cases can be seen more permissibly as: ${\cal O}((m n)^{2-\log\MJ})$.


\subsection{Sub-optimality of \textbf{\FLCA} Algorithm} \label{subsection:algorithm:Proofs}

\def \MJT {{\MJ '}}

\def \RiS  {R^{*}_{i}}
\def \TiS  {T^{*}_{i}}
\def \CjS  {C^{*}_{j}}
\def \RpS  {R^{*}_{\Dp}}
\def \FijS {\EF^{*}_{i,j}}

Next, we analyze the ability of the \FLCA\ algorithm to mine \textit{coherent} and \textit{relevant} {\FLC}s.
In particular we prove that the algorithm \textit{guarantees} to mine, with fixed probability, in a polynomial number of iterations,
a \FLC\ that encompasses an optimal \FLC.
The mined cluster will acquire the rows of the optimal cluster and their lags with a maximum 2 ratio of its columns.
Consequently, the mined cluster will have a maximum 2 ratio of the optimal cluster error.
%
We demonstrate this guarantee with experiments on both artificial and real-life datasets in \SECTION~\ref{sec:Experiments}.

Since the \FLCA\ algorithm augments the non-{\FLC}ing \LCA\ algorithm \cite{shaham2011sc},
its capabilities and theoretical analysis are deeply inspired by it.
The structure of the proof consists of two major stages.
The first stage is based on an important insight 
stating that a sufficient size for a discriminating set is logarithmic in the size of the set \cite{procopiuc2002mca,lonardi2006fbr}.
Following this result, we show that by taking any small random subset of columns of size
${\cal O}(\log (mn))$, we can discriminate an \textbf{optimal} {\FLC} with a probability of at least 0.5.
The second stage utilizes the previous result to mine,
in a polynomial number of iterations and with a probability of at least 0.5, clusters that encompass the \emph{optimal} {\FLC}.

The definition of a discriminating set for the fuzzy lagged model is given as follows.
\begin{definition} \label{proofs::discriminating}
 Let $(I,T,J,\MF)$ be a {\FLC}
 with an error $w$ and $\Dp \in I$.
 $\DS \subseteq J$ is a \emph{discriminating set} for $(I,T,J,\MF)$ with respect to $\Dp$ if it satisfies:
 \begin{enumerate}
  \item[1.]
   $\SWTF(\{i,\Dp\},\DS) \leq w$ for all $(i,t) \in (I,T)$.
  \item[2.]
   $\SWTF(\{i,\Dp\},\DS)> w$ for all $(i,t) \notin (I,T)$.
 \end{enumerate}
\end{definition}
The importance of using a discriminating set lies in its ability to discriminate, i.e., \emph{include} fuzzy lagged rows which belong to the \FLC\ and \emph{exclude} those that do not.
Therefore, as will be later shown, a discriminating set serving as a seed would grow in a deterministic way to a unique {\FLC}, i.e., choosing a discriminating set more than once will yield the same \FLC.
Next, \THEOREM~\ref{proofs:theorem:d3} states that for an optimal {\FLC} $(I^*,T^*,J^*,\MF)$,
there is an abundance of small sub-sets of columns, 
each of which is a discriminating set with a probability of at least 0.5.

\begin{theorem} \label{proofs:theorem:d3}
	Let $(I^*,T^*,J^*,\MF)$ be an optimal \FLC\ of error
	$w$, with $\MJ \leq (|J^*| / n) < \MJT$, and let $p\in I^*$.
	Any randomly chosen columns subset $\DS$ of $J^*$, of size $|\DS| \geq \log(4mn)/\log(1/3\MJT(2\MF+1))$,
	is a discriminating set for $(I^*,T^*,J^*,\MF)$, with respect to $\Dp$,
	with a probability of at least $0.5$.
\end{theorem}
\begin{proof*}
	Let $(I^*,T^*,J^*,\MF)$ be a {\FLC} with
	a column profile         $\RiS,\ i\in\!I^*$,
	a lagged column profile  $\TiS,\ i\in\!I^*$ and
	a row profile            $\CjS,\ j \in\!J^*$.
	We show that for any $\DS$ that satisfies the above,
	condition~(1) of \DEFINITION~\ref{proofs::discriminating} always holds
	and that the probability of condition~(2) not to hold is less than $0.5$.
    This allows the probabilistic guarantee by the repeated execution.

	Condition~(1) 
	is always satisfied, as
	$\{i,\Dp\}\subseteq I^*$ and $\DS \subseteq J^*$.
	Therefore, $\SW_{_{T,\MF}}(\{i,\Dp\},\DS) \leq \SW_{_{T^*,\MF}}(I^*,\DS) \leq \SW_{_{T^*,\MF}}(I^*,J^*) \leq w$,
    i.e., as being part of the optimal \FLC, the error is not greater than $w$.

	Moving to condition~(2),
    we first extract an upper bound for the probability of $\DS$ to fail to be a discriminating set for $(I^*,T^*,J^*,\MF)$ with respect to $\Dp$, for a \emph{particular} row, its corresponding lag
    and fuzziness.
    Based on \emph{all} possible combinations of rows, lags and fuzziness,
    we calculate the lower bound for the probability of $\DS$ to discriminate,
    showing it to be greater than 0.5.

    The subset $\DS$ fails to be a discriminating set for $(I^*,T^*,J^*,\MF)$
	with respect to $\Dp$, only if there exists a fuzzy lagged row $i$ with it's corresponding lag $t$,
    $(i,t) \notin (I^*,T^*)$, which fits the cluster, i.e., $\SWTF(\{i,\Dp\}, \DS) \leq w$.
	Next, we calculate a bound for the probability of this to hold for a
	\textit{particular} row $i$, lag $t$ and fuzziness $f$.
	According to \DEFINITION~\ref{model:def-flc}, $\SWTF(\{i,\Dp\}, \DS)\leq w$ means that there are
	$R_i$, $T_i$, $\EF_{i,j}$, $R_{\Dp}$, $T_{\Dp} \: (=$0$)$, $\EF_{\Dp,j} \: (=$$\{$0$\})$ and $C_j,\ j \in S$,
	such that:
	$|A_{i,j}-R_i-C_{j+T_i+\EF_{i,j}}| \leq w$ and $|A_{\Dp,j}-R_\Dp-C_{j+T_\Dp+\EF_{\Dp,j}}| \leq w$
	$\forall j \in \DS$.
	Shifting and aligning row $i \in I$ (in the first inequality) by $T_i$ and $\EF_{i,j}$ respectively,
	and subtracting the second inequality (of row $\Dp$)
	we obtain, for all $j \in \DS$ and some $R\:(=R_i-R_\Dp)$:
    \begin{equation} \label{proofs::equation::condToSections}
        |A_{i,j}-A_{\Dp,j}-R| \leq 2w.
    \end{equation}
	%
	Next, we show that
    due to the optimality of the \FLC,
    there are no more than $3 |J^*|$ columns
	that satisfy the above equation for each row $i\in I$.
	If $|A_{i,j}-A_{\Dp,j}-R| \leq 2w$ then:
	$-2w \leq A_{i,j}-A_{\Dp,j}-R \leq 2w$.
	After adding $(A_{\Dp,j}-\CjS-\RpS)$ to both sides we obtain:
	$(A_{\Dp,j}-\CjS-\RpS)-2w \leq A_{i,j}-\CjS-\RpS-R \leq (A_{\Dp,j}-\CjS-\RpS)+2w$.
	Since $(I^*,T^*,J^*,\MF)$ is an optimal \FLC,
	then $|A_{\Dp,j}-\CjS-\RpS| \leq w$ for all $j \in J^*$.
	Therefore, we obtain:
	\begin{equation} \label{proofs::equation::sections}
			-3w \leq A_{i,j} -\CjS -\RpS -R \leq 3w.
	\end{equation}
    We now present \LEMMA~\ref{proofs::lemma::section}, which enables calculating a bound for the number of columns
    that satisfies \EQUATION~\ref{proofs::equation::sections},
    i.e., columns that if considered to be part of the discriminating set will result in adding rows that do not belong to the optimal \FLC.
	\begin{lemma} \label{proofs::lemma::section}
		Let $J \subseteq J^*$, and let $(i,t) \notin (I^*,T^*)$.
		If $|A_{i,j}-\CjS-r| \leq w$ for some $r$ and all $j \in J$,
		then $J \subset J^*$. 
	\end{lemma}
	\begin{proof}	
        By negation, suppose 
        that $(I,T,J,\MF)$ is a \FLC\ 
        that augments the optimal \FLC\ $(I^*,T^*,J^*,\MF)$ by using
        $J \supseteq J^*$, $I=I^*\cup\{i\}$ and $T=T^*\cup\{t\}$.
        The new cluster is a \FLC\ of
		error $w$ satisfying $\TF(I,J)>\TF(I^*,J^*)$,
		hence contradicting the optimality of $(I^*,T^*,J^*,\MF)$. 
	\end{proof}

	The result of \LEMMA~\ref{proofs::lemma::section} is that
	for a fuzzy lagged row $(i,t) \notin (I^*,T^*)$
	there are at most $|J^*|$ columns that lie in an interval of length $2w$
    (derived from $|A_{i,j}-\CjS-r| \leq w$ of \LEMMA~\ref{proofs::lemma::section}).
    Therefore, the interval $[-3w,3w]$ of \EQUATION~\ref{proofs::equation::sections},
    which can be seen as the three intervals
	$[-3w,-w]$, $[-w,w]$ and $[w,3w]$,
	contains at most $3|J^*|$ columns
    that satisfy \EQUATION~\ref{proofs::equation::condToSections}. 
	Choosing \emph{all} columns of $\DS$ out of the above $3|J^*|$ columns would result
    in the inclusion of an undesirable fuzzy lagged row $(i,t) \notin (I^*,T^*)$.
    Therefore, choosing $|\DS|$ columns from the $3|J^*|$ columns out of the $n$ matrix columns has a probability which is bounded by:
	$(3 |J^*| / n)^{|\DS|} \leq (3 \MJT )^{|\DS|}$.
	
	The latter probability refers to a
	\emph{particular} row $i$, lag $t$ and fuzziness $f$.
	The number of combinations for
	\textit{some} row $i$ ($1\!\leq\!i\!\leq\!m$),
	\textit{some} lag $t$ ($-n\!\leq\!t\!\leq\!n$)
	and \textit{some} fuzziness $f$ ($-\MF\!\leq\!f\!\leq\!\MF$)
	is: $(m) (2n) (2\MF+1)^{|\DS|}$
    (as each of the $|\MF|$ columns can be assignment with any fuzziness within the range $-\MF\!\leq\!f\!\leq\!\MF$).
	Therefore, the probability of \emph{not} discriminating is bounded
	(after substituting $|\DS| \geq \log(4mn)/\log(1/3\MJT(2\MF+1))$)
	by:
    $2 m n (2\MF+1)^{|\DS|} (3\MJT)^{|\DS|} = 2 m n (3\MJT(2\MF+1))^{|\DS|} < 0.5$.
\QEDclosed
\end{proof*}
This result of \THEOREM~\ref{proofs:theorem:d3} is important
since upon selecting $\Dp \in I^*$ and $\DS \subseteq J^*$,
we can deduce $I^*$ and $T^*$.

Moving to the second part of the proof,
we show that when the {\FLCA} algorithm is run a polynomial number of iterations,
it mines, with a probability of at least 0.5, a \FLC\ encompassing the \textbf{optimal} {\FLC}.
We base this on \THEOREM~\ref{proofs:theorem:d3}, which shows the
abundance of randomly selected discriminating sets of size ${\cal O}(\log (mn))$ with a discriminating probability of at least 0.5.

\begin{theorem} \label{proofs:theorem:rc}
	Let $\DS$ be a discriminating set for an optimal \FLC\ $(I^*,T^*,J^*,\MF)$ of error $w$.
	Provided $\LOOPS \geq 2 \ln2 / \MI \MJ^{|\DS|}$,
	the {\FLCA} algorithm will mine a \FLC\ $(I,T,J,\MF)$ of error $2w$ such that:
	$I=I^*$, $T=T^*$, $J \supseteq J^*$ and $|J| \leq 2|J^*|$, with a probability of at least 0.5.
\end{theorem}
\begin{proof*}
	Since $|I^*| \geq \MI m$, the probability of choosing a row (see \LINE~\ref{algorithm:AlgDiscRow})
	that satisfies $\Dp \in I^*$ is at least $\MI$.
	As $|J^*| \geq \MJ n$, the probability of choosing a discriminating columns set
    (see \LINE~\ref{algorithm:AlgDiscColSet})
	which satisfies $\DS \subseteq J^*$ is at least $\MJ^{|\DS|}$.
	%
	Following \THEOREM~\ref{proofs:theorem:d3}, any given $\DS \subseteq J^*$
	is a discriminating set with a probability of at least $0.5$ with respect to $\Dp$.
	Therefore, the probability that all
	$\LOOPS$ iterations (see \LINE~\ref{algorithm:AlgLoop})
	fail to find a discriminating row $\Dp$
  and a discriminating columns set $\DS$
  is $(1-0.5 \MI \MJ ^{|\DS|})^{\LOOPS}$.
  Substituting $\LOOPS \geq 2 \ln2 / \MI \MJ^{|\DS|}$ we obtain a maximum probability of
  $(1 - 0.5 \MI \MJ ^{|\DS|})^{ 2\ln2 / (\MI \MJ^{|\DS|})}$.
	Using the inequality $(1-1/x)^x < 1/e$, for $x \geq 1$, with $x=\frac{2}{\MI \MJ ^{|\DS|}}$
	we get a probability that does not exceed
	$(1 - \frac{\MI \MJ ^{|\DS|}}{2})^{\frac{2}{\MI \MJ^{|\DS|}} \ln2 } < 1/e^{ln2} = 0.5$.
	It follows that the algorithm's chances of mining a {\FLC} upon a $\Dp \in I^*$ and
	$\DS \subseteq J^*$ is at least $0.5$.
	When such a \FLC\ is mined, we obtain from
    the discriminating property of $\DS$
    (see \DEFINITION~\ref{proofs::discriminating})
	that $I=I^*$ and $T=T^*$.
	
    The following lemmas prove that the mined {\FLC} contains $J^*$
    and at most $|J^*|$ additional columns.
    \pagebreak[4]
	\begin{lemma} \label{proofs::J-factor-2}
		$|J| \leq 2|J^*|$, i.e., the size of the mined columns set $J$
		is a maximum 2 factor of the size of the optimal cluster columns set $J^*$.
	\end{lemma}

	\begin{proof}	
		A column $j$ is added to $J$ only if:
        $\max_i(A_{i,j+T_i+f}$-$A_{i,\Ds+T_i})$-$\min_i(A_{i,j+T_i+f}$-$A_{i,\Ds+T_i})$$\leq$4w
		(see {\LINE}s~\ref{algorithm:AlgColPhaseSlidingWindow}-\ref{algorithm:AlgColPhaseCheck}),
        which is equal to $\SWTF(I, J) \leq 2w$
        (see \REMARK~\ref{proofs:theorem:remark-calc-profile}, below, in the case of a matrix of two columns).
		Therefore, $\forall i \in I$ and $\forall j \in J$ there exists
		$R_i$, $T_i$, $C_j$ and $\EF_{i,j}$ such that: $-2w \leq R_i + C_{j+T_i+\EF_{i,j}} - A_{i,j}\leq 2w$.
		Since $I=I^*$, $T=T^*$ and initially $J = S \subseteq J^*$,
        we obtain from the optimality of $(I^*,T^*,J^*,\MF)$,
		that for each of the intervals $[-2w,0]$ and $[0,2w]$, there are at most $|J^*|$ columns $j$
		satisfying $|\RiS + C^{*}_{j+\TiS+\FijS} - A_{i,j}| \leq w$.
		Thus, $J$ accumulates up to a maximum of $2|J^*|$ columns.	
	\end{proof}

	\begin{lemma} \label{proofs::J-factor-3}
		$J \supseteq J^*$, i.e., the mined columns set $J$ contains the optimal cluster columns set $J^*$.
	\end{lemma}
	\begin{proof}	
		For each $j \in J^*$, we obtain from the optimality of ($I^*, T^*, J^*, \MF)$ that
		$|A_{i,j} -  \RiS  - C^{*}_{j+\TiS+\FijS}| \leq w$.
		Thus, $j$ will be added to $J$, namely $j \in J$.
	\end{proof}

	As a consequence of \LEMMA~\ref{proofs::J-factor-2},
	additional columns that are not in $J^*$ might be added.
	Yet the maximum number of added columns is $|J^*|$ (see \LEMMA~\ref{proofs::J-factor-3}) and
    the mined cluster will have a maximal error of $2w$.
\QEDclosed
\end{proof*}

\begin{remark} \label{proofs:theorem:remark-setting-S}
    The bound of 	
	$|\DS|\geq$ $\log(4mn)/$ $\log(1/3\MJT(2\MF+1))$
	includes the parameter $\MJT$ whose value is not given as part of the problem input.
	In \SUBSECTION~\ref{Expr:artificial:discr-set-size}, we show experimentally that
	a random subset of size $0.6 \log_2(4mn) -1$ will suffice,
	freeing the user from the burden of specifying the $\MJT$-related trade-off.
\end{remark}

\begin{remark} \label{proofs:theorem:remark-S-discriminating-probability}
  \THEOREM~\ref{proofs:theorem:d3} describes a discriminating set
  $\DS$ with a minimum discriminating probability of 0.5.
	In \SUBSECTION~\ref{Expr:artificial:discr-probabilities}, we illustrate the
	relation between various magnitudes of $|\DS|$ and their discriminating probability.	
\end{remark}

\def \MN  {k}
\def \IMN  {j}
\begin{remark} \label{proofs:theorem:remark-calc-profile}
    \THEOREM~\ref{proofs:theorem:d3}, \THEOREM~\ref{proofs:theorem:rc} and \ALGORITHM~\ref{algorithm:FLC} all
    assume only the \emph{existence} of the profiles $\RiS$, $\TiS$ and $\CjS$.
    Although the actual values of the profile are \emph{not} calculated in practice, an explicit calculation can be computed.
    %
    In the case of a matrix $A$ of size [2$\times$\MN]
    (equivalent to [\{i,\Dp\}$\times$\DS], see \DEFINITION~\ref{proofs::discriminating}),
    one can use the following polynomial technique~\cite[\SUBSECTION~4.1]{melkman2004sc}.
    %
    First, permute the columns of the matrix $A$ so that:
    $$A_{1,1}-A_{2,1}\leq A_{1,2}-A_{2,2}\leq \cdots \leq A_{1,\MN}-A_{2,\MN}.$$
    Next, set
    $w = [(A_{1,\MN}-A_{2,\MN}) - (A_{1,1}-A_{2,1})]/2$,
    $h=[(A_{1,\MN}-A_{2,\MN}) + (A_{1,1}-A_{2,1})]/2$
    and let $\ell$ be
    such that: $A_{1,\ell}-A_{2,\ell} \leq h \leq A_{1,\ell+1}-A_{2,\ell+1}$.
    Then $R=<\!0,-h\!>$ and $C=<\!A_{1,1}+w,A_{1,2}+w,\ldots, A_{1,\ell}+w,A_{1,\ell+1}-w,\ldots, A_{1,\MN}-w\!>$.
    Therefore, we get $\SWTF(I, J)$=$[\max_{\IMN \in J} (A_{1,\IMN}-A_{2,\IMN})-\min_{\IMN  \in J} (A_{1,\IMN}-A_{2,\IMN})]/2$, where $|I|$=2 and $|J|$=\MN.

    \nin
    For cases of a general matrix size,
    we refer the reader to Melkman et al.~\cite[\SUBSECTION~4.2]{melkman2004sc},
    which is a discrete version of the Diliberto-Straus algorithm~\cite{diliberto1951approximation}.

\end{remark}

\begin{remark} \label{proofs:theorem:remark-generic-distribution}
    Neither \THEOREM~\ref{proofs:theorem:d3} nor \THEOREM~\ref{proofs:theorem:rc}
    make any assumption whatsoever on the distribution of the data in the matrix nor on the distribution of the data in the \FLC\ to be mined. The {\FLCA} algorithm is completely generic.
\end{remark}


\subsection{Extensions} \label{subsection:algorithm:model-extensions}

Next, we present several extensions to the {\FLCA} algorithm and the resulting algorithmic modifications supporting these extensions.

\subsubsection{Varying Fuzziness} \label{algorithm:impl-note-MF-foreach-row}
	A major characteristic of the {\FLCA} algorithm is the maximum allowed fuzziness $\MF$.
	This fuzziness is assumed to be common to all entries of the matrix.
	The algorithm can be extended to include a \emph{different} maximum fuzziness
    for each row of the matrix, denoted ${\MF}_i,\ i\in m$.
	To achieve this, the algorithm needs to be modified in the events which are later used by the sliding window
    ({\LINE}s~\ref{algorithm:AlgRowPhaseSlidingWindow} and \ref{algorithm:AlgColPhaseSlidingWindow} of \ALGORITHM~\ref{algorithm:FLC}).
    Essentially, the modification is in using the row's maximum allowed fuzziness ${\MF}_i$ instead of the global fuzziness $\MF$.
    The modifications are as follows.
    \begin{itemize}

    \item[$\circ$]
    \LLINE~\ref{algorithm:AlgRowPhaseSlidingWindow}, which accumulates rows, should be modified to: \\
    \textbf{Slide a $4w$ width window on} \\
    $\{e_{s,j,f} \:|\: e_{s,j,f} = A_{i,j+f} -  A_{\Dp,s},\: \forall j\in n,\: \forall s \in \DS,\: -{\MF}_i \leq f \leq {\MF}_i \}$.

    \item[$\circ$]
    \LLINE~\ref{algorithm:AlgColPhaseSlidingWindow}, which accumulates columns, should be modified to: \\
    \textbf{Slide a $4w$ width window on} \\
    $\{e_i\:|\: e_i = A_{i,j+T_i+f} - A_{i,\Ds+T_i},\: \forall i\in I,\: -{\MF}_i \leq f \leq {\MF}_i  \}$.
    \end{itemize}
    Similarly,
    the algorithm can also be extended to
    include a different maximum allowed fuzziness for each column of the matrix, denoted ${\MF}_j,\ j\in n$.

\subsubsection{Reduction in Size of the Discriminating Set} \label{algorithm:impl-note-SFZ}
	\def \SFZ {S^0}

	The discriminating set $\DS$, as shown by \THEOREM~\ref{proofs:theorem:d3},
	is a small subset of size ${\cal O}(\log (mn))$.
	Nevertheless, the number of iterations $\LOOPS$ required for achieving the probabilistic guarantee
    as shown by \THEOREM~\ref{proofs:theorem:rc},
	is \emph{exponentially} proportional to $|\DS|$.
    To improve the algorithm's run-time, we propose to take
   	a \textit{subset} of the discriminating columns set $\DS$, denoted $\SFZ$,
    and assume it has zero fuzziness over all of the cluster's rows,
	i.e., let $(I,T,J,\MF)$ be a {\FLC} with a discriminating set of columns $\DS$,
    with the assumption that $\EF_{i,j}$=0, for all $i\in I,\ j\in \SFZ$.
	The assumption reduces the combinatorial number of rows that needs to be filtered by the discriminating set,
	and thus reduces the set size needed for the task
    (\LINE~\ref{algorithm:AlgDiscColSet} of \ALGORITHM~\ref{algorithm:FLC}).
    However, this comes with the cost of limiting the nature of the clusters mined,
    i.e., {\FLC}s which do not have a minimum of $|\SFZ|$ columns of zero fuzziness would not be mined.
    To achieve that, we modify the {\FLCA} algorithm in the following way.
	\LLINE~\ref{algorithm:AlgDiscColSet}, in addition to randomly choosing $\DS$,
	also randomly chooses $\SFZ \subseteq \DS$.
	Next, we modify \LINE~\ref{algorithm:AlgRowPhaseSlidingWindow}
    to set the fuzziness $f$
    such that
    if $s\in\SFZ$ then $f$=0, or otherwise $-\MF\!\leq\!f\!\leq\!\MF$.
    The results of an experiment to evaluate the effectiveness of this approach are reported in
    \SUBSECTION~\ref{subsec:Expr:artificial} (see \ExprArtificialDiscrSetSize),
    revealing that even a moderated subset of $\DS$ for which a zero fuzziness is assumed,
    e.g., $|\SFZ|$=3, supplies a good balance
	between the gain in run-time and the constraint it implies on the model.
    {\ExprArtificialLowDiscrSizeMoreIterations}
    suggests a technique which enables the use of an even lower discriminating set size.

\subsubsection{Finding the Maximal Columns Set} \label{algorithm:impl-note-Japaneses-Bridges-problem-solution}

	\def \IG {\widehat{G}}

	As part of the process of column addition
    ({\LINE}s~\ref{algorithm:AlgColPhaseStart}-\ref{algorithm:AlgColPhaseEnd} of \ALGORITHM~\ref{algorithm:FLC}),
    each added column has its fuzziness setting.
    Nevertheless, when considering those settings in the context of a cluster, it may well happen that they do not co-exist.
	Take for example the following simple scenario:
    column $j_1$ has a fuzziness of $\EF_{i,j_1}$=$1$ and column $j_2$ has a fuzziness of
    $\EF_{i,j_2}$=$-1$, $\forall i \in I$.
    In the case of zero lag ($T_i$=0) and $j_2$=$j_1+1$,
	the cluster's fuzziness setting would not be valid
    as although $j_1<j_2$,
    the actual matrix columns for $j_1$ and $j_2$ would be $j_2$ ($=j_1+1=j_1+\EF_{i,j_1}$) and $j_1$ ($=j_2-1=j_2+\EF_{i,j_2}$), respectively.
    In the case where $j_1<j_2$ are time points,
    we expect that their actual matrix entries will also maintain the same ordering relations (and not $j_2<j_1$).

    \def \HEIGHT    {0.26\textheight} 
    \def \WIDTH	    {0.40\textwidth} 
    \def \HSPACE    {20pt}
	\begin{figure}
	  \centering   	
		  \subfloat [A bridge matrix] {
			\includegraphics[clip=true,trim=50 122 60 90,
                                height=\HEIGHT, width=\WIDTH]{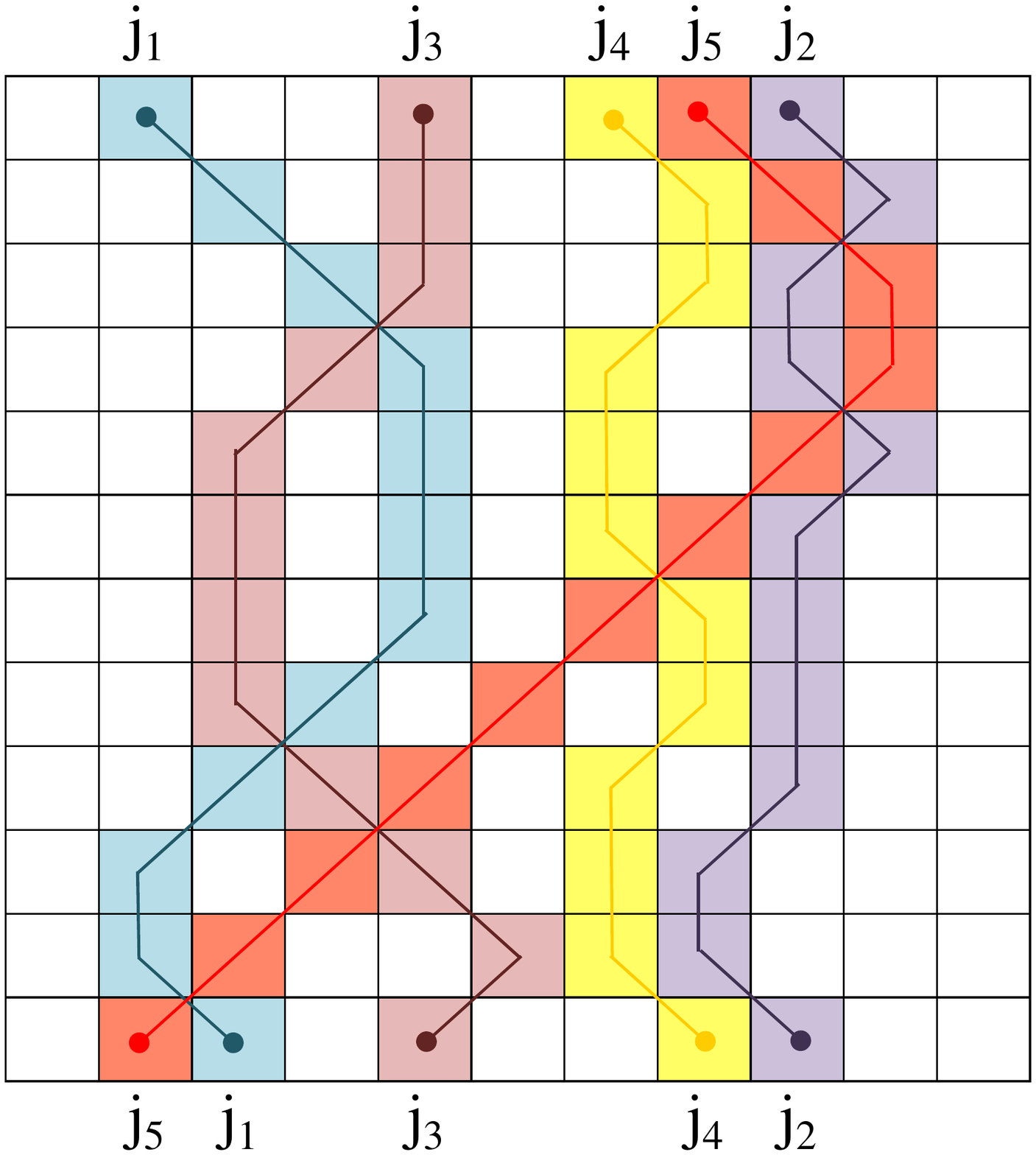}
			\label{Algo:japaneses-bridges-example-matrix}
		  }	
	  	  \subfloat [Intersection graph] {
            \hspace{\HSPACE}
            \includegraphics[clip=true,trim=170 242 310 34,
                                height=\HEIGHT, width=\WIDTH]{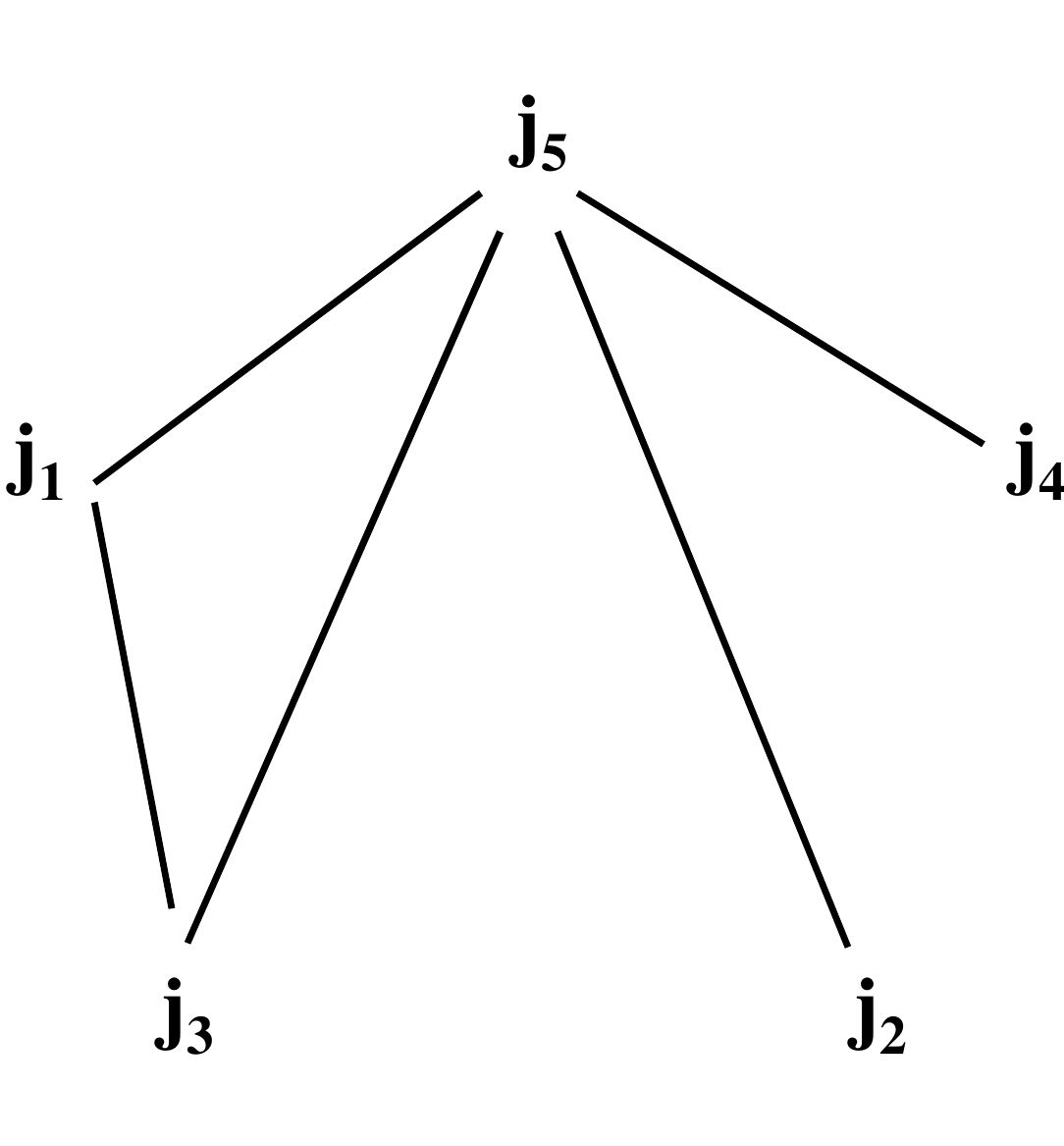}
			\label{Algo:japaneses-bridges-example-intersection}
		  }	
		  \caption{
           \protect\subref{Algo:japaneses-bridges-example-matrix} Example of a bridge matrix.
           Each color represents a different bridge.
           The lines related to each color represent the bridge graph.
           Take for example column $j_2$ and $j_4$. Assuming a zero lag (i.e., $T_i=0,\ \forall i \in I$),
           the column's fuzziness is $\{0, 1, 0, 0, 1, 0, 0, 0, 0, -1, -1, 0\}$
           and $\{0, 1, 1, 0, 0, 0, 1, 1, 0, 0, 0, 1\}$, respectively.
           \protect\subref{Algo:japaneses-bridges-example-intersection} The intersection graph of the bridge graph.
           The bridge of $j_5$ (red color) intersects all other bridges
           ($j_1$, $j_2$, $j_3$ and $j_4$).
            The bridges of $j_1$ (blue) and $j_3$ (wine) also intersect.
           All other bridges do not intersect
           (e.g., the bridges of $j_2$ (purple) and $j_4$ (yellow)).
           Therefore, the maximum non-intersecting set of bridges is either
           \{$j_1$, $j_2$, $j_4$\} or \{$j_3$, $j_2$, $j_4$\}.
		  }
		\label{Algo:japaneses-bridges-example}
	\end{figure}
	One solution to this problem can be a post-processing step.
	We denote each of the columns' fuzziness setting as a \emph{bridge},
	where the bridges are drawn on the discrete entries of the matrix $A$
    (see example in \FIGURE~\ref{Algo:japaneses-bridges-example-matrix}).
    We therefore wish to find the maximum non-intersecting set of bridges.
	To do so, consider the following problem:
	let $G$=$(V,E)$ be a bi-partite bridge graph with $|V|$=$n$ vertices on each side, and each
	edge $e$$\in$$E$ is a monotonic path between the upper and the lower side,
    i.e., a monotonous path of $\langle A_{i_1,j_1}, A_{i_2,j_2}, \ldots, A_{i_k,j_k}\rangle$, where $i_1<i_2<\ldots<i_k$.
	The goal is to find the maximal set of non-intersecting edges in graph $G$.
    We do so by first showing in \LEMMA~\ref{algorithm:impl-note-Japaneses-Bridges-perfect-graph}
    that the intersection graph $\IG$
    of the bridge graph $G$
    (see example in \FIGURE~\ref{Algo:japaneses-bridges-example-intersection} and \ref{Algo:japaneses-bridges-example-matrix}, respectively)
    is a perfect graph.
    Next, we conclude in \COROLLARY~\ref{algorithm:impl-note-Japaneses-Bridges-perfect-corollary}
    that the graph $G$ is also a perfect graph. As such, polynomial algorithms for finding a maximum clique can be applied, which results in finding the maximum set of non-intersecting columns' bridges.

    \def \HEIGHT    {0.21\textheight} 
    \def \WIDTH	    {0.35\textwidth} 
    \def \HSPACE    {20pt}
	\begin{figure}
	  \centering   	
		  \subfloat [Intersection Graph $\IG$] {
			\includegraphics[clip=true,trim=174 272 311 64,
                                height=\HEIGHT, width=\WIDTH]{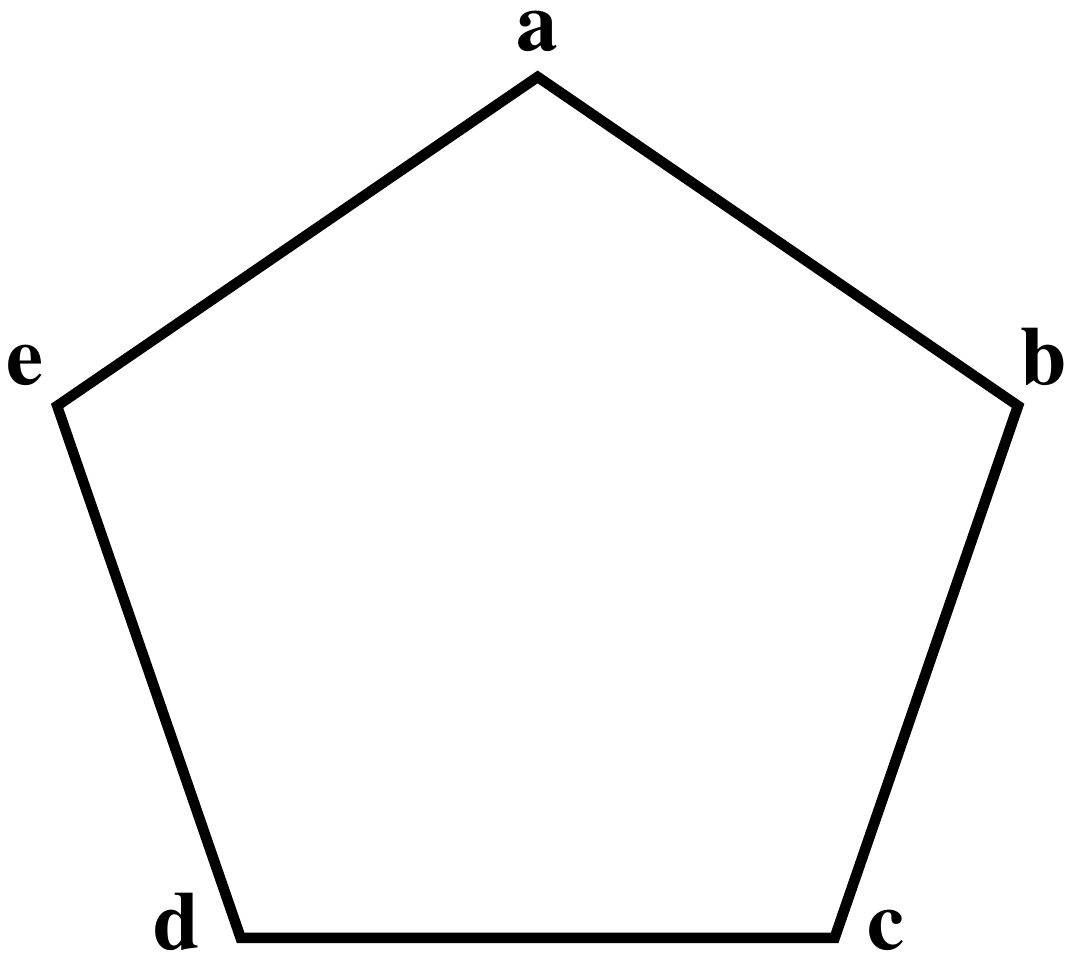}
			\label{Algo:japaneses-bridges-problem-intersection}
		  }	
	  	  \subfloat [Bridge Graph $G$] {
            \hspace{\HSPACE}
            \includegraphics[clip=true,trim=174 140 300 114,
                                height=\HEIGHT, width=\WIDTH]{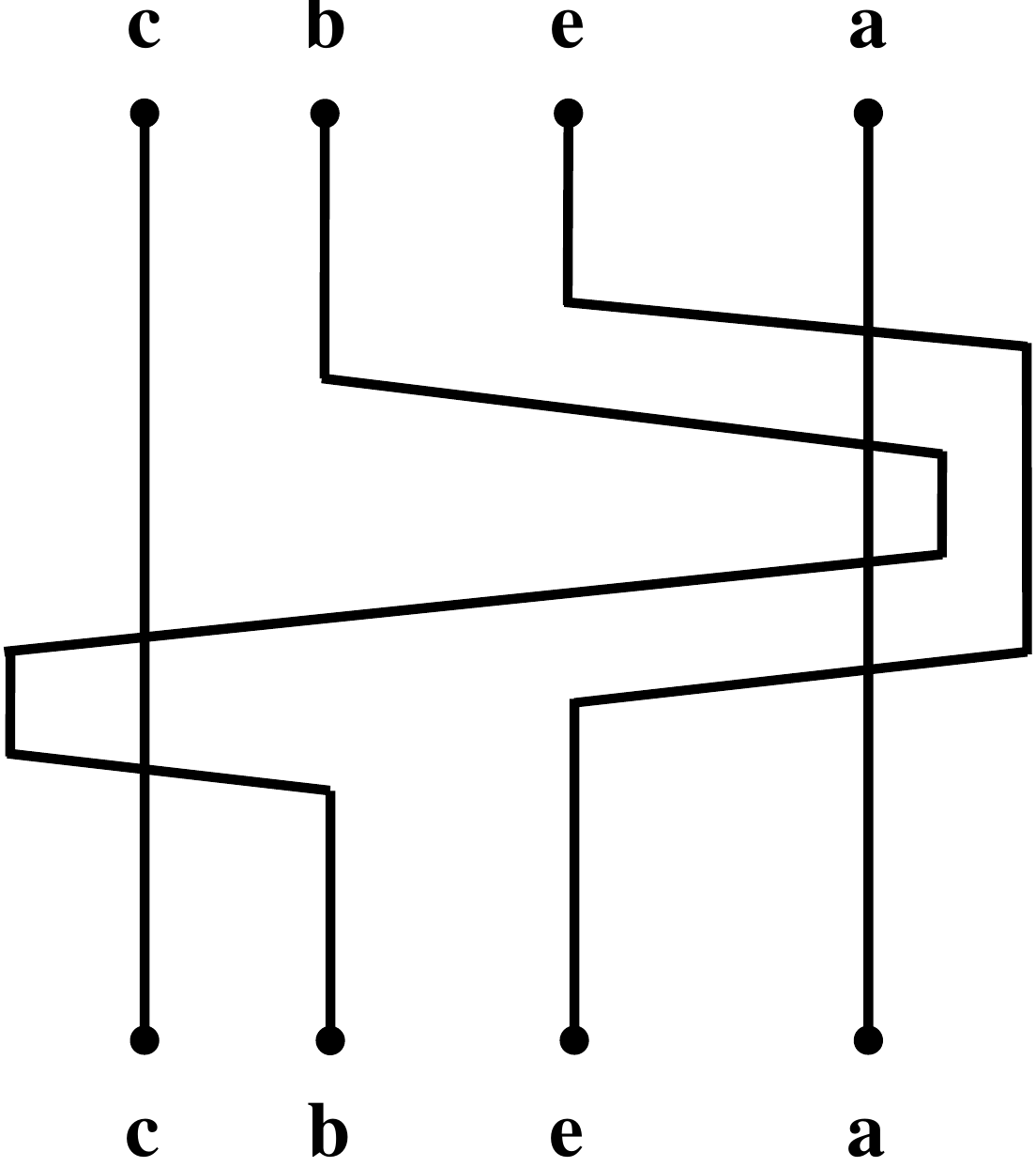}
			\label{Algo:japaneses-bridges-problem-5-cycle}
		  }	
		  \caption{
		  	A graph is perfect if it cannot have a cycle of a minimum length of $5$ \cite{robertson2006strong}.
		  	\FIGURE~\ref{Algo:japaneses-bridges-problem-intersection} presents a $5$-cycle intersection graph $\IG$.
		  	\FIGURE~\ref{Algo:japaneses-bridges-problem-5-cycle} illustrates the \emph{failure} to draw
		  	the correlating $5$-cycle bridge graph $G$.
		  }		 	
		\label{Algo:japaneses-bridges-problem}
	\end{figure}
	\begin{lemma} \label{algorithm:impl-note-Japaneses-Bridges-perfect-graph}
		A bridge graph is a perfect graph.
	\end{lemma}
	\begin{proof}
		Let $\IG$ be the intersection graph of the bridge graph $G$.
		We show by negation that the intersection graph $\IG$ cannot have a cycle of a minimum length of $5$,
		and thus $\IG$ is a perfect graph \cite{robertson2006strong}.
		\FIGURE~\ref{Algo:japaneses-bridges-problem} depicts the following steps.
		\begin{enumerate}
			\item
			By negation, let us assume that there is an intersection graph $\IG$
			with a cycle of a minimum length $5$
            (see \FIGURE~\ref{Algo:japaneses-bridges-problem-intersection}).
			
			\item
			Let us denote the $5$ vertices of $\IG$ as $a$, $b$, $c$, $d$ and $e$.

			\item
			Because $a$ and $c$ do not intersect, assume (w.l.o.g.) that $a$ is to the right of $c$
            (see \FIGURE~\ref{Algo:japaneses-bridges-problem-5-cycle}).
			
			\item
			Observe that because $c$ and $e$ do not intersect, $e$ cannot be to the left of $c$
            as it should intersect $a$.
			Therefore $e$ is to the right of $c$.
			
			\item
			$b$ intersects $c$ but does not intersect $e$. Therefore $e$ is to the right of $b$.
			
			\item
			$d$ does not intersect $b$:
			
			\nin
			(a) if $d$ is to the left of $b$ it cannot intersect $e$ -- negation.
			
			\nin
			(b)	if $d$ is to the right of $b$ it must also be to the right of $a$
            but then it cannot intersect $c$ -- negation.
		\end{enumerate}
	\end{proof}

	\begin{corollary} \label{algorithm:impl-note-Japaneses-Bridges-perfect-corollary}
		The algorithm's column addition phase,
        which results in a maximum set of non-intersecting columns,
        has a polynomial run-time.
	\end{corollary}
	\begin{proof}
		Following \LEMMA~\ref{algorithm:impl-note-Japaneses-Bridges-perfect-graph},
		the intersection graph $\IG$ is a perfect graph.
		As the complement graph of a perfect graph is also a perfect graph \cite{lovasz1972normal},
		we can apply a maximum clique polynomial run-time procedure
        \cite{grotschel1989Geometric,grotschel1984polynomial}
		to the graph $G$
		in order to acquire a maximum columns set.
	\end{proof}



\section{Experiments} \label{sec:Experiments}

Next, we present an extensive evaluation of the {\FLCA} algorithm, using both artificial and real-life data.



\subsection{Experiments with Artificial Data} \label{subsec:Expr:artificial}

In comparison to real-life data, the use of artificial data enables maximum control over the algorithm's input and parameter  settings, which in turn enables
the verification and validation of the algorithm's output.
Specifically, the contributions of the experimentation used for the {\FLCA} algorithm with artificial data are threefold.
First, it establishes default values for the various parameters.
Second, it enables the verification of theoretical bounds.
Finally, it demonstrates a feasible actual run-time.\footnote{While the number of iterations is proved to be polynomial, we want to ensure that the actual performance for large inputs is feasible.}

\subsubsection*{{\ExprArtificialProbArtifact}: Probability of Artifacts}
\label{Expr:artificial:prob-artifact}

\def \CLIFF {{phase transition}} 

\nin
An interesting question in the context of the fuzzy lagged model is
how frequently artifacts are mined.
An artifact is a submatrix that was formed not as a result of some hidden regulatory mechanism, but as a mere aggregation of noise. Such artifacts are undesirable as they add irrelevant output.

Given a matrix with randomly generated values (from a uniform distribution),
the probability of mining an artifact \FLC\ $(I,T,J,\MF)$ depends on several parameters:
(1) the matrix dimensions, $[m \times n]$;
(2) the \FLC\ dimensions, $[|I| \times |J|]$;
(3) the error $\SW$, $0\% \leq \SW \leq 100\%$;
and (4) the fuzziness $\MF$.
Intuitively, the larger the error $\SW$, fuzziness $\MF$, and matrix size $m$ and $n$, and the smaller the requested cluster dimensions $I$ and $J$,
the greater the chance of mining artifact clusters with an increasing probability of smaller clusters.
To examine the correlation between these parameters, we present the following upper bound probability analysis.

\def \FormulaCell 			{1-(1-\min(2\SW, 1))^{2\MF+1}}
\def \FormulaRow  			{[\FormulaCell]^{|J|}}
\def \FormulaCluster			{[\FormulaCell]^{|I||J|}}
\def \FormulaNoCluster			{1-[\FormulaCell]^{|I||J|}}
\def \FormulaNoClusterBinom		{\{\FormulaNoCluster\}^{\binom{2mn}{|I|} \binom{3n}{|J|}}}
\def \FormulaClusterBinom		{1-\FormulaNoClusterBinom}

Assume we know the column profile $p$.
The probability of a column $j\in J$ of a fuzzy lagged row $i\in I$ to
be within a surrounding of fuzziness $\MF$ and error $w$, encircling $p$ is: $\FormulaCell$.
Hence, the probability of all columns $j\in J$ of a fuzzy lagged row $i\in I$ to
be within a surrounding of fuzziness $\MF$ and error $w$, encircling $p$ is: $\FormulaRow$.
Thus, the probability of all rows $I$ to form a \FLC\ is: $\FormulaCluster$.
The probability of not having such a \FLC\ is therefore: $\FormulaNoCluster$.
The representation of a fuzzy lagged matrix of size $[m \times n]$  as a non-lagged matrix, results in a matrix of size $[2mn \times 3n]$ (see \SUBSECTION~\ref{subsection:NP-completeness}).
Thus, choosing a set size $|I|$ out of $(2mn)$ rows has $\binom{2mn}{|I|}$ combinations.
Similarly, choosing a set size $|J|$ out of $(3n)$ columns has $\binom{3n}{|J|}$ combinations.
Therefore, the probability that none of the possible sub-matrices of this size in the matrix forms a {\FLC} is:  $\FormulaNoClusterBinom$.
Hence, an upper bound for the probability of at least one artifact \FLC\ to exist is:
\begin{equation} \label{Expr:artificial:probability-eq}
		\FormulaClusterBinom.
\end{equation}
As $\SW$, $\MF$, $m$ and $n$ increase, the above probability will increase.
As $I$ and $J$ increase the above probability will decrease.

\def \HEIGHT    {0.167\textheight} 
\def \WIDTH	    {0.27555\textwidth} 
\def \HSPACE    {-8pt}
\def \FHeight   {3.5cm}

To facilitate understanding of the formula, we present \FIGURE~\ref{Expr:artificial:prob-cliff-change-I-J} and \FIGURE~\ref{Expr:artificial:prob-cliff-change-f-e} (generated by Wolfram$|$Alpha~\cite{WolframAlpha}).
\begin{figure*}
  \centering

  \subfloat[$\MI$=1\%, $\MJ$=1\%.]
  {
    \hspace{\HSPACE}
  	\includegraphics[clip=true,trim=100 380 220 140,height=\HEIGHT, width=\WIDTH]{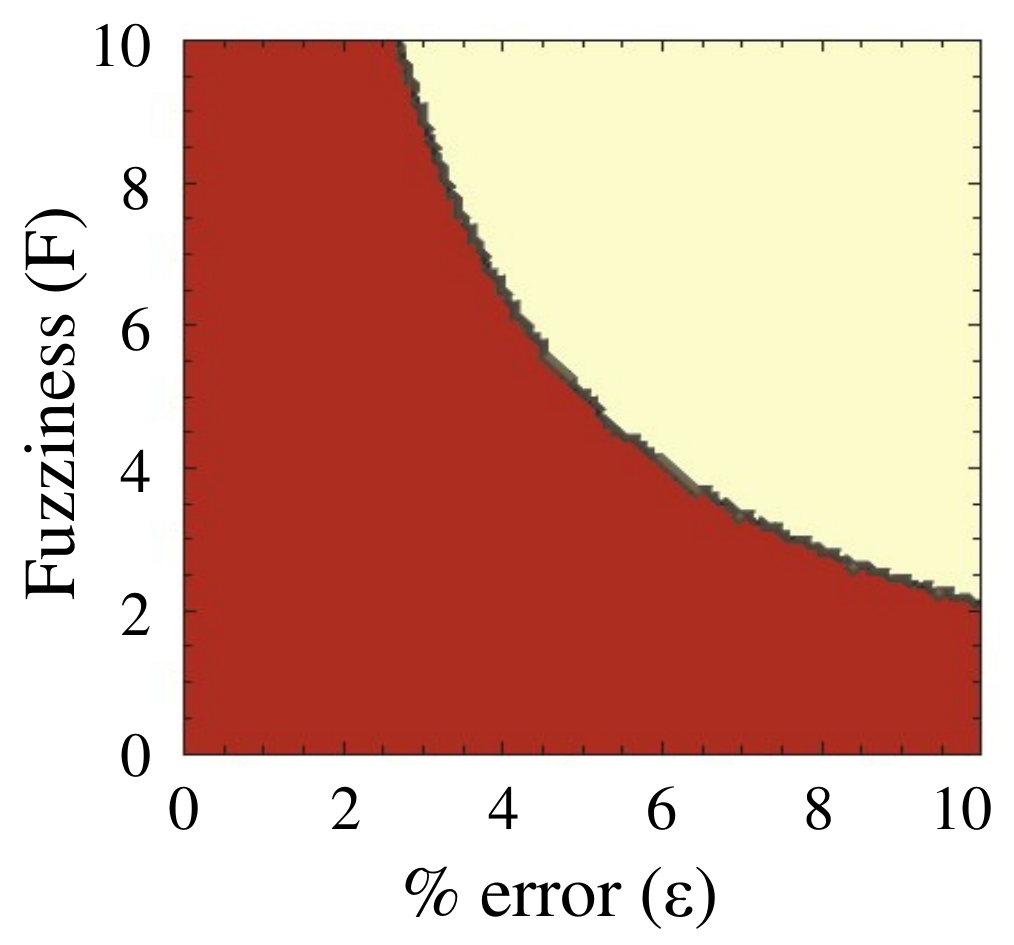}
  	\label{Expr:artificial:prob-cliff-0.01x0.01}
  }
  \subfloat[$\MI$=1\%, $\MJ$=2\%.]
  {
    \hspace{\HSPACE}
  	\includegraphics[clip=true,trim=100 380 220 140,height=\HEIGHT, width=\WIDTH]{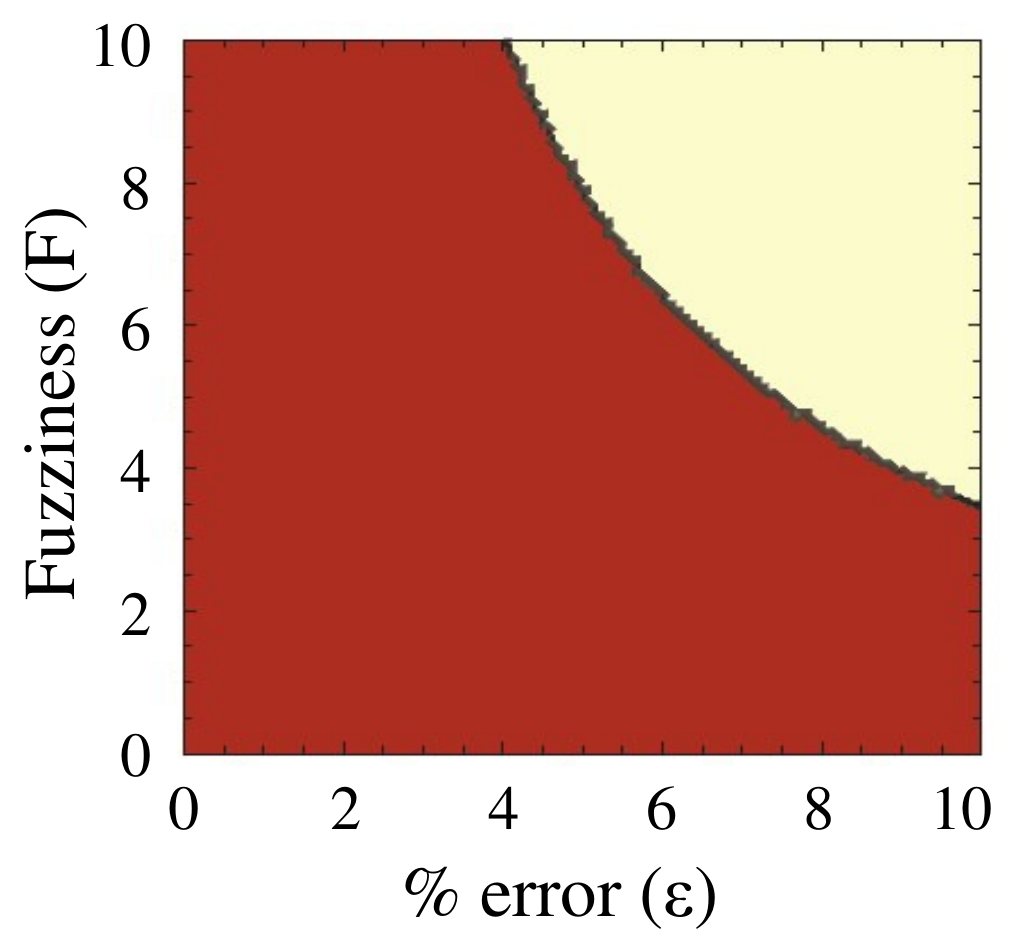}
  	\label{Expr:artificial:prob-cliff-0.01x0.02}
  }
  \subfloat[$\MI$=1\%, $\MJ$=4\%.]
  {
    \hspace{\HSPACE}
  	\includegraphics[clip=true,trim=100 380 220 140,height=\HEIGHT, width=\WIDTH]{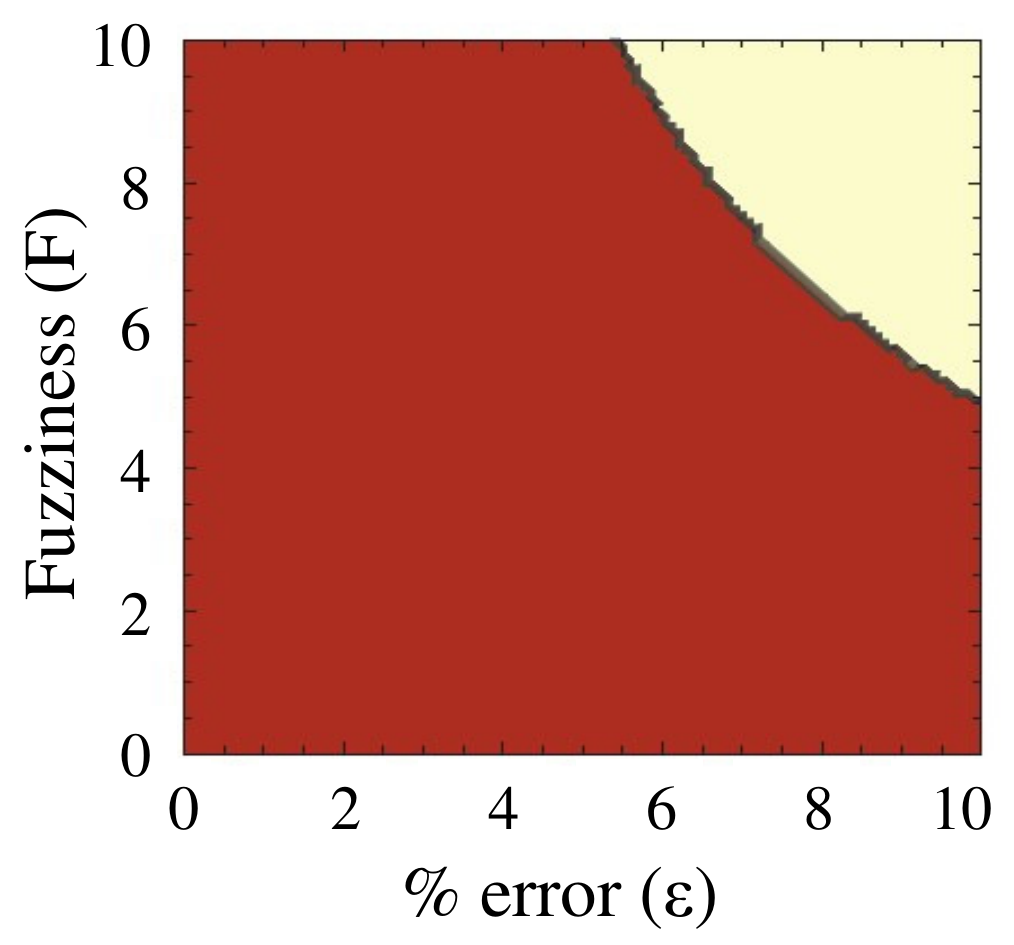}
  	\label{Expr:artificial:prob-cliff-0.01x0.04}
  }
  \subfloat[$\MI$=2\%, $\MJ$=1\%.]
  {
    \hspace{\HSPACE}
  	\includegraphics[clip=true,trim=100 380 220 140,height=\HEIGHT, width=\WIDTH]{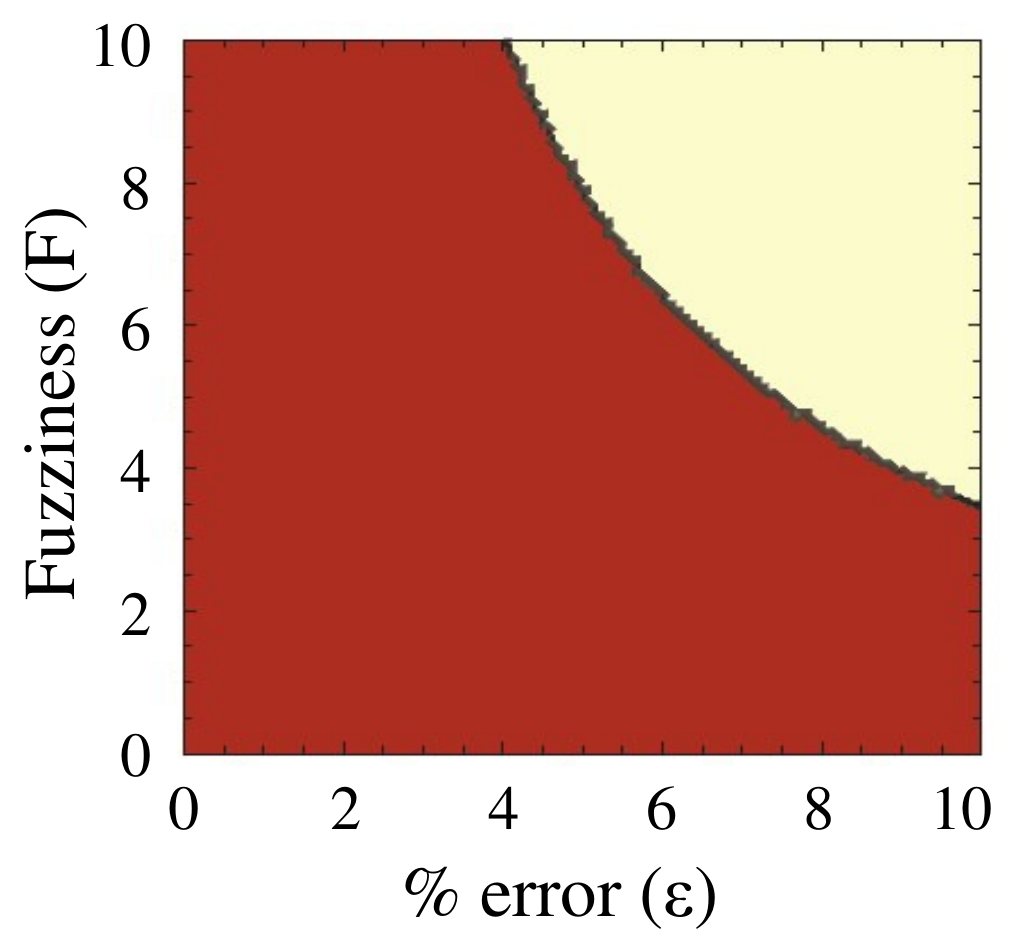}
  	\label{Expr:artificial:prob-cliff-0.02x0.01}
  }
  \subfloat[$\MI$=2\%, $\MJ$=2\%.]
  {
    \hspace{\HSPACE}
  	\includegraphics[clip=true,trim=100 380 220 140,height=\HEIGHT, width=\WIDTH]{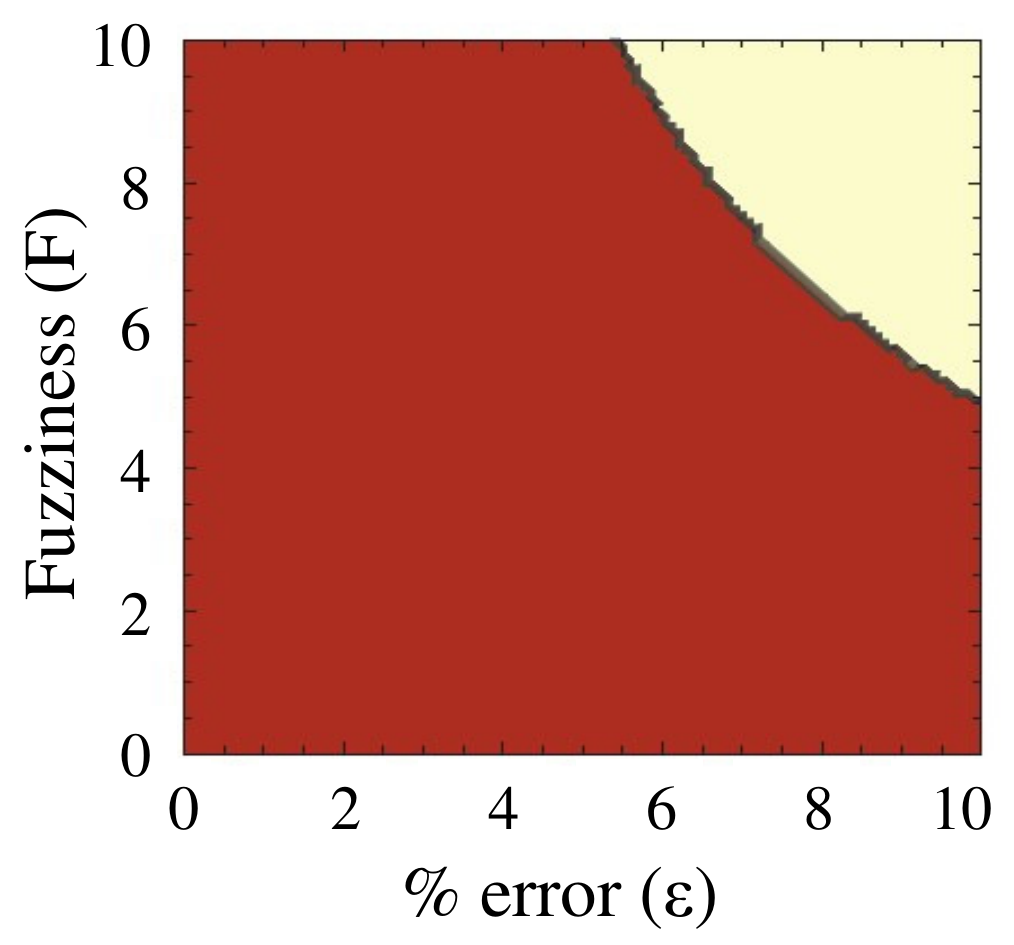}
  	\label{Expr:artificial:prob-cliff-0.02x0.02}
  }
  \subfloat[$\MI$=2\%, $\MJ$=4\%.]
  {
    \hspace{\HSPACE}
  	\includegraphics[clip=true,trim=100 380 220 140,height=\HEIGHT, width=\WIDTH]{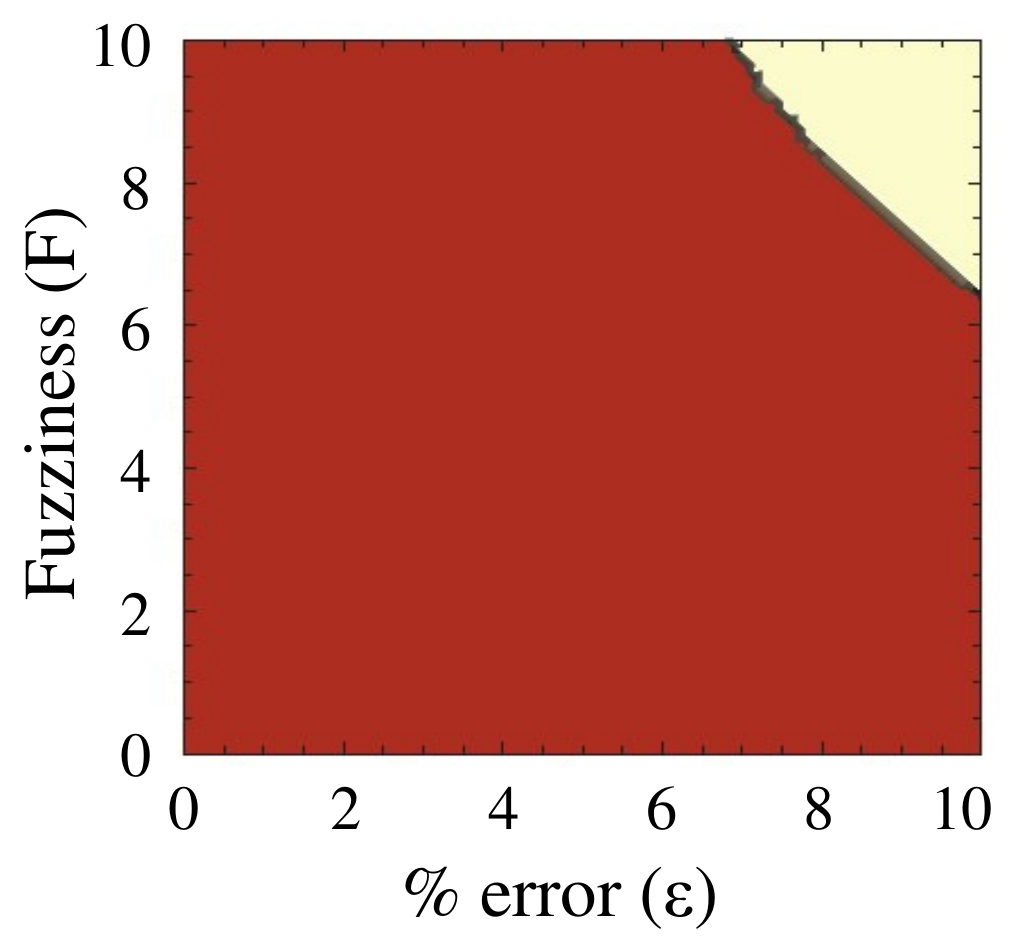}
  	\label{Expr:artificial:prob-cliff-0.02x0.04}
  }
  \subfloat[$\MI$=4\%, $\MJ$=1\%.]
  {
    \hspace{\HSPACE}
  	\includegraphics[clip=true,trim=100 380 220 140,height=\HEIGHT, width=\WIDTH]{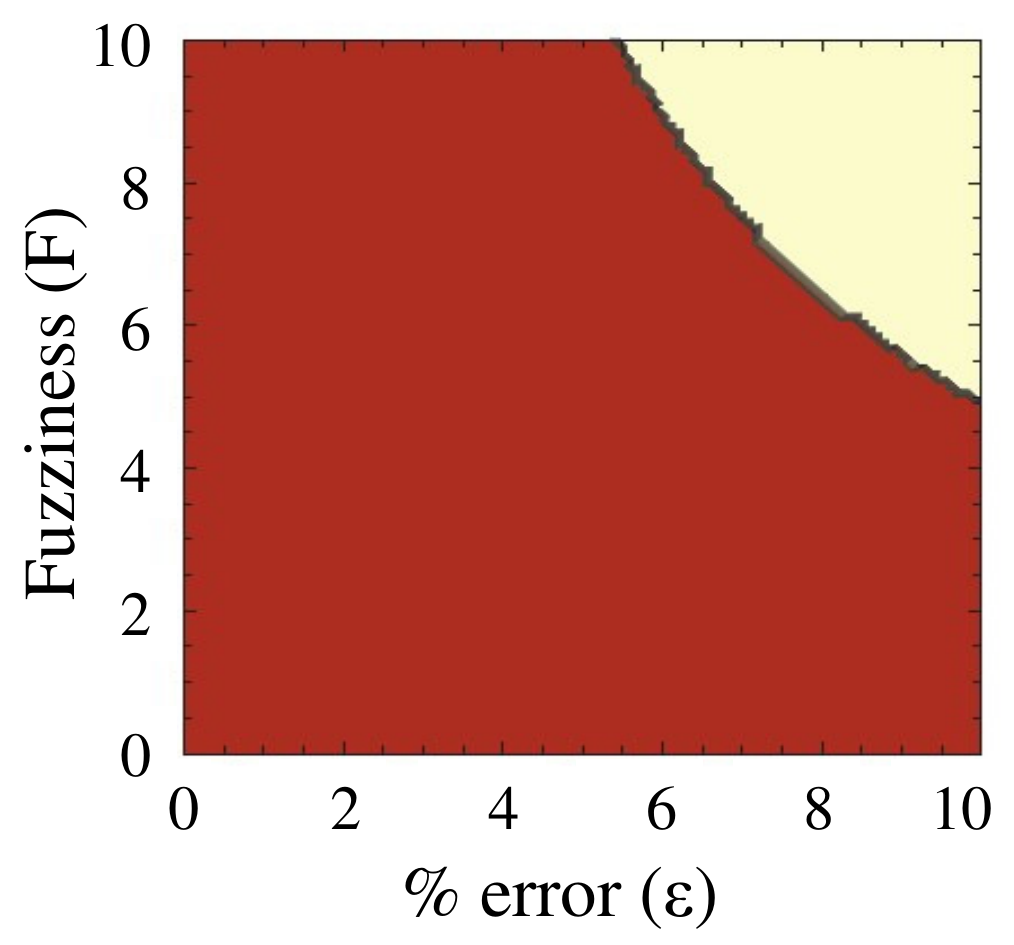}
  	\label{Expr:artificial:prob-cliff-0.04x0.01}
  }
  \subfloat[$\MI$=4\%, $\MJ$=2\%.]
  {
    \hspace{\HSPACE}
  	\includegraphics[clip=true,trim=100 380 220 140,height=\HEIGHT, width=\WIDTH]{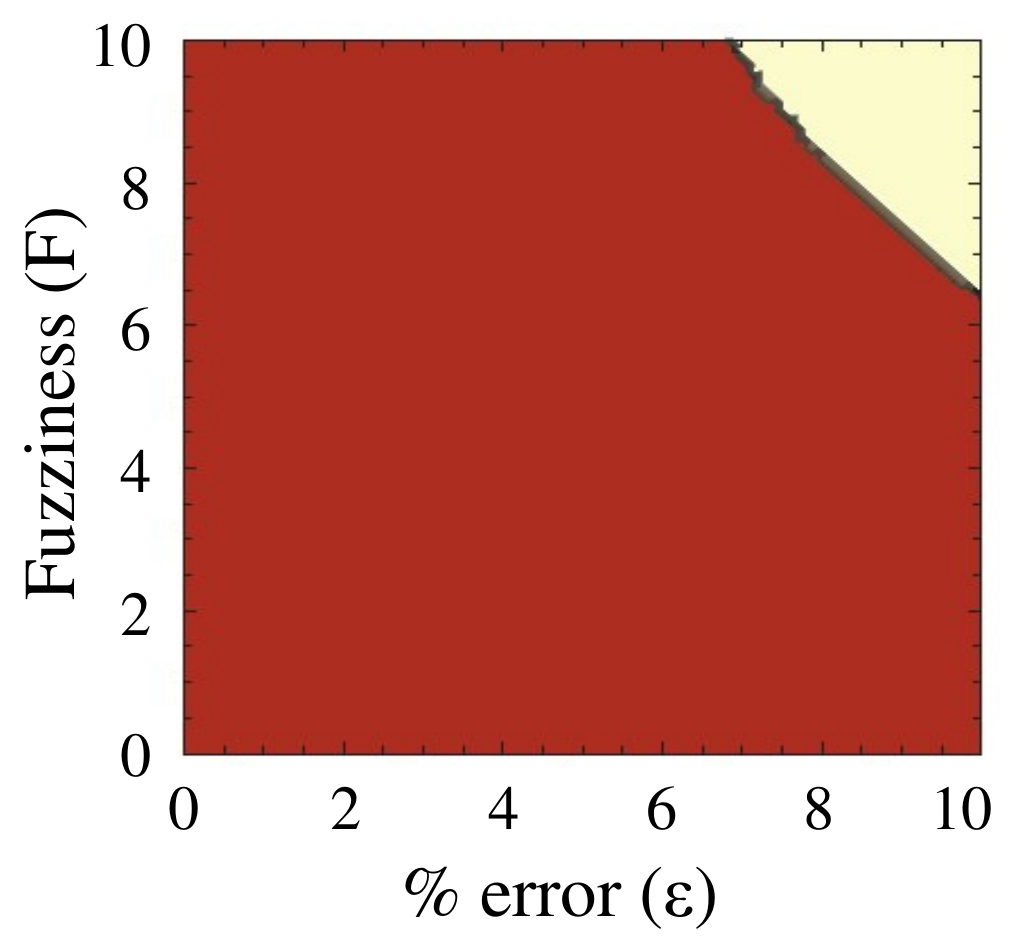}
  	\label{Expr:artificial:prob-cliff-0.04x0.02}
  }
  \subfloat[$\MI$=4\%, $\MJ$=4\%.]
  {
    \hspace{\HSPACE}
  	\includegraphics[clip=true,trim=100 380 220 140,height=\HEIGHT, width=\WIDTH]{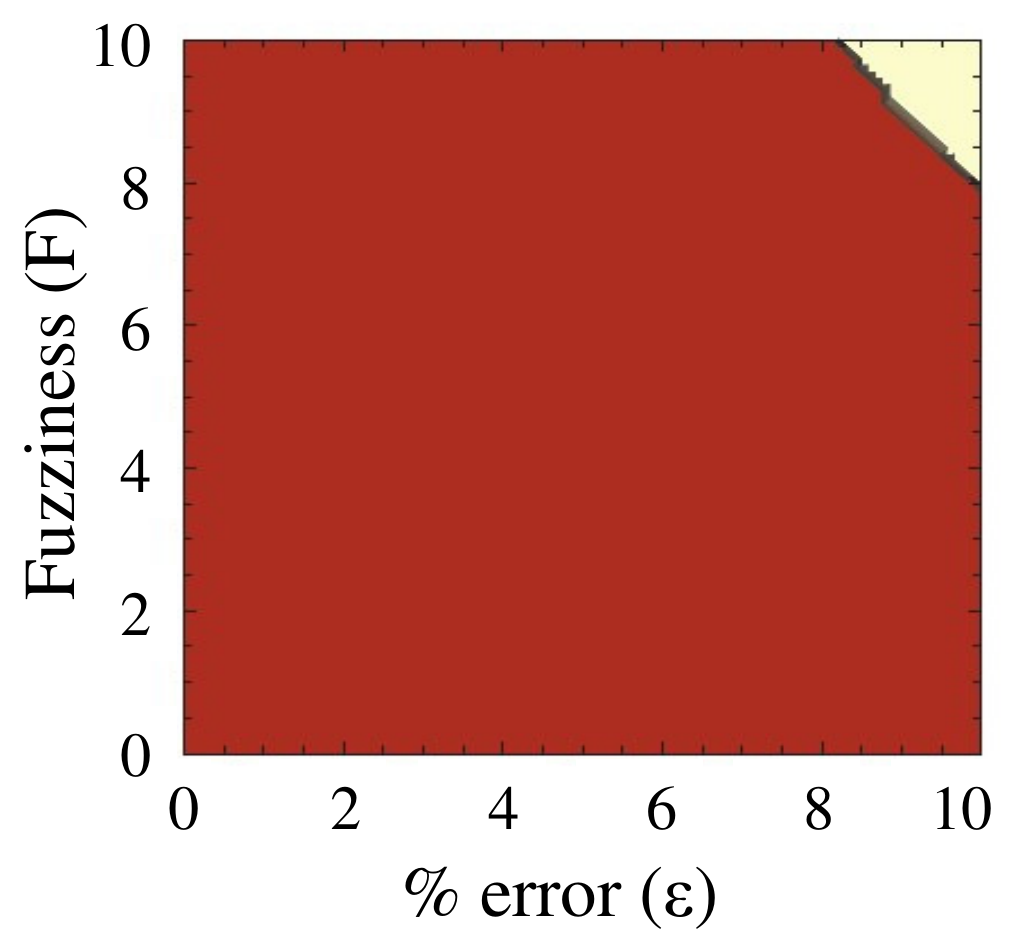}
  	\label{Expr:artificial:prob-cliff-0.04x0.04}
  }

  \subfloat[$\MI$=7\%, $\MJ$=7\%.]
  {
  	\includegraphics[clip=true,trim=100 240 220 280,height=\HEIGHT, width=\WIDTH]{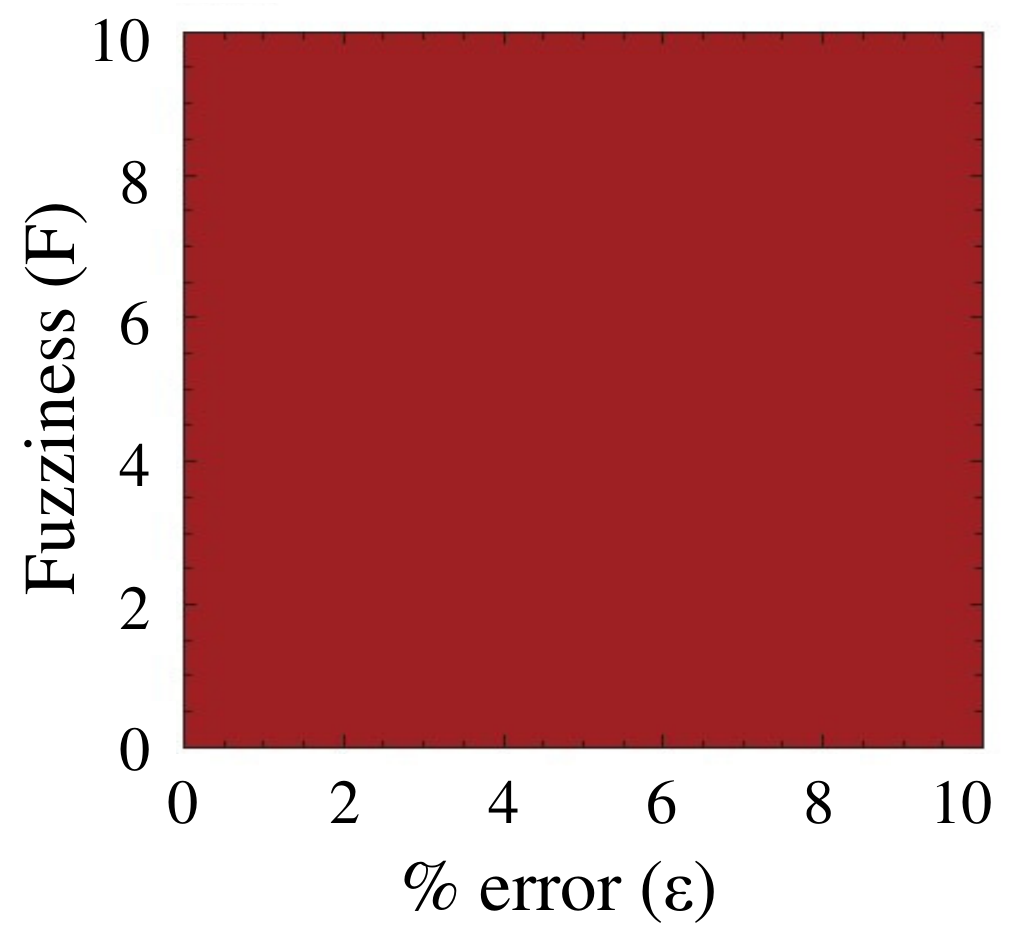}
  	\label{Expr:artificial:prob-cliff-0.07x0.07}
  }

  \caption{
  	Probability of an artifact \FLC\ for various $\MI$ and $\MJ$, in a matrix of $[1000 \times 1000]$.
  	The red colored sections (bottom left area) represent a probability of 0.0.
  	The light yellow colored sections (upper right area) represent a probability of 1.0.
  	An interesting fact is the existence of a ``{\CLIFF}'', where probabilities rapidly climb from 0.0 to 1.0,
    and its withdraw as $\MI$ and $\MJ$ increase.
  	From the figures, we see that in this case a {\FLC} of size greater than 0.5\% of the matrix size
  	has an insignificant probability to randomly appear.	  	
  }
  \label{Expr:artificial:prob-cliff-change-I-J}
\end{figure*}
\begin{figure*}
  \centering

  \subfloat[$\SW$=1\%, $\MF$=1.]
  {
    \hspace{\HSPACE}
  	\includegraphics[clip=true,trim=100 380 220 140,height=\HEIGHT, width=\WIDTH]{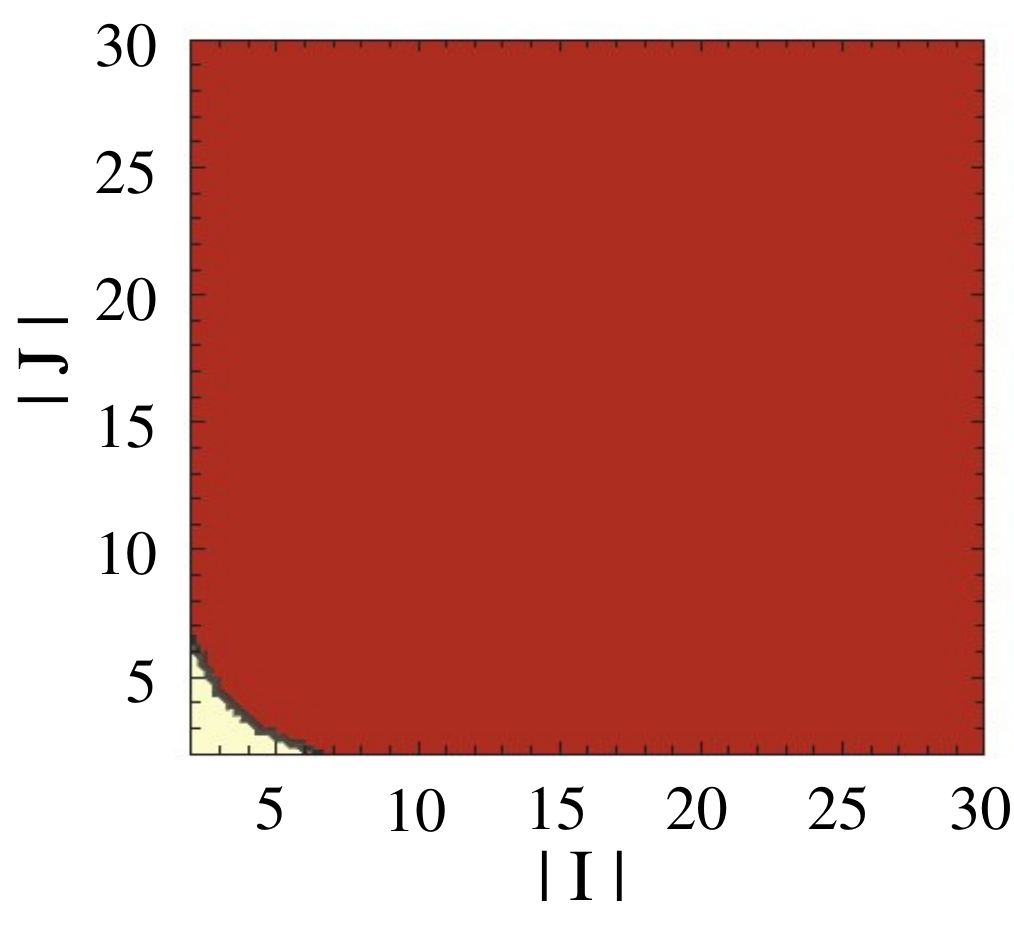}
  	\label{Expr:artificial:expr-f1-e1}
  }
  \subfloat[$\SW$=1\%, $\MF$=2.]
  {
    \hspace{\HSPACE}
  	\includegraphics[clip=true,trim=100 380 220 140,height=\HEIGHT, width=\WIDTH]{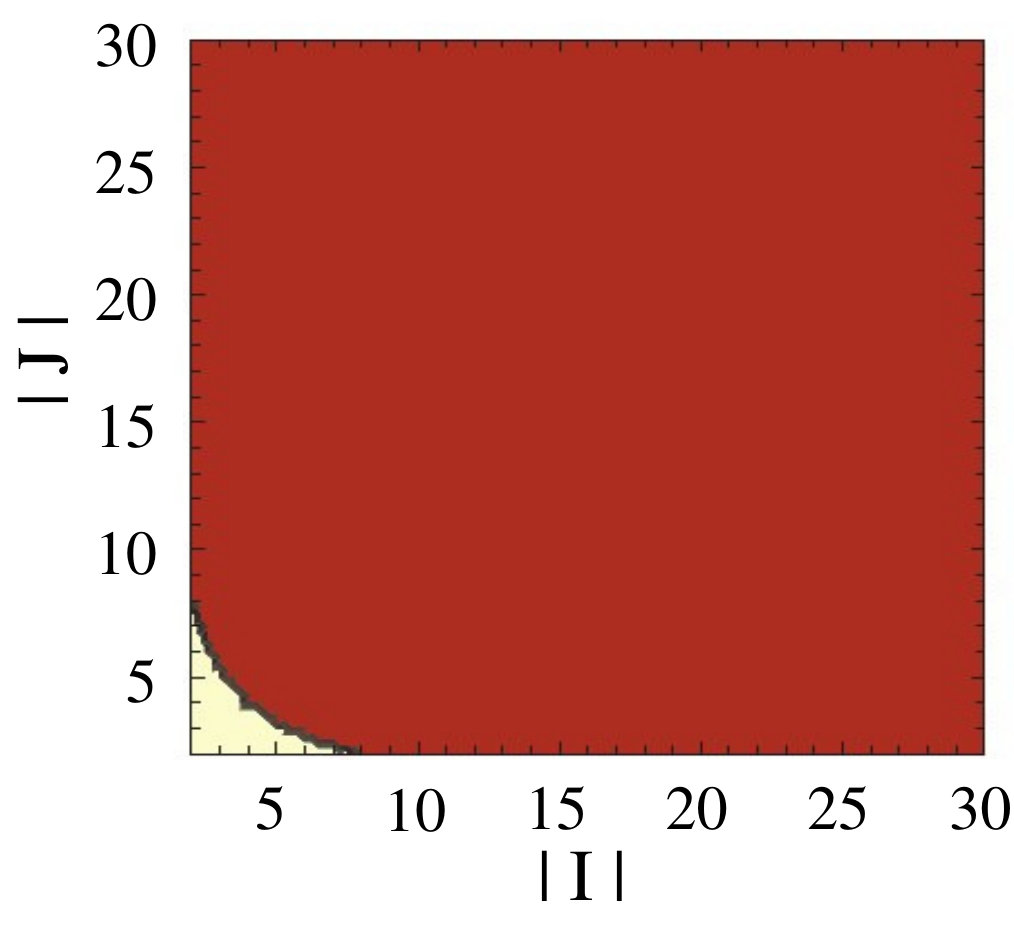}
  	\label{Expr:artificial:expr-f2-e1}
  }
  \subfloat[$\SW$=1\%, $\MF$=4.]
  {
    \hspace{\HSPACE}
  	\includegraphics[clip=true,trim=100 380 220 140,height=\HEIGHT, width=\WIDTH]{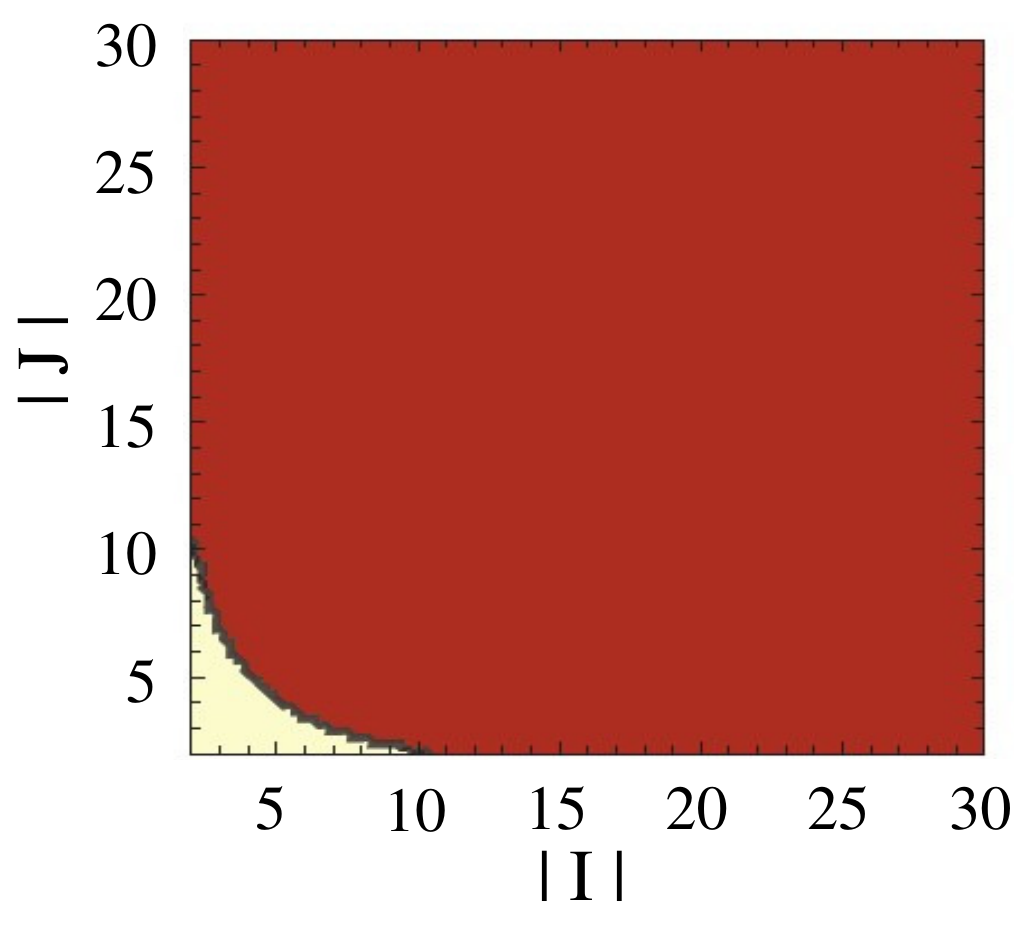}
  	\label{Expr:artificial:expr-f4-e1}
  }

  \subfloat[$\SW$=5\%, $\MF$=1.]
  {
    \hspace{\HSPACE}
  	\includegraphics[clip=true,trim=100 380 220 140,height=\HEIGHT, width=\WIDTH]{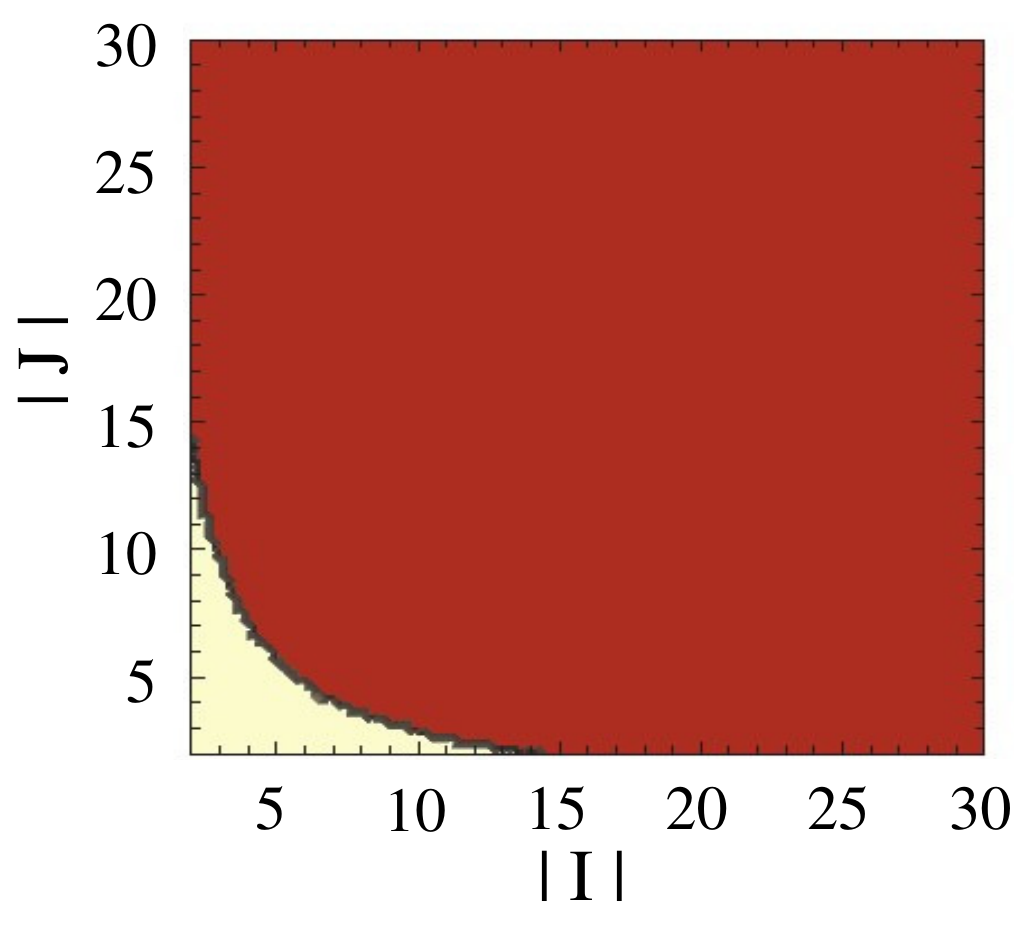}
  	\label{Expr:artificial:expr-f1-e5}
  }
  \subfloat[$\SW$=5\%, $\MF$=2.]
  {
    \hspace{\HSPACE}
  	\includegraphics[clip=true,trim=100 380 220 140,height=\HEIGHT, width=\WIDTH]{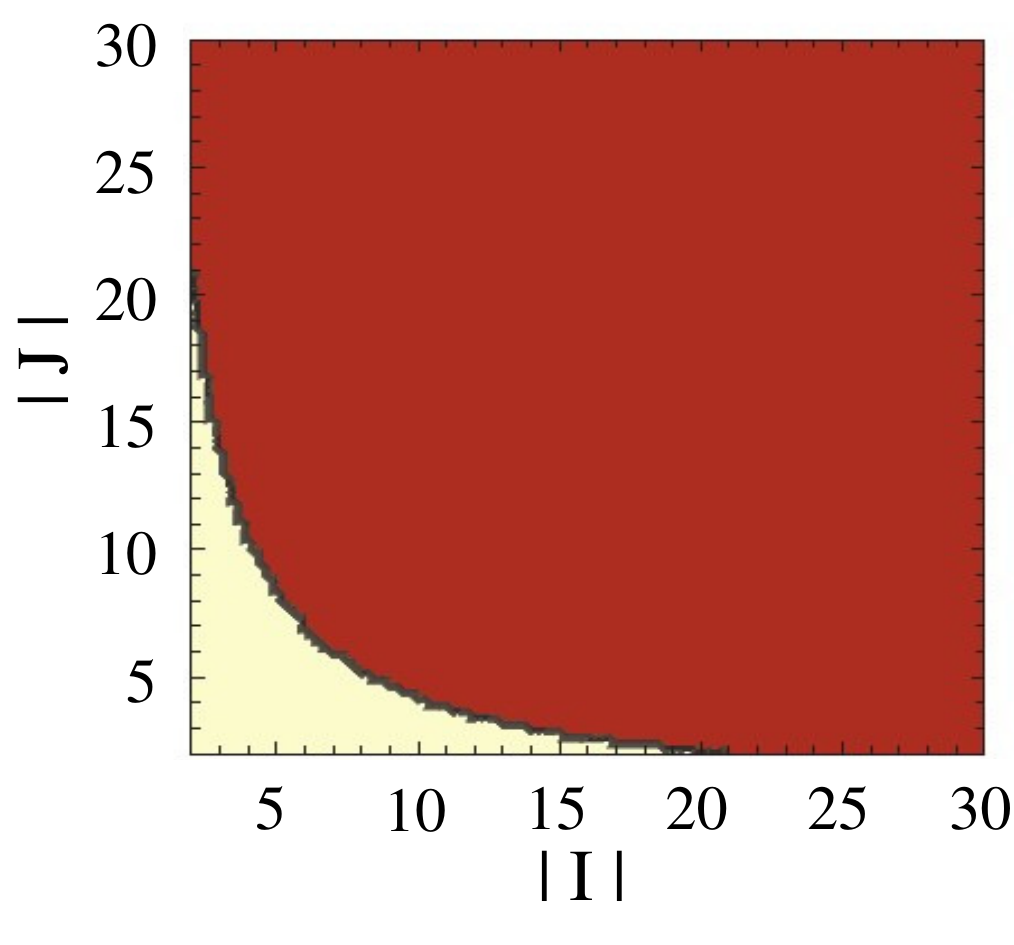}
  	\label{Expr:artificial:expr-f2-e5}
  }
  \subfloat[$\SW$=5\%, $\MF$=4.]
  {
    \hspace{\HSPACE}
  	\includegraphics[clip=true,trim=100 380 220 140,height=\HEIGHT, width=\WIDTH]{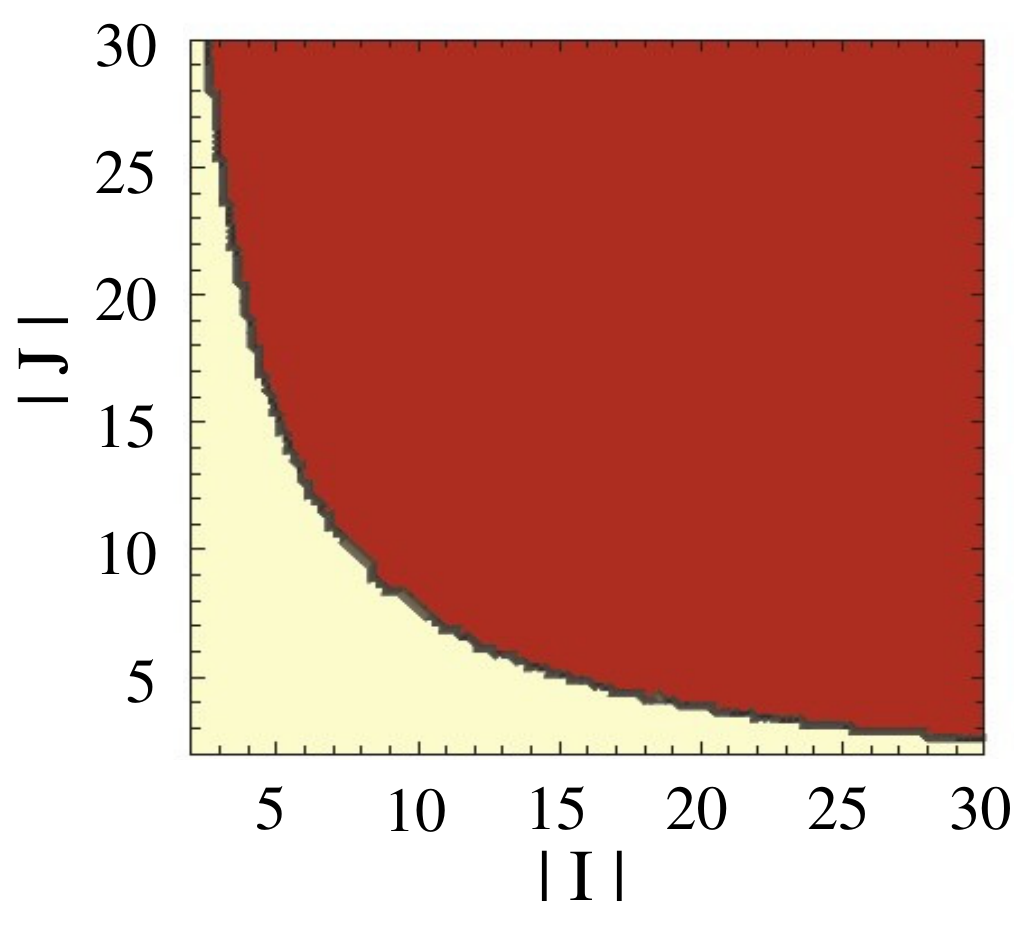}
  	\label{Expr:artificial:expr-f4-e5}
  }

  \caption{
  	Probability of an artifact \FLC\ for various $\SW$ and $\MF$, in a matrix of $[1000 \times 1000]$.
  	The red colored sections (upper right area) represent a probability of 0.0.
  	The light yellow colored sections (bottom left area) represent a probability of 1.0.
  	The figures show the existence of the ``{\CLIFF}'', where probabilities fall from 1.0 to 0.0,
  	and its withdrawal as $\SW$ and $\MF$ increase.
  }
  \label{Expr:artificial:prob-cliff-change-f-e}
\end{figure*}
The main conclusion based on the figures is that
{\FLC}s of small dimensions (i.e., clusters smaller than 0.5\% of the matrix size)
already have an insignificant probability of being caused by a random formation and presenting artifact patterns.
Thus, {\FLC}s representing a regulatory mechanism, which are naturally large in dimensions, have an insignificant probability of being noise.
Consequently, ordinary mining using practical dimensions has an insignificant probability of mining artifacts.

\subsubsection*{{\ExprArtificialDiscrSetSize}: Discriminating Set Size}
\label{Expr:artificial:discr-set-size}

\def \HEIGHT    {0.255\textheight} 

\nin
\THEOREM~\ref{proofs:theorem:d3} provides us with the following bound for the discriminating set:
$|\DS| \geq \log(4mn)/\log(1/3\MJT(2\MF+1))$,
where $\MJT$ specifies the ratio between the number of columns in an optimal {\FLC} and the number of matrix columns.
The fact that the bound depends on $\MJT$ is undesirable,
since this parameter is not part of the problem input and the user has no knowledge about it.
To get a sense of the magnitude of feasible values for $|\DS|$, we conducted the following experiment.
We first created random $[m \times n]$ matrices, with sizes ranging from $10^2$ to $10^5$
and values uniformly distributed in the range of $100$ to $1100$.
We set the dimensions of the cluster size to $\MI,\MJ \in \{0.3, 0.5, 0.8\}$.
Then, we generated a random \FLC\ of error $\SW$=1\%, fuzziness $\MF$=1, size $[\MI m \times \MJ n]$ and put it at a random location in the matrix overriding the existing values.
Then, a size - $k$ subset of the \FLC\ columns was chosen at random $100$ times, to check whether it is a discriminating  set according to \DEFINITION~\ref{proofs::discriminating}.
This process was repeated for $k$=1,$\ldots$ until reaching a value for $k$ which the subset successfully  discriminated in all of the $100$ trials.
We repeated the above procedure for various sizes of $\SFZ$ in order to examine the effectiveness of $\SFZ$ in reducing the size of $\DS$ (see \SUBSECTION~\ref{algorithm:impl-note-SFZ}).

\FIGURE~\ref{Expr:artificial:discr-set-size-vs-sfz} presents the results of extensive experimentation relating to the trade-off between the size of the discriminating column set $|\DS|$ as a function of $\log_2(4mn)$ for various $|\SFZ|$.
\begin{figure}
	\centering
	\includegraphics[clip=true,trim=70 96 182 92,height=\HEIGHT]{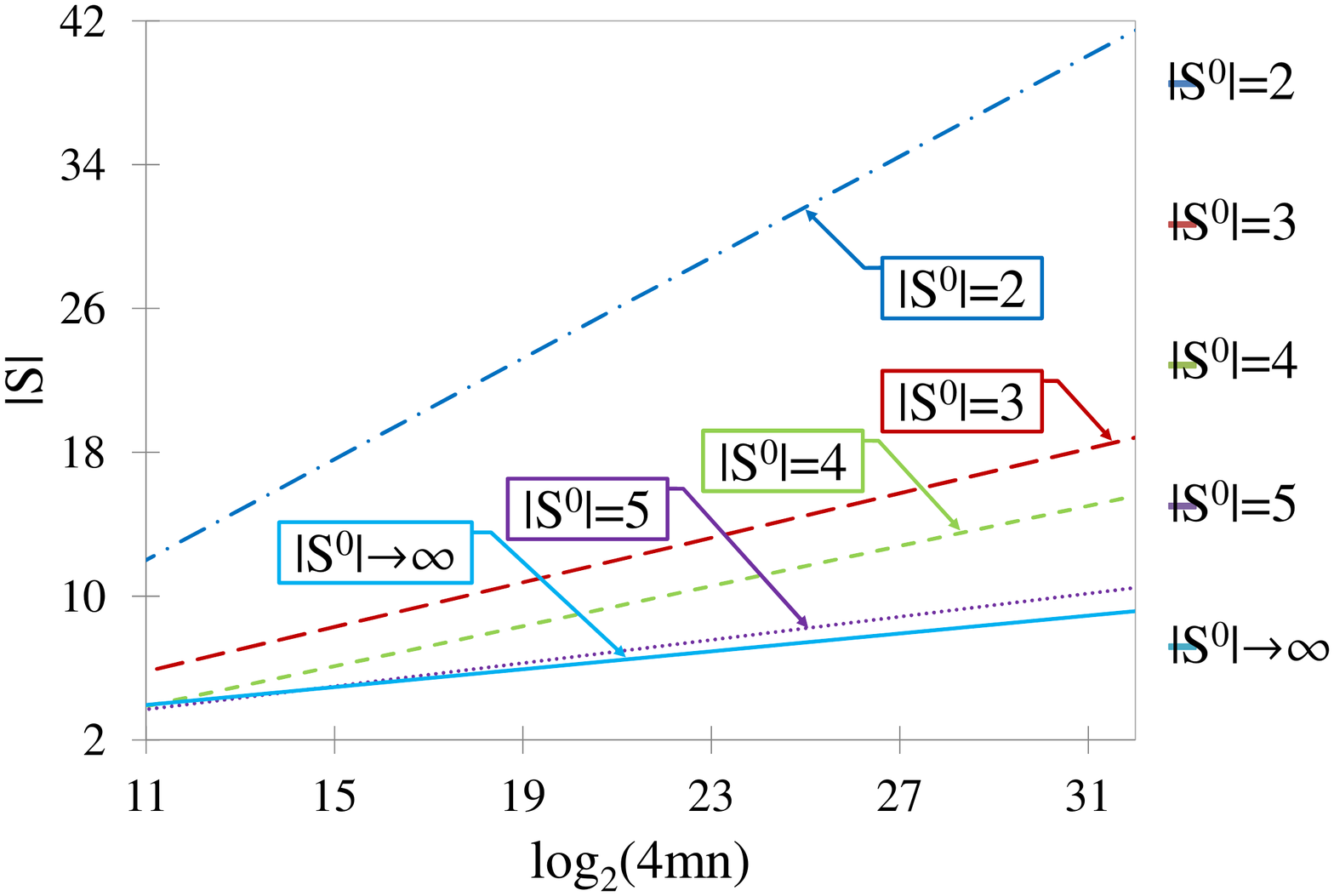}
	\caption{
		The size of the discriminating column set $|\DS|$ as a function of $\log_2(4mn)$ for various $|\SFZ|$.
		The lower line of $|\SFZ|$$\rightarrow$$\infty$ represents the case of mining lagged clusters
		with \emph{no} fuzziness ($\MF$=0).
	}
	\label{Expr:artificial:discr-set-size-vs-sfz}
\end{figure}
The following important observations can be made from \FIGURE~\ref{Expr:artificial:discr-set-size-vs-sfz}:
(1) for each $|\SFZ|$ value used, we obtain the linear relationship derived from \THEOREM~\ref{proofs:theorem:d3};
(2) the decrease in size of the discriminating set $\DS$ is proportional to the size of $\SFZ$, i.e., the larger $|\SFZ|$ used, the lower $|\DS|$ needed.
Setting $|\SFZ|$=3 seems to be the most effective in this case, as it offers a good balance between run-time reduction and the resultant limitation of the model (i.e., assuming $|\SFZ|$ non-fuzzy columns);
and (3) we obtain an easy-to-use, $\MJT$ free, formula for setting $|\DS|$. For example, using $|\SFZ|$=3 we obtain:
$|\DS| = 0.6197 \log_2(4mn) -1.0063 \approx 0.6 \log_2(4mn) -1$.

\subsubsection*{{\ExprArtificialDiscrProbabilities}: Discriminating Probability vs. Discriminating Set Size}
\label{Expr:artificial:discr-probabilities}

\nin
The previous experiment considered a set of size $|\DS|$ to be discriminating if it successfully discriminated in all $\LOOPS$ trials ($\LOOPS$=100).
In this experiment, we wish to explore the relationship between $|\DS|$ and its discriminating probability (i.e., in how many of the $\LOOPS$ trials did the set actually discriminate).
We do so by recording different sizes of $|\DS|$ and their ability to discriminate.
The experiment was conducted using the same methodology as {{\EXPERIMENT} II}, using $|\SFZ|$=3 as suggested.

\FIGURE~\ref{Expr:artificial:disct-size-hit} presents
the discriminating probability as a function of the discriminating set size $|\DS|$.
\begin{figure} [htb]
	\centering
	\includegraphics[clip=true,trim=66 66 64 72,height=\HEIGHT]{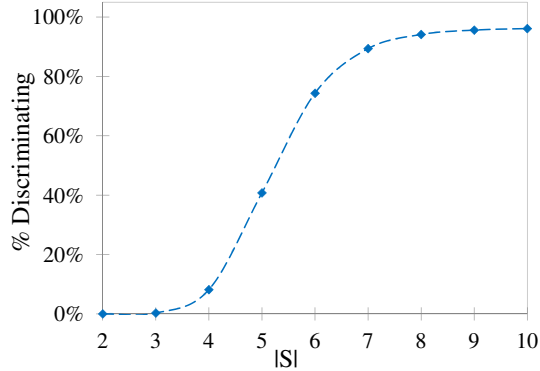}
	\caption{
		The discriminating probability as a function of the discriminating set size $|\DS|$.
		To facilitate reading, we only present the average results over $\MI,\MJ \in \{0.3, 0.5, 0.8\}$,
		as the results for the specific settings were of insignificant difference.
	}
	\label{Expr:artificial:disct-size-hit}
\end{figure}
The main finding from \FIGURE~\ref{Expr:artificial:disct-size-hit} is that even for small sizes of $|\DS|$,
a substantial discriminating probability is achieved (e.g., 89\% for $|\DS|$=7).
Since $|\DS|$ appears as an exponent in the estimated run-time, choosing a smaller $|\DS|$ will have a notable effect on the reduction of run-time, without having any major negative effects on the results.

\subsubsection*{{\ExprArtificialLowDiscrSizeMoreIterations}: Discriminating Set Size vs. Number of Iterations Needed}
\label{Expr:artificial:low-discr-size-more-iterations}

\nin
\THEOREM~\ref{proofs:theorem:rc} shows that the probability of the {\FLCA} algorithm to mine a \FLC\ is at least 0.5.
We refer to that probability as a ``hit rate''.
The hit rate depends on the discriminating probability $p$, of the discriminating set $\DS$,
and the number of iterations $\LOOPS$ being used: (hit rate) $= 1 -$ (miss rate) $= 1 -(1-p\MI\MJ^{|\DS|})^{\LOOPS}$.
Using \THEOREM~\ref{proofs:theorem:rc}, the number of iterations is $\LOOPS$=$2 \ln2 / \MI \MJ^{|\DS|}$,
resulting in a hit rate of: $1-0.25^p$, i.e.,
(hit rate) $= 1-(1-p\MI\MJ^{|\DS|})^{\LOOPS} = 1-(1-p\MI\MJ^{|\DS|})^\frac{2 \ln2}{\MI\MJ^{|\DS|}} = 1-(1-p\MI\MJ^{|\DS|})^\frac{p \cdot 2 \ln2}{p \cdot \MI\MJ^{|\DS|}} \geq 1-\frac{1}{e^{2p\ln2}} = 1-\frac{1}{2^{2p}} = 1-0.25^p$.\footnote{
\THEOREM~\ref{proofs:theorem:rc} uses \THEOREM~\ref{proofs:theorem:d3} discriminating sets of $p$=0.5 and thus results in a hit rate of 0.5.}

{\ExprArtificialDiscrProbabilities} implies
that using a discriminating set smaller than the one recommended by {\ExprArtificialDiscrSetSize} will not only exponentially \textit{decrease} the run-time, but will also ensure a reasonable discriminating probability.
However, a decrease in the discriminating set size results in a decrease in the hit rate.
Therefore, in order to improve the hit rate, an increase in the number of iterations is required.
In practice, by reducing the discriminating set size, it is possible to reduce the run-time by more than the increase needed to ensure the desired hit-rate.

We next present an analysis aimed at finding the best setting to achieve a minimum run-time.
As a base line, we use {\ExprArtificialDiscrSetSize} discriminating sets,
with a discriminating probability of $\approx$100\% and a size of 8.5 for matrices of size $[100 \times 100]$.
Such sets will yield a hit rate of $\approx$75\%.

Based on the average curve, shown in \FIGURE~\ref{Expr:artificial:disct-size-hit}, the discriminating probabilities of $|\DS|$=\{4, 5, 6, 7, 8, 9\} are $P_{|\DS|}$=\{8.2\%, 40.8\%, 74.3\%, 89.4\%, 94.1\%, 95.6\%\}, respectively.
In order to reach a hit rate of 75\%,
we need to compensate for the loss in the above discriminating probability by increasing the number of iterations.
Note that when the number of iterations $\LOOPS$ is multiplied by $C$=$1/p$, the resulting hit rate becomes:
	$1-(1-p\MI\MJ^{|\DS|})^{(\frac{1}{p} \cdot \LOOPS)}$ =
	$1-(1-p\MI\MJ^{|\DS|})^\frac{2\ln2}{p \MI\MJ^{|\DS|}}$ $\geq$
	$1-\frac{1}{e^{2\ln2}}$ =
	$1-\frac{1}{2^{2}}$ =
	$0.75$.
Therefore, by factoring the number of iterations
${\LOOPS}(|\DS|)$=$2 \ln{2} /(\MI \MJ^{|\DS|})$
by the inverse of the discriminating probability,
we obtain for $|\DS|$=\{4, 5, 6, 7, 8, 9\},
a ${C}_{|\DS|}$=\{1/8.2\%, 1/40.8\%, 1/74.3\%, 1/89.4\%, 1/94.1\%, 1/95.6\%\}=\{12.24, 2.45, 1.35, 1.12, 1.06, 1.05\},  respectively.
\FIGURE~\ref{Expr:artificial:discr-size-iterations} depicts
the ratio between ${C}(|\DS|)$$\times$${\LOOPS}(|\DS|)$ and the base line ${\LOOPS}(8.5)$ for various $|\DS| \in [4-9]$.
The ratio is independent of $\MI$ and equal to ${C}(|\DS|) \times {\MJ}^{(8.5-|\DS|)}$.
The lower the ratio, the better the performance as less iterations are required.
\begin{figure}
	\centering
	\includegraphics[angle=0, clip=true,trim=50 84 62 80,height=\HEIGHT]{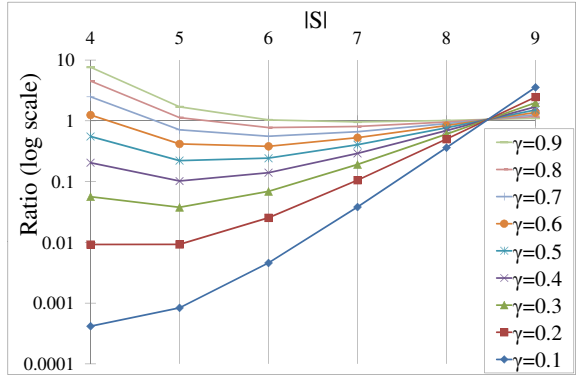}
	\caption{
		The ratio of iterations required for various $|\DS|$ to reach a 75\% hit rate base line.
		The lower the ratio, the better the performance (i.e., less iterations required).
	}
	\label{Expr:artificial:discr-size-iterations}
\end{figure}
The main finding of the experiment is the ability to use discriminating sets with lower discriminating probability, compensated by a larger number of iterations to achieve the same hit rate levels while having lower run-time.
An example from \FIGURE~\ref{Expr:artificial:discr-size-iterations} is the preferable use of $|\DS|$=5 for $\MJ \leq$ 0.5 and $|\DS|$=6 for $\MJ \geq$ 0.6 over setting $|\DS|$=8.

\subsubsection*{{\ExprArtificialRunTimeHitRate}: Run-time, Number of Iterations and Hit Rate}
\label{Expr:artificial::Run-Time-Hit-Rate}

\nin
\THEOREM~\ref{proofs:theorem:rc} states that for any
given discriminating set with a discriminating probability of 0.5 (see \THEOREM~\ref{proofs:theorem:d3}) and
for $\LOOPS \geq 2 \ln{2} /(\MI \MJ^{|\DS|})$ trials,
we are guaranteed to find a factor 2 optimal {\FLC} with a probability of at least $0.5$.
We report on the \textit{actual} performance of the algorithm in mining an optimal {\FLC} in terms of those three parameters.

The experiment was conducted
by creating a random matrix of size $[100 \times 100]$ and randomly placing random {\FLC}s of varying sizes (i.e., $\MI,\ \MJ \in \{0.3, 0.5, 0.8\}$) overriding the original values.
To obtain sets with a discriminating probability of 0.5,
we set $|\DS|$=5 with a discriminating probability of 40.8\% (see {\ExprArtificialDiscrProbabilities}).
While repeating the execution of the algorithm 10,000 times for each cluster size $\MI,\ \MJ \in$ \{0.3, 0.5, 0.8\}, we counted: 

(1) hit rate: how many times out of the 10,000 repetitions the algorithm managed to mine the planted cluster;
(2) iterations: how many iterations it took in practice to mine the optimal cluster;
and (3) run-time: how long (in minutes) it took to mine the optimal cluster.
The experiment was conducted using the platform:
Intel core i7 @ 2.00GHz CPU with 6GB RAM, Windows 7 64 bit. The algorithm was programmed in Java 7.0.
The results obtained are as follows.
\begin{itemize}

	\item \textbf{Hit Rate:}
	An actual average hit rate of 44.0\%, higher than the expected hit rate of 43.2\%
	for a discriminating set
	with a discriminating probability of 40.8\%.\footnote{
		Following {\ExprArtificialLowDiscrSizeMoreIterations} formula of: hit rate = $1-0.25^p$,
		a discriminating probability of $p$=40.8\%,
		results in an expected hit rate of 43.2\%.	
	}
		
	\item \textbf{Number of Iterations:}
	\begin{figure}
		\centering
		\includegraphics[clip=true,trim=62 90 80 100,height=\HEIGHT]{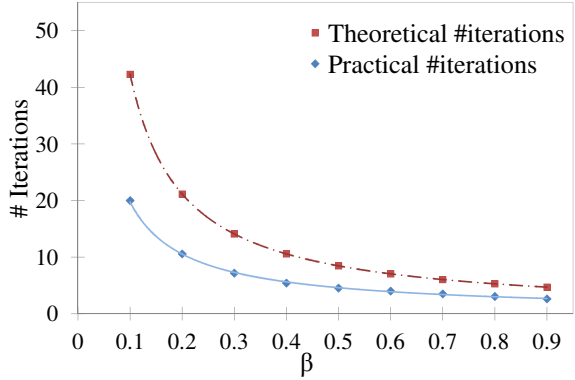}
		\caption{
			Number of iterations required to mine a \FLC\ vs. the theoretical bound.
			We present here only the results of $\MJ$=0.8,
			as the results for the other $\MJ$ values were insignificantly different.
			Since both $\MJ$ and $|\DS|$ were held fixed at $0.8$ and $5$, respectively,
			both theoretical and practical situations present a behavior of $\LOOPS$=${\cal O}(1/\MI)$.
		}
		\label{Expr:artificial::number-of-iter}
	\end{figure}	
	\FIGURE~\ref{Expr:artificial::number-of-iter} presents the actual number of iterations needed to
	mine the optimal {\FLC}
	in relation to the theoretical boundary.
	On average, the actual number of iterations needed is 49\% of the specified theoretical bound.
	
	\item \textbf{Run-time:}
		The run-time boundary, as specified by \SUBSECTION~\ref{subsection:algorithm:Run-Time} is:
		$t$ = ${\cal O}((m n)^2  / \MJ^{|\DS|})$.
    Fitting the actual run-time to an equation of type:
    $t$=$c /(\MI^x \MJ^y)$ ($t$ in ms), where $c=(m n)^2$,
    we obtain: $x=0.83$, $y=4.49$ and $c=107.2$.
    As expected, the power of $\MI$ is close to 1 and the power of $\MJ$ is close to 5 (we set $|\DS|$=5).
    These results are better than the theoretical bound due to $\MI, \MJ \leq 1.0$.

\end{itemize}

To summarize, on our test set the {\FLCA} algorithm manages to mine
{\FLC}s with a hit rate higher than the expected hit rate,
less than 50\% of the needed theoretical number of iterations,
and does so within a feasible run-time.

\subsubsection*{{\ExprArtificialErrFuzzEffect}: The Effect of Error and Fuzziness}
\label{Expr:artificial::Err-Fuzz-Effect}

\def \MyRNIA   {$1\!-\!RNIA$}

\nin
The main objective of the previous experiments was to
demonstrate properties of the {\FLC}ing model
while
focusing on the correctness of the theoretical bounds of \ALGORITHM~\ref{algorithm:FLC}.
In this experiment, we wish to examine the extent of changes in the mining results as a function of the error and fuzziness used.
To do so, we planted random {\FLC}s of specific error and fuzziness and set the miner's parameters of error and fuzziness to various values.

To measure how well the miner performed in each setting, we used the complement of the RNIA 
score~\cite{patrikainen2006comparing}, defined as follows.
Let $C_1$ and $C_2$ be {\FLC}s.
$RNIA(C_1, C_2) = (|U| - |I|)/|U|$,
where $U$ and $I$ are the matrix elements in the union and intersection of $C_1$ and $C_2$, respectively.
Hence, \MyRNIA$(C_1, C_2)$=$|I|/|U|$, achieves a score of $1$ when $C_1$ and $C_2$ are equal, and a score of $0$ when completely disjoint.

\def \HEIGHT    {0.2\textheight} 
\def \WIDTH	    {0.48\textwidth} 
\def \HSPACE    {-0pt}
\def \PLANTED   {Planted cluster: }

\begin{figure*}
  \centering

  \subfloat[\PLANTED $\SW$=0.1\%, $\MF$=4.]
  {
    \hspace{\HSPACE}
  	\includegraphics[clip=true,trim=64 64 61 63,height=\HEIGHT, width=\WIDTH]{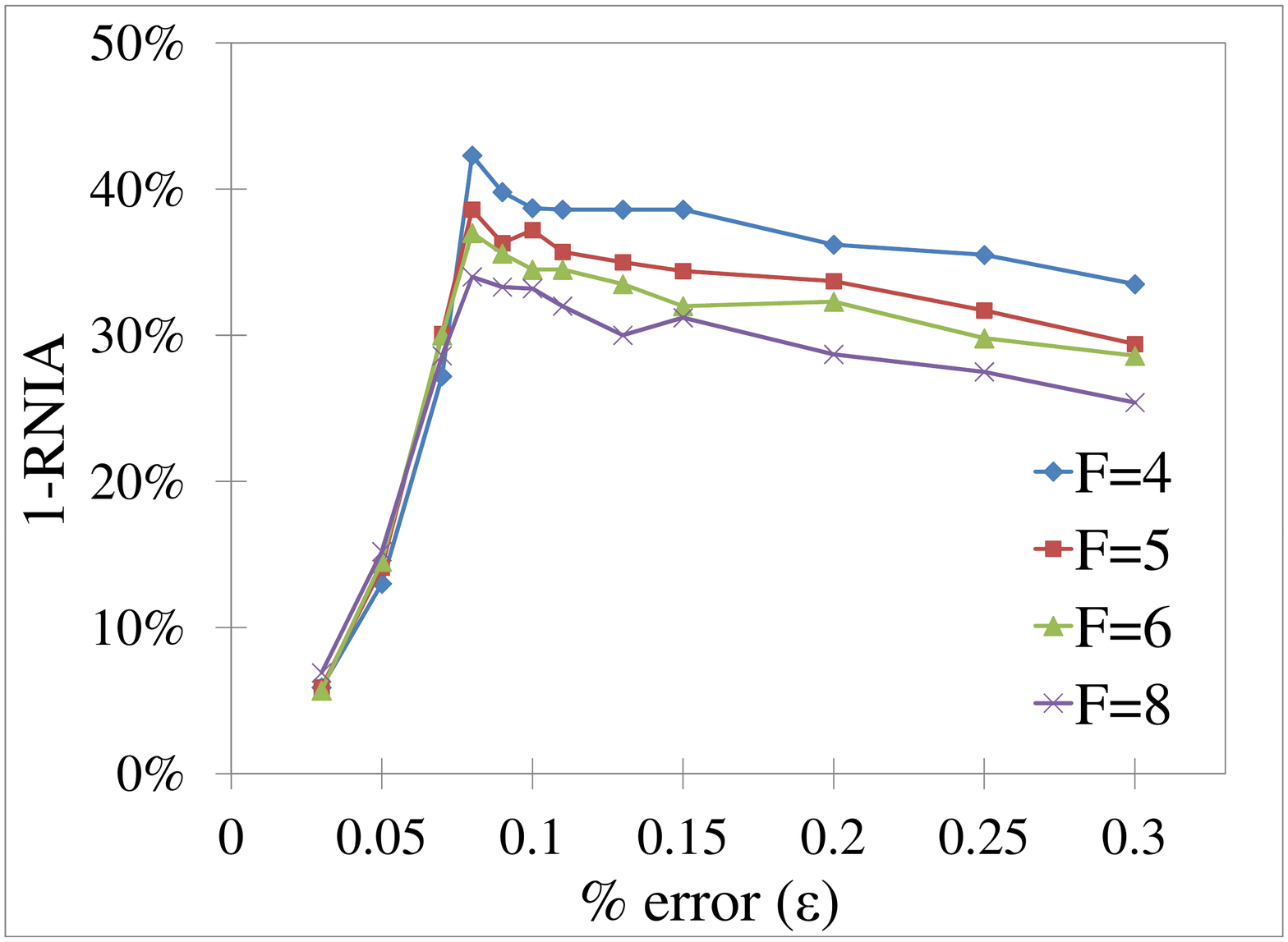}
  	\label{Expr:artificial:Expr816-Err-01-Fuz-4}
  }
  \subfloat[\PLANTED $\SW$=0.01\%, $\MF$=4.]
  {
    \hspace{\HSPACE}
  	\includegraphics[clip=true,trim=61 64 74 63,height=\HEIGHT, width=\WIDTH]{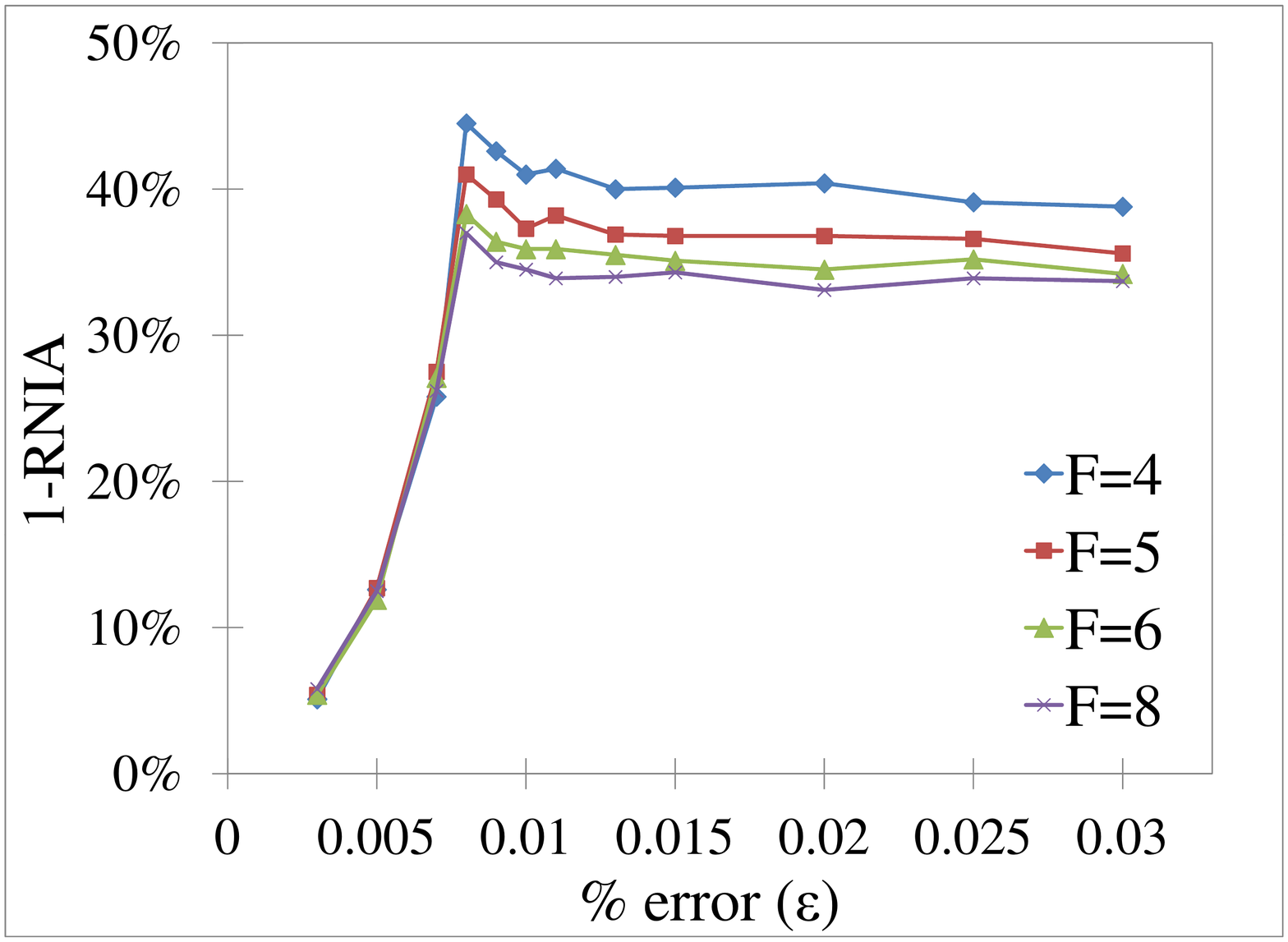}
  	\label{Expr:artificial:Expr816-Err-001-Fuz-4}
  }

  \subfloat[\PLANTED $\SW$=0.1\%, $\MF$=2.]
  {
    \hspace{\HSPACE}
  	\includegraphics[clip=true,trim=64 64 61 62,height=\HEIGHT, width=\WIDTH]{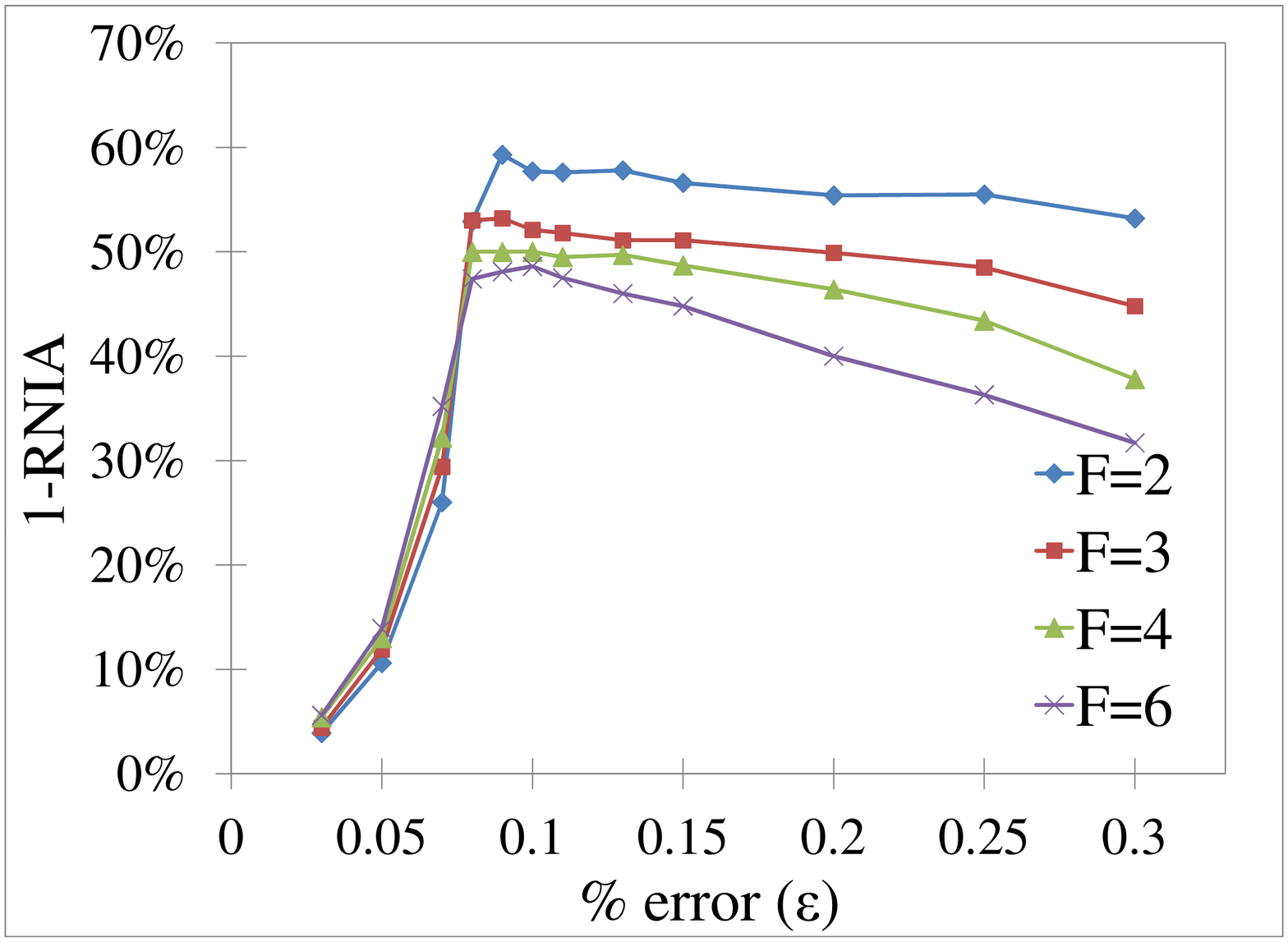}
  	\label{Expr:artificial:Expr816-Err-01-Fuz-2}
  }
  \subfloat[\PLANTED $\SW$=0.01\%, $\MF$=2.]
  {
    \hspace{\HSPACE}
  	\includegraphics[clip=true,trim=61 64 74 62,height=\HEIGHT, width=\WIDTH]{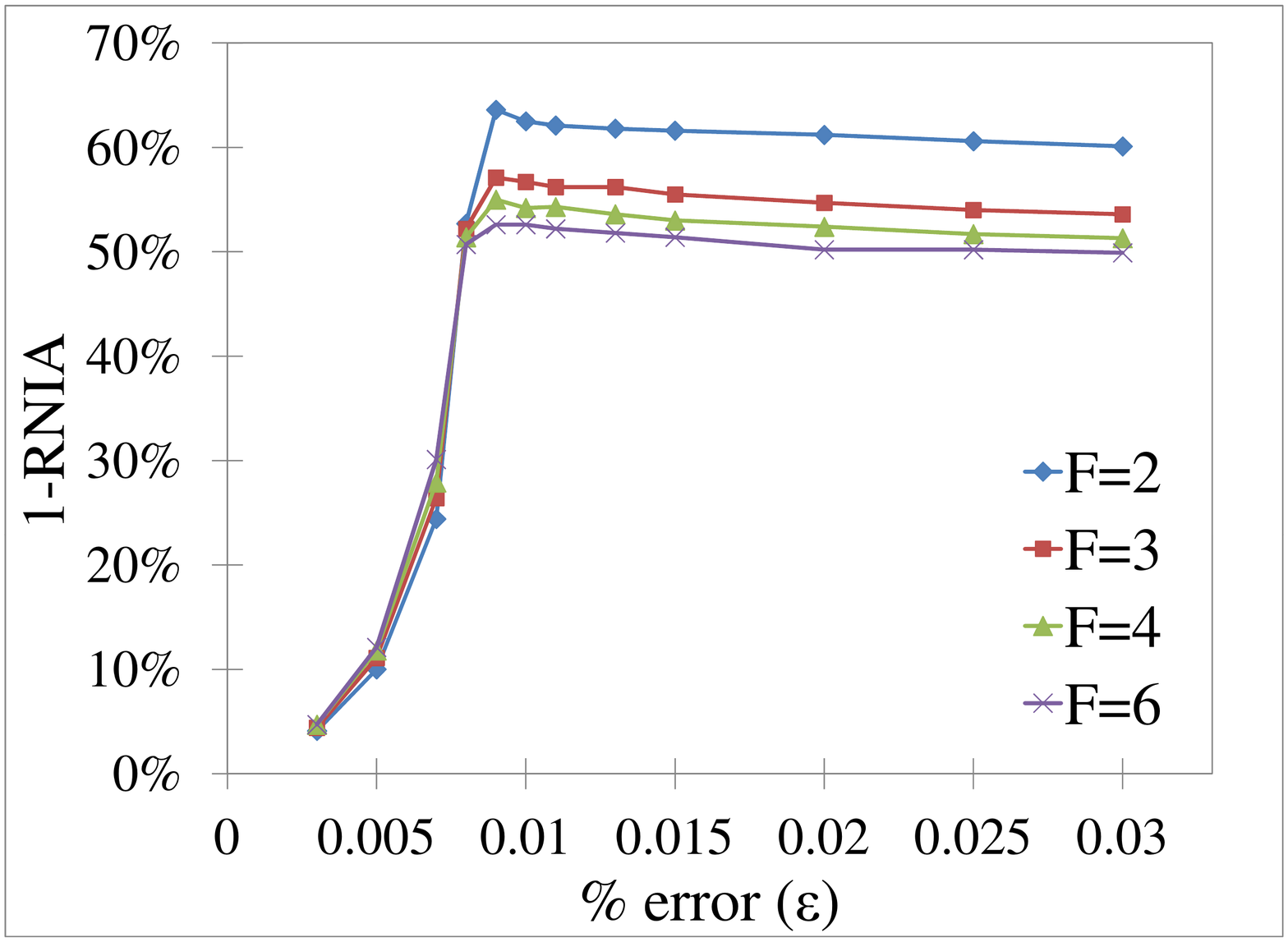}
  	\label{Expr:artificial:Expr816-Err-001-Fuz-2}
  }

  \caption{
  	The miner's capability to capture {\FLC}s as a function of the miner's settings of error and fuzziness
    (measured in terms of: \MyRNIA).
  }

  \label{Expr:artificial:Expr816}
\end{figure*}
\FIGURE~\ref{Expr:artificial:Expr816} depicts the mining performance (measured in terms of: \MyRNIA) of four different combinations of error and fuzziness of the planted clusters ($\SW\!\in\!\{0.01\%, 0.1\%\}$, $\MF\!\in\!\{2, 4\}$) as a function of the miner's configuration of error and fuzziness.
Obviously the cluster in each configuration can be mined only if the fuzziness used by the miner is greater than or equal to the maximal fuzziness allowed when constructing the planted cluster. Therefore the miner's fuzziness configuration was set to $\MF\ge2$ in the case where the planted cluster was of maximum fuzziness of $2$ and $\MF\ge4$ in the case of maximum fuzziness of $4$.
The figure presents the average score over $100$ trials for each of the four combinations of the planted clusters' parameters and for each configuration of the miner.
The graphs demonstrate that, as expected,
the best performances are achieved when the miner is set to
the fuzziness of the planted cluster and to an error within the surrounding of the planted cluster.
In addition, the higher the error or fuzziness to which the miner is set, the lower the performance achieved.
This is due to increasing noise being added to the mined clusters.
%
The charts in \FIGURE~\ref{Expr:artificial:Expr816} are significant as they establish the importance of introducing fuzziness into the lagged-pattern model.



\subsection{Experiments with Flight of Pigeon Flocks} \label{subsec:Expr:birds}

\def \hf {homing flight}
\def \Hf {Homing flight}
\def \HF {Homing Flight}
\def \ff {free flight}
\def \Ff {Free flight}
\def \FF {Free Flight}

\def \ALGFLC    {\FLCANB}
\def \ALGDBSCAN {DBSCAN}

\def \FOne   {{F$_1$}}
\def \FScore {{\FOne\ score}}

\def \HEIGHT    {0.2\textheight} 
\def \WIDTH	    {0.48\textwidth} 

In a second set of experiments, we examined the capability of the {\FLCA} algorithm to mine clusters from real-life data.
One key goal for these experiments was to demonstrate the extent of improvement achieved in terms of mining
coherency when transitioning from the lagged model to the fuzzy lagged model.
For that purpose, we used two real-life datasets containing GPS readings\footnote{
\label{subsec:Expr:birds:GPS:x-y-coordinates}
Of the GPS readings, only the $x$ and $y$ coordinates were used.
This is due to the error of the $z$-coordinate which is
much larger than those of the horizontal directions \cite{nagy2010hierarchical}.
}
of the flight of pigeon flocks  \cite{nagy2010hierarchical} (see a snapshot in  \FIGURE~\ref{Expr:pigeon-hf-interleaving}):
(1) \hf\ data, consisting of four different datasets recording the flights of pigeons from point A to point B;
and (2) \ff\ data, consisting of 11 different datasets recording the flights of pigeons around the home loft,
i.e., flight from point A back to point A.
Each dataset (four of \hf\ and 11 of \ff) represents a different \textit{flock} release, containing an average of nine individuals.
\begin{figure}
  \centering

  \subfloat[\hf] {
  	\includegraphics[clip=true,trim=58 67 62 70,height=\HEIGHT, width=\WIDTH]{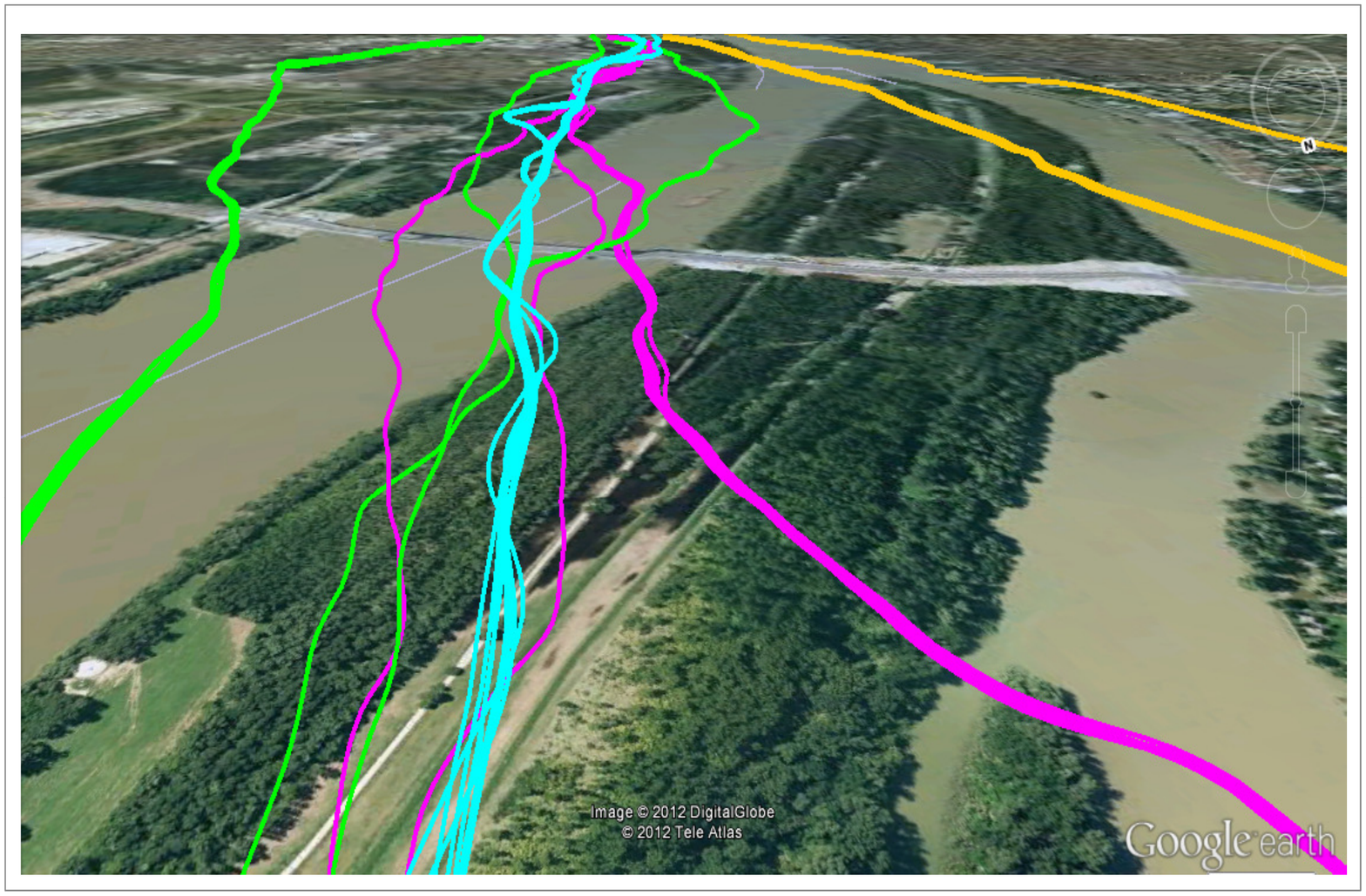}
  	\label{Expr:pigeon-interleaving-hf}
  }
  \subfloat[\ff\ and \hf] {
  	\includegraphics[clip=true,trim=88 427 474 50,height=\HEIGHT, width=\WIDTH]{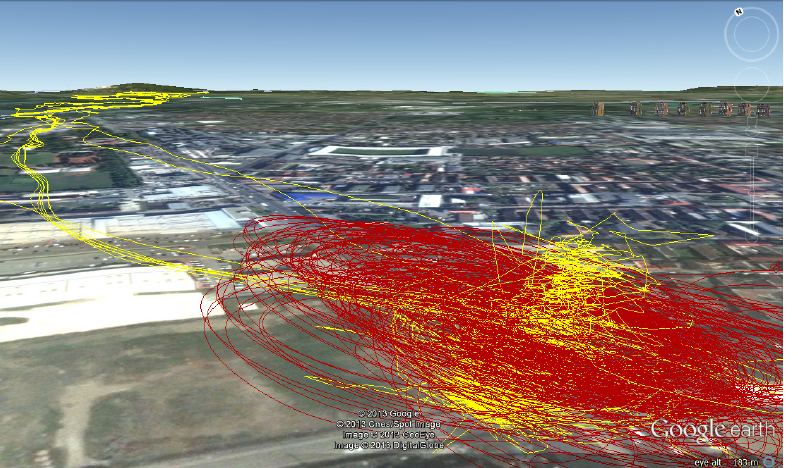}
    \label{Expr:pigeon-interleaving-hf-ff}
  }

  \caption{
   	Snapshot of the \hf\ and \ff\ datasets.
   	Each line represents a pigeon's trajectory.
   	Pigeons belonging to the same flock are painted in the same color.
   	The figures illustrate the presence of interleaving trajectories
   	and therefore the difficulty in mining clusters
   	which \emph{\textbf{only}} contain pigeons of the same flock.
   	These datasets present a serious challenge to mining algorithms
   (e.g., density-based algorithms such as \ALGDBSCAN~\cite{ester1996density},
   see \EXPERIMENT~\ref{subsubsec:Expr:birds:hf} and \ref{subsubsec:Expr:birds:hf.vs.ff}),
   as well as to humans (see \EXPERIMENT~\ref{subsubsec:Expr:birds:hf:humans}).
  }
  \label{Expr:pigeon-hf-interleaving}
\end{figure}

Generally speaking, a flock's flight formation is a lagged pattern where the lag is the distance between the fliers.
Nevertheless, clustering the pigeons according to their flock membership is not a trivial task.
The trajectories of the flock members depend on multiple parameters:
flier (e.g., physical ability, navigation capabilities, leader-follower relationships, threats)
and weather conditions (e.g., wind streams, temperature).
Many of these parameters change dramatically over time and space. The data are inherently noisy due to human error and equipment inaccuracy (e.g., GPS errors$/$inaccuracies$/$distortion, loss of signal, device failure).
A further complication in the dataset is that flight trajectories are spatially close and highly interleaved.
This is caused by the fact that flocks were all released (at different times) from a similar location heading to the same destination.
For example, the \hf\ pigeons followed the Danube river for about 15km until reaching their loft.
This lack of spatial differentiation imposes a great mining challenge, especially to density-based algorithms, which might mistakenly merge trajectories that belong to different flocks.
Therefore, mining such datasets for {\FLC}s is highly complex.

In the following experiments,
we consider a cluster to be accurate if \emph{all} the participating pigeons belong to the \emph{same} flock.
We note that it is unlikely to mine a cluster containing all pigeons in the flock as it is fairly common for pigeons
to deviate from their flock for a substantial period of the flight (e.g., during flight no.3, two birds broke away from the group
soon after release).
Thus, such deviating pigeons cannot be accurately clustered.

\subsubsection{Mixed Datasets: \HF\ and \FF} \label{subsubsec:Expr:birds:hf.vs.ff}

The goal of the experiment is to test the \textit{error} and \textit{fuzziness} impact on the precision and recall of the mined clusters.
To do so, we use a dataset compiled from mixed pairs of a \hf\ and a \ff\ dataset
(we compile $44$ different pairs of datasets which are the result of four {\hf}s and $11$ {\ff}s dataset combinations,
see example in \FIGURE~\ref{Expr:pigeon-interleaving-hf-ff}).
We ran the {\FLCA} algorithm on each pair of datasets with
various errors $\SW$$\in$[0.005\%--0.5\%] 
and fuzziness $\MF$$\in$[0--10] 
combinations, recording the {\FScore}\footnote{
	\textbf{\FOne} score (also known as \textbf{F-measure}) is defined as:
	\FOne $= 2\cdot (precision \cdot recall) / (precision+recall)$ \cite{van1979information}.
	In terms of Type-I and type-II errors:
	\FOne $= (2\cdot true\ positives) / (2\cdot true\ positives + false\ negatives + false\ positives)$.
} of the mined clusters.
In addition, we ran the \ALGDBSCAN\ algorithm~\cite{ester1996density} in order to compare its results to the \ALGFLC\ algorithm.
The \ALGDBSCAN\ algorithm has two main parameters:
(1) distance, denoted $Eps$, which represents the maximum neighborhood of a point;
and (2) density, denoted $MinPts$, which represents the minimum number of points within the neighborhood of a point.
The \ALGDBSCAN\ algorithm was run on each pair of datasets,
with various combinations of distance $Eps$$\in$[0.001--10000] and density $MinPts$$\in$[2--10000],
recording the \FScore\ of the mined clusters.

\def \FHeight  {5.0cm}

\def \HEIGHT    {0.2\textheight} 
\def \WIDTH	    {0.48\textwidth} 
\def \ALG	    {} 
\begin{figure}
  \centering

  \subfloat[\ALGFLC\ALG: \FOne\ vs. \%error ($\SW$).] { 
  	\includegraphics[clip=true,trim=68 90 114 104,height=\HEIGHT, width=\WIDTH]{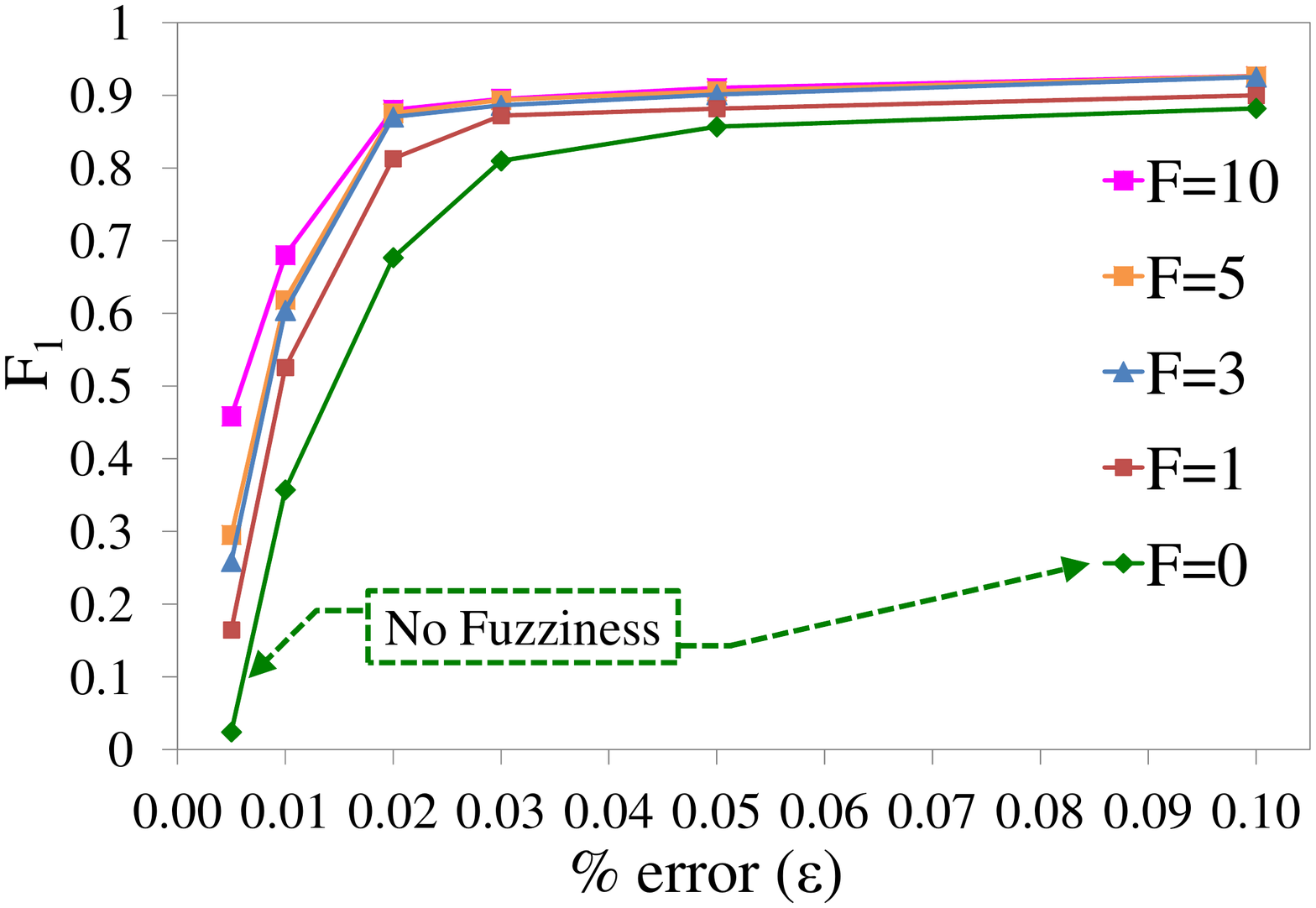}
  	\label{Expr:pigeon-hf-ff-err-impact}
  }
  \subfloat[\ALGDBSCAN\ALG: \FOne\ vs. distance ($Eps$).] { 
  	\includegraphics[clip=true,trim=64 60 88 79,height=\HEIGHT, width=\WIDTH]{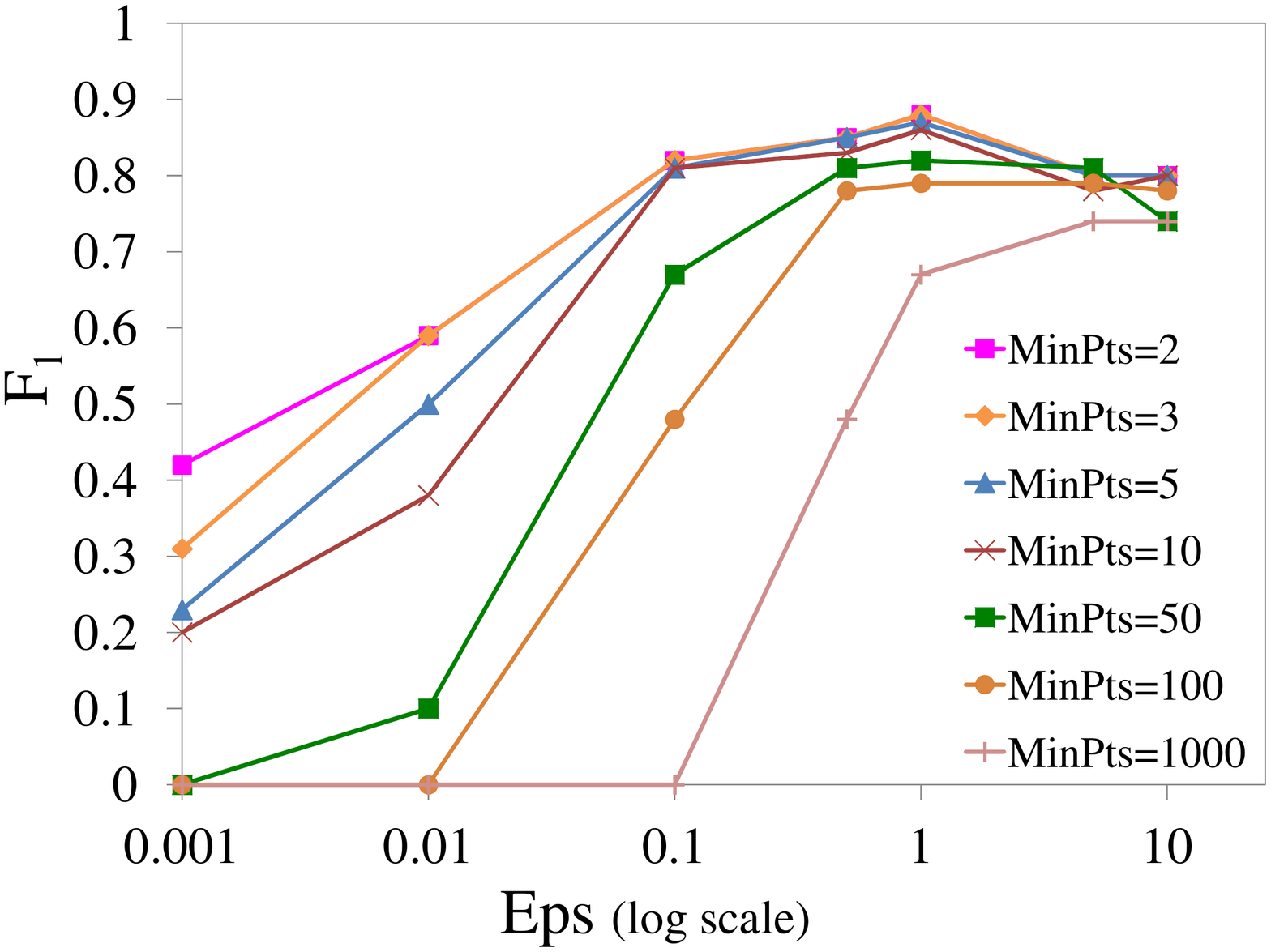}
  	\label{Expr:pigeon-hf-ff-dbscan-eps-impact}
  }


  \subfloat[\ALGFLC\ALG: \FOne\ vs. fuzziness ($\MF$).] {
  	\includegraphics[clip=true,trim=60 88 100 104,height=\HEIGHT, width=\WIDTH]{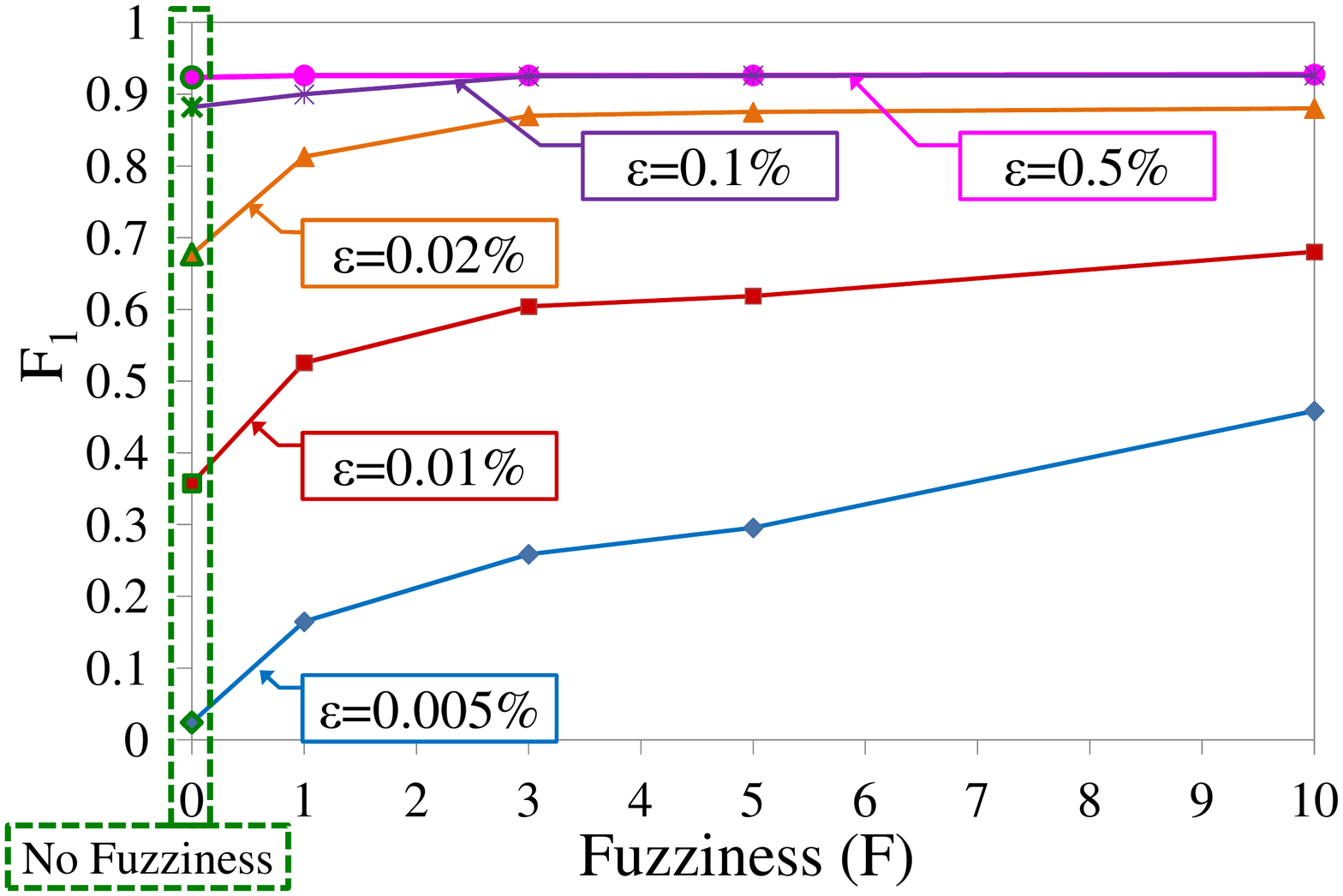}
	\label{Expr:pigeon-hf-ff-fuzz-impact}
  }
  \subfloat[\ALGDBSCAN\ALG: \FOne\ vs. density ($MinPts$).] {
  	\includegraphics[clip=true,trim=64 58 88 79,height=\HEIGHT, width=\WIDTH]{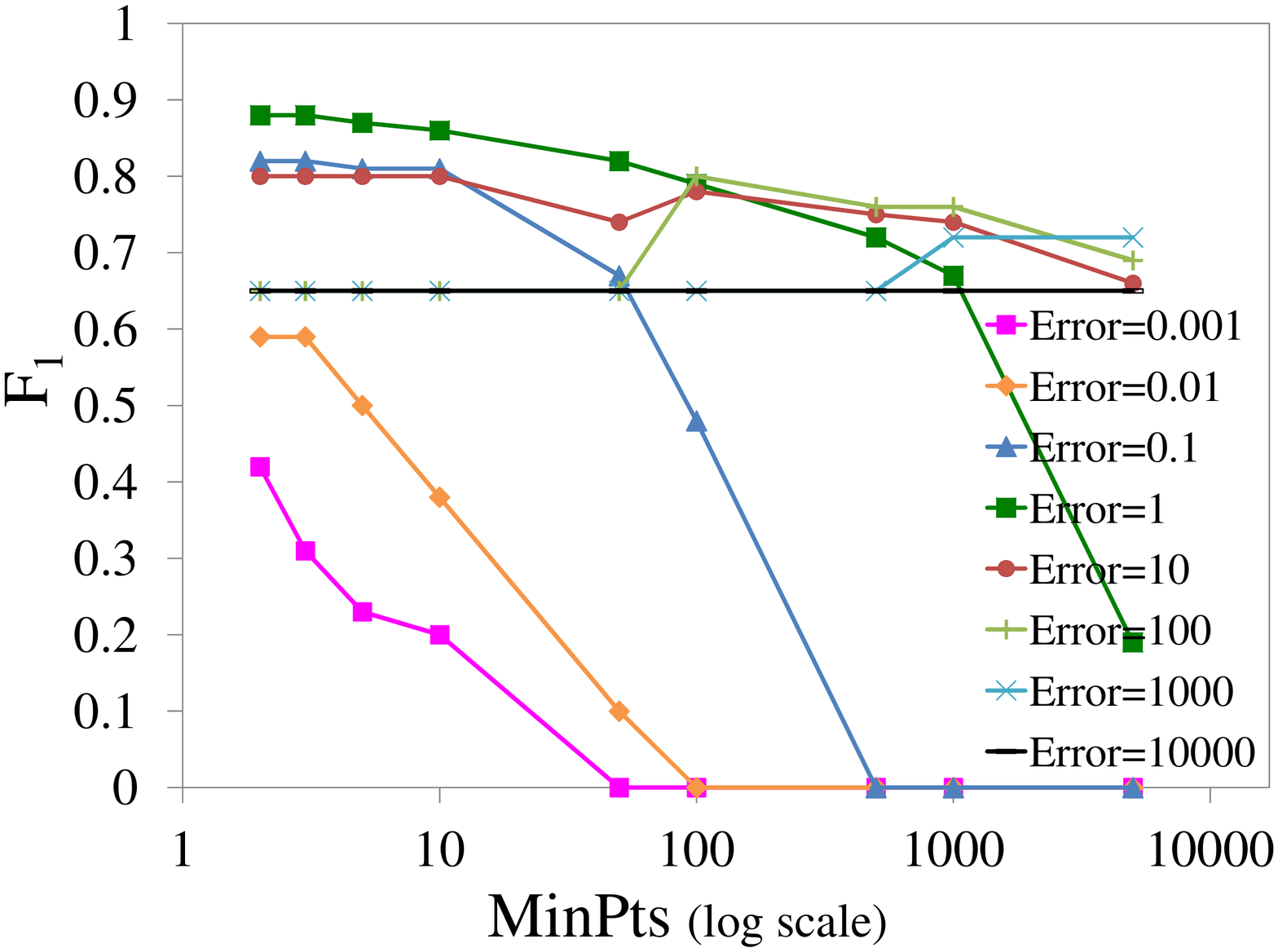}
  	\label{Expr:pigeon-hf-ff-dbscan-minpts-impact}
  }

  \caption{
    \FScore\ comparison of the \ALGFLC\ and \ALGDBSCAN\ algorithms.
    On the left, \FIGURE~\ref{Expr:pigeon-hf-ff-err-impact} and \FIGURE~\ref{Expr:pigeon-hf-ff-fuzz-impact}
    depict the \FScore\ as obtained by the \ALGFLC\ algorithm as a function of the error ($\SW$) and fuzziness ($\MF$), respectively.
    On the right, \FIGURE~\ref{Expr:pigeon-hf-ff-dbscan-eps-impact} and \FIGURE~\ref{Expr:pigeon-hf-ff-dbscan-minpts-impact}
    depict the \FScore\ as obtained by the \ALGDBSCAN\ algorithm as a function of the distance ($Eps$) and density ($MinPts$), respectively.
  }
  \label{Expr:pigeon-hf-ff}
\end{figure}
\FIGURE~\ref{Expr:pigeon-hf-ff-err-impact} and \ref{Expr:pigeon-hf-ff-fuzz-impact} depict the \FScore\ of the \ALGFLC\ algorithm
as a function of the error ($\SW$) and fuzziness ($\MF$), respectively.
$\MF$=0 is in fact the case of mining lagged clusters with \emph{no} fuzziness.
As expected, any increase in $\SW$ or $\MF$ results in an increase in the \FOne\ score
as more data points are reachable from the cluster's seed.
An important finding obtained from the figures is that for relatively low errors,
a significant increase in the \FScore\ is recorded when fuzziness is used.
For example, for $\SW$=0.005\%, we obtain for $\MF$=\{0, 1, 3, 5, 10\}
a score of \FOne=\{0.024, 0.164, 0.259, 0.295, 0.458\},
which reflects an increase of a factor of \textbf{$\times$}=\{1, 7, 10, 12, 19\}, respectively.
Figures~\ref{Expr:pigeon-hf-ff-dbscan-eps-impact} and \ref{Expr:pigeon-hf-ff-dbscan-minpts-impact}
depict the equivalent performance of \ALGDBSCAN\ for the same settings (i.e., \FScore\ as a function of the distance ($Eps$) and density ($MinPts$), respectively) used for generating
figures~\ref{Expr:pigeon-hf-ff-err-impact} and \ref{Expr:pigeon-hf-ff-fuzz-impact}.
Comparison of the \ALGFLC\ and \ALGDBSCAN\ algorithms reveals the stability
of setting the \ALGFLC\ parameters vs. the sensitivity of configuring the \ALGDBSCAN\ parameters
(surveyed in \cite{berkhin2006survey,han2001spatial}).
In addition, even when considering the best configuration for the \ALGDBSCAN\ algorithm,
the \ALGFLC\ algorithm still outperforms the best \FScore\ achieved with \ALGDBSCAN.
The comparison of \FIGURE~\ref{Expr:pigeon-hf-ff-err-impact} and \FIGURE~\ref{Expr:pigeon-hf-ff-dbscan-eps-impact}
reveals the difference in the error (distance) behavior between the \ALGFLC\ and \ALGDBSCAN\ algorithms, respectively
(note:
the \ALGDBSCAN\ uses an Euclidean distance measure ($L_2$ norm),
while the \ALGFLC\ uses the Manhattan distance measure ($L_1$ norm)).
While the \ALGFLC\ algorithm maintains a high \FScore\ as the error increases,
the \ALGDBSCAN\ algorithm results in mining futile clusters of \FOne=0.66 containing the entire dataset
(the datasets used contain two classes with an equal number of members. Therefore, a cluster containing the entire dataset,
will have a recall=1.0, precision=0.5 and thus \FOne=0.66). 
The comparison of \FIGURE~\ref{Expr:pigeon-hf-ff-fuzz-impact} and \FIGURE~\ref{Expr:pigeon-hf-ff-dbscan-minpts-impact} reveals that
while the fuzziness parameter of the \ALGFLC\ algorithm steadily increases the achieved \FScore,
the effect of the density parameter of the \ALGDBSCAN\ algorithm is not conclusive and depends on the adjacent error value.

The latter finding strengthens the importance and necessity of the fuzzy lagged model.
The mining of \textit{non-fuzzy} {\LC}s achieves a substantially lower \FScore\ in comparison to the mining of {\FLC}s.
Indeed, an increase in the \FOne\ measure can also be achieved by increasing the error $\SW$;
however, this is dangerous as any increase in $\SW$ substantially increases the chance of mining artifacts.
In addition, the increase in error may not always achieve a high \FScore.
For example, datasets with large spatial distances between the data points would require an error so large that it might cover the entire dataset, which in turn results in futile clusters.
On the other hand, the use of fuzziness as part of the model, enables mining accurate and coherent clusters without increasing the allowable error.
In addition, as illustrated in \FIGURE~\ref{Expr:pigeon-hf-ff},
a score close to $1$ for \FOne\ can be
obtained even when considering moderate fuzziness.

\subsubsection{\HF\ Dataset} \label{subsubsec:Expr:birds:hf}

In order to examine the algorithm's ability to properly classify each pigeon to its flock,
we merged all four \hf\ datasets resulting in a matrix of size $[37 \times 13892]$,
comprising 13892 GPS readings of 37 pigeons.

\label{subsubsec:Expr:birds:hf:humans}
To demonstrate the difficulty of clustering the dataset into flocks (i.e., specifying how many flocks are present)
we conducted the following experiment.
We asked $25$ people, of differing sex, age, occupation and nationality,
to specify how many flocks they could identify in the dataset.
For that purpose, we enabled them to use Google Earth to view the pigeons'
trajectories (see \FIGURE~\ref{Expr:pigeon-hf-human} for sample snapshots of the user view).
The subjects could use all functionalities within Google Earth
(e.g., view the trajectories from different angles, enlarge, and so on), and were given $5$-minute to reach an answer.
The average answer was $6.5$, with a standard deviation of $3.2$. Only $16\%$ of the subjects gave the correct answer
(i.e., $4$ flocks).
The results indicate that the dataset cannot be trivially mined.
\def \FHeight  {4.0cm}
\begin{figure*}
  \centering
  \subfloat[] {
  	\includegraphics[clip=true,trim=56 90 81 102,height=\FHeight]{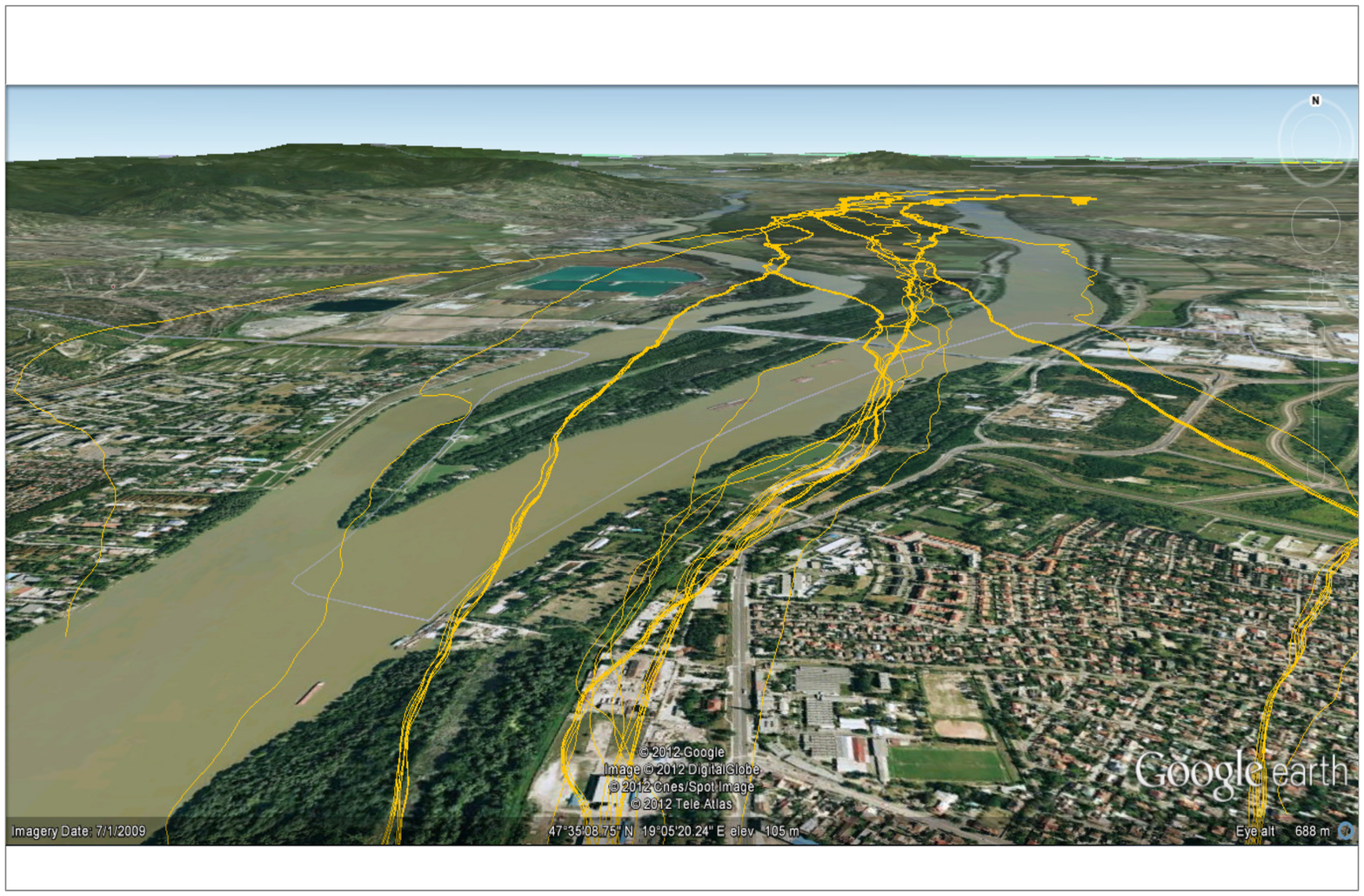}
  }
  \subfloat[] {
  	\includegraphics[clip=true,trim=56 90 80 102,height=\FHeight]{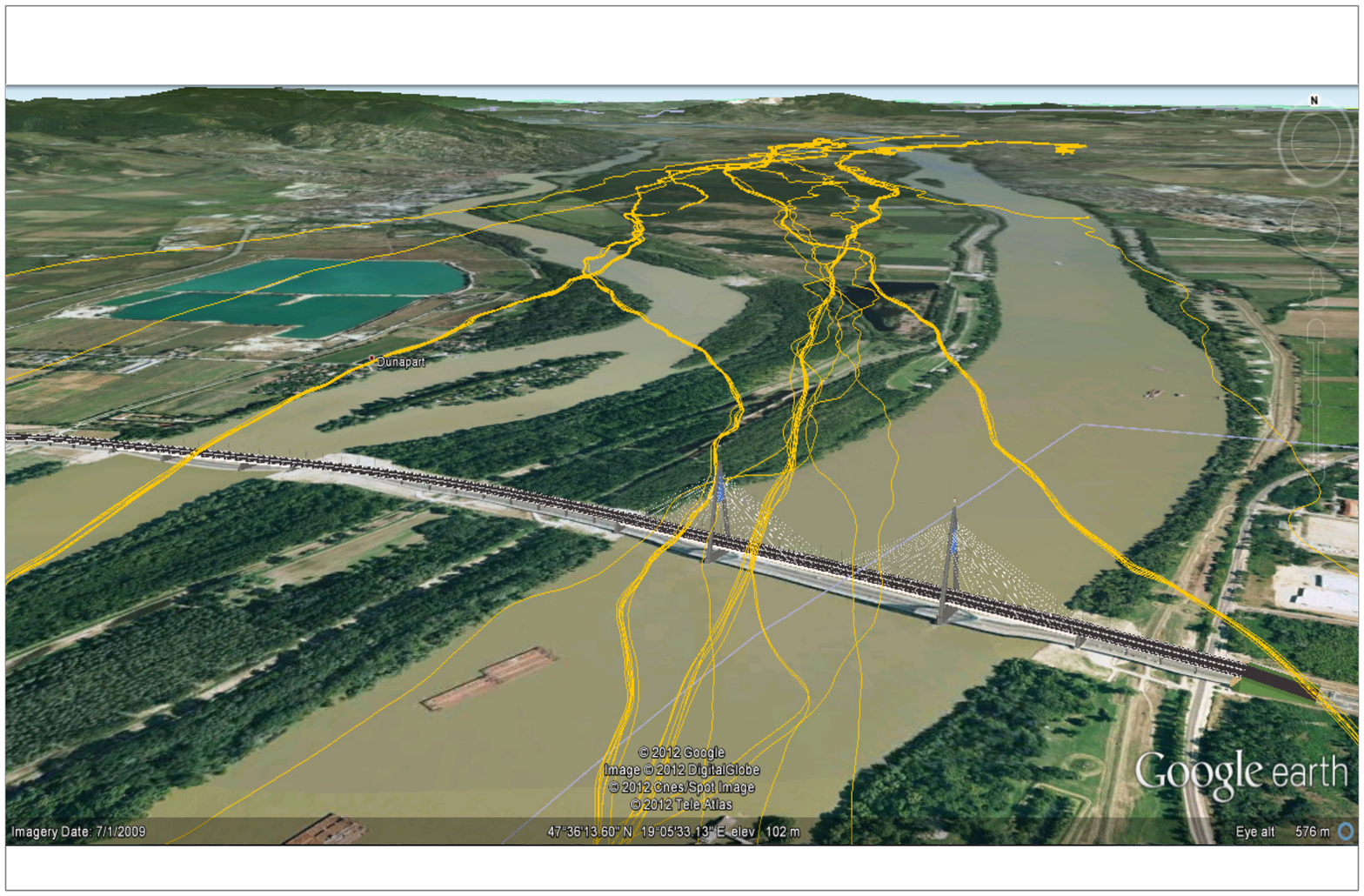}
  }
  \caption{
    Snapshots of a user view of the \hf\ dataset.
    The user was asked to specify the number of flocks (the correct answer is four).
  }
	\label{Expr:pigeon-hf-human}
\end{figure*}

%
%
\def \HEIGHT    {0.2\textheight} 
\def \WIDTH	    {0.48\textwidth} 
\begin{figure*}
	\centering

    \subfloat[Entropy vs. fuzziness.]
    {
	   \includegraphics[clip=true,trim=64 62 78 70,height=\HEIGHT, width=\WIDTH]{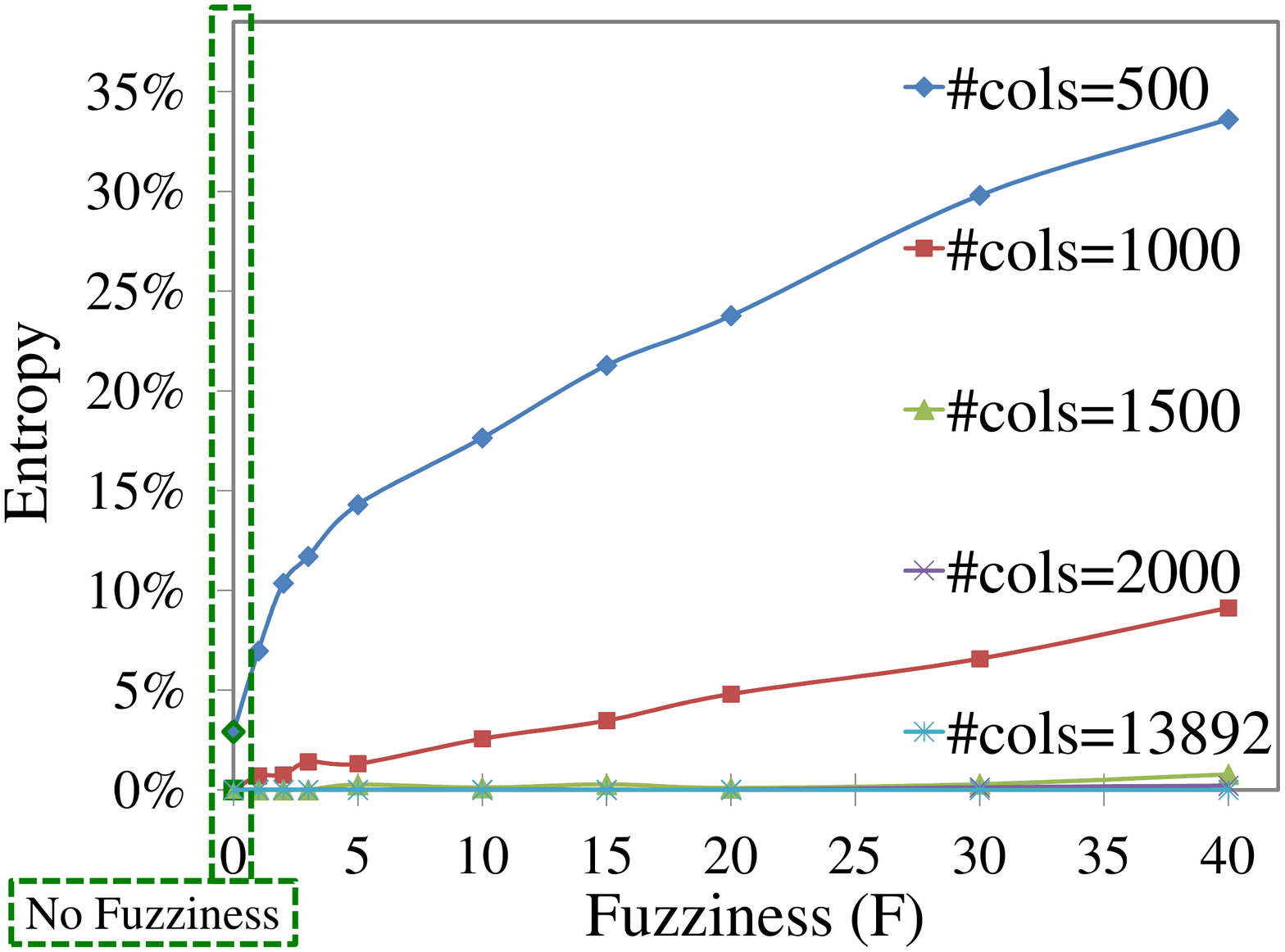}
       \label{Expr:pigeon-hf-confusions:entropy}
    }
    \subfloat[Inter-flock vs. fuzziness.]
    {
	   \includegraphics[clip=true,trim=70 62 78 70,height=\HEIGHT, width=\WIDTH]{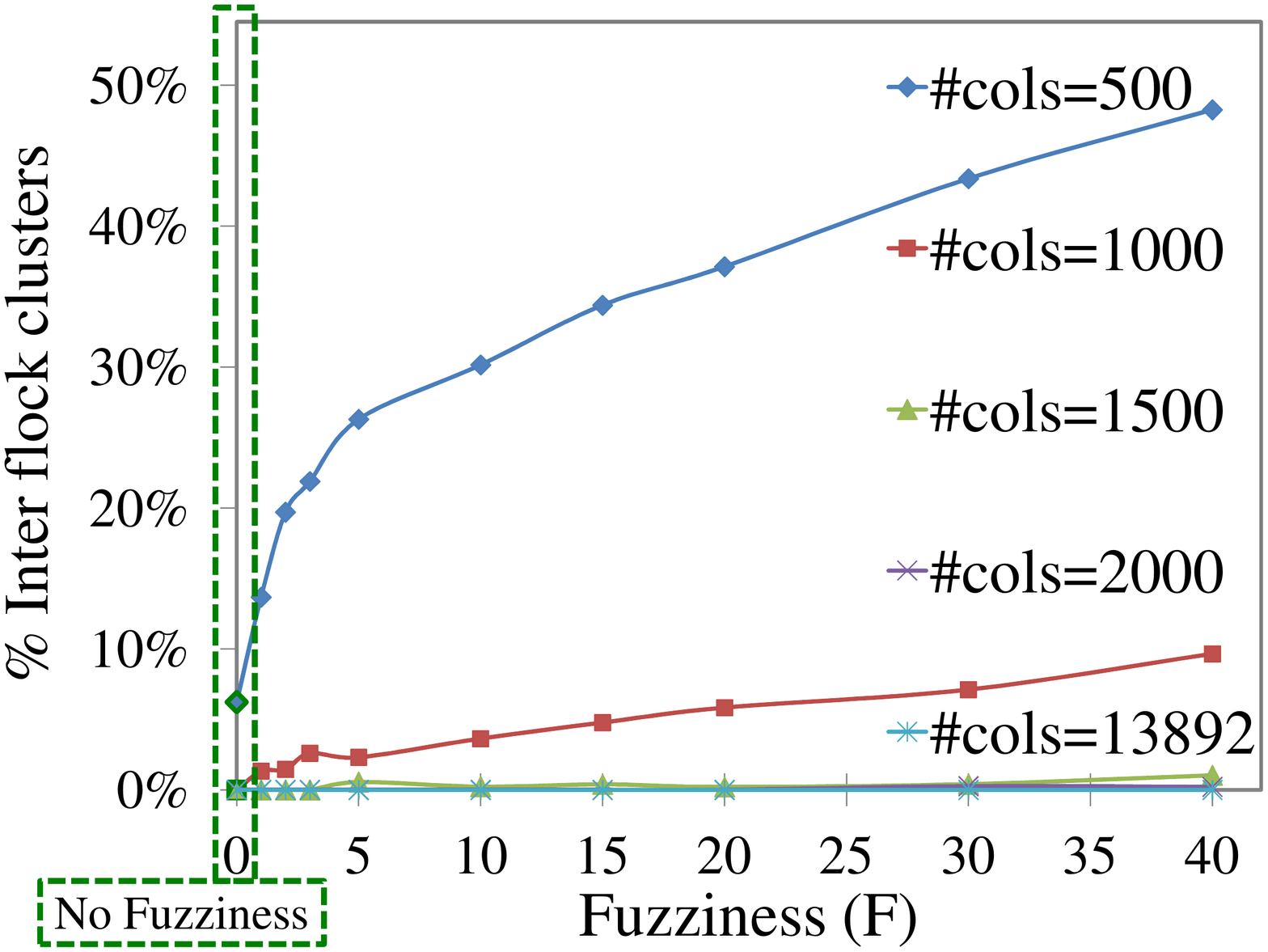}
       \label{Expr:pigeon-hf-confusions:inter-flock}
    }
	\caption{
        Confusions of the mined clusters.
        \FIGURE~\ref{Expr:pigeon-hf-confusions:entropy} presents the entropy,
        while \FIGURE~\ref{Expr:pigeon-hf-confusions:inter-flock} presents the
		percentage of inter-flock clusters, i.e., clusters containing pigeons of different flocks.
		The lower the percentage, the higher the accuracy of the cluster.
        \REM {
		The probability of an inter-flock cluster
		for small number of columns (i.e., less than 1000, which is 7\% of the dataset columns)
		is high.
		Due to the interleaving of the pigeons' trajectories, using a small number of columns does not supply enough
		data to discriminate between pigeons belonging to different flocks.
		On the other hand, using a large enough dataset (i.e., more than 10\% of the dataset columns),
		results in an insignificant probability of mining an inter-flock cluster,
		even with respect to a growing fuzziness.
        }
	}
	\label{Expr:pigeon-hf-confusions}
\end{figure*}
%
To examine how many flocks the {\FLCA} algorithm would specify, we conducted the following experiment.
We ran the {\FLCA} algorithm with
an error set to $\SW$=0.005\% (following \FIGURE~\ref{Expr:pigeon-hf-ff-fuzz-impact},
this error setting would enable the examination of the fuzziness effect)
and various values of fuzziness $\MF$$\in$\{0, 1, 2, 3, 5, 10, 15, 20, 30, 40\} for 20,000 trials.
In addition, we wished to examine whether the use of partial knowledge (i.e., partial dataset),
which necessarily reduces the overall mining run-time,
would preserve the quality of the mining results.
To do so,
each of the above settings was run on various datasets comprising
4\%, 7\%, 11\%, 14\% and 100\% of the dataset's columns (500, 1000, 1500, 2000 and 13892 of dataset's columns, respectively).
%
\FIGURE~\ref{Expr:pigeon-hf-confusions} depicts the accuracy of the mined {\FLC}s.
\FIGURE~\ref{Expr:pigeon-hf-confusions:entropy} presents the entropy of the mined clusters
as a function of the fuzziness $\MF$ and the number of columns.
The entropy of a cluster $C$ is computed as:
$H(C) = - \sum\nolimits_{i=1}^k p(i|C) \cdot \log(p(i|C))$
for $k$ class labels in cluster $C$.\footnote{
Due to the fact that classes are generally of the same size (membership-wise), no problem of imbalanced biasing arises.
}
Ideally, a cluster $C$ should contain objects of only one class and thus, have a zero entropy.
For a set of clusters, we take the average entropy weighted by the number of objects per cluster.
For readability, we normalize the entropy to the range of 0\% to 100\% by dividing by the maximum entropy,
i.e., $H(C)/\log(k)$ \cite{assent2007dusc,sequeira2004schism}.
\FIGURE~\ref{Expr:pigeon-hf-confusions:inter-flock} presents the percentage of inter-flock clusters,
i.e., clusters which contain pigeons from different flocks, as a function of the fuzziness $\MF$ and the number of columns.
As the results demonstrate,
due to the interleaving of the pigeons' trajectories, using a small number of columns does not supply enough
data to discriminate between pigeons belonging to different flocks.
The probability of an inter-flock cluster for small number of columns (i.e., less than 1000, which is 7\% of the dataset columns)
is high.
On the other hand, when using a large enough number of columns
(i.e., more than 10\% of the dataset columns)
the probability of mining an inter-flock cluster is insignificant, even with respect to a growing fuzziness.
Although we do not claim the generality of this approach,
this latter finding is important in the aspect of run-time, as also the use of a partial dataset yields promising results.
Moreover, we can use a post-process procedure which merges clusters that share common objects
(i.e., clusters that had at least one pigeon in common were merged).
Thus, with a high probability, the merged clusters will represent the different flocks.
\FIGURE~\ref{Expr:pigeon-hf-nb-flock} presents the number of flocks (as yielded by the post-process stage)
as a function of the fuzziness $\MF$ and the number of columns.
\begin{figure}[htb]
	\centering
	\includegraphics[clip=true,trim=56 60 68 70,height=5cm]{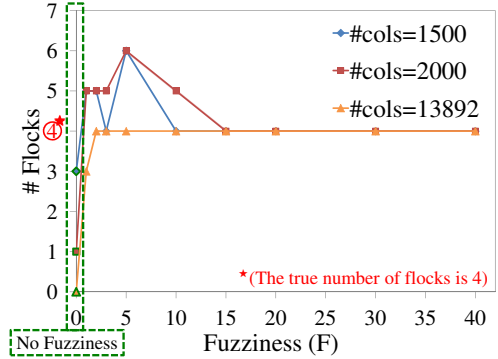}
	\caption{
	  Number of flocks (as yielded by the post-process stage)
	  as a function of the fuzziness and the number of columns.
		We notice that for $\MF$=0 (i.e., using the {\LC}ing model) we do not obtain the correct answer (i.e., $4$),
		regardless of the number of columns being used.
		On the other hand, when using the {\FLC}ing model,
		we quickly converge to the correct result, as the value of $\MF$ increases.
		Furthermore, we observe a quick convergence to the correct answer as the number of columns increases 		
		(e.g., for 13892 columns, we already obtain the correct result when using $\MF$=2).
	}
	\label{Expr:pigeon-hf-nb-flock}
\end{figure}
We notice that for $\MF$=0 (i.e., using the {\LC}ing model) we do not obtain the correct answer (i.e., four), regardless of the number of columns being used.  On the other hand, when using the {\FLC}ing model, we quickly converge to the correct result, as the value of $\MF$ increases.  Furthermore, we observe a quick convergence to the correct answer as the number of columns increases (e.g., for 13892 columns, we already obtain the correct result when using $\MF$=2).
As observed from the figure, the use of fuzziness led to the correct finding (four flocks) regardless of the number of columns used. Both the number of columns and the value of $\MF$ have a positive effect on the speed of convergence to the correct finding. In particular, with a large number of columns, only a very small level of fuzziness needs to be considered.
In contrast, using the \emph{non-}fuzzy model yielded on average only one group (one flock).
As we next show, this is due to poor mining results as reflected by the coverage of $\MF$=0 in \FIGURE~\ref{Expr:pigeon-hf-coverage}.
Worth mentioning in this context
is that human subjects gave an answer of (on average) 6.5 flocks.

A by-product of the post-process merging stage is the actual flock coverage, i.e., how many of the flock members have been covered by the mined clusters.
\FIGURE~\ref{Expr:pigeon-hf-coverage} depicts the
average flock coverage (over the numbers of columns$\in$\{1500, 2000, 13892\}) for the \hf\ dataset as a function of the fuzziness $\MF$.
\begin{figure}[htb]
	\centering
	\includegraphics[clip=true,trim=60 62 68 70,height=5cm]{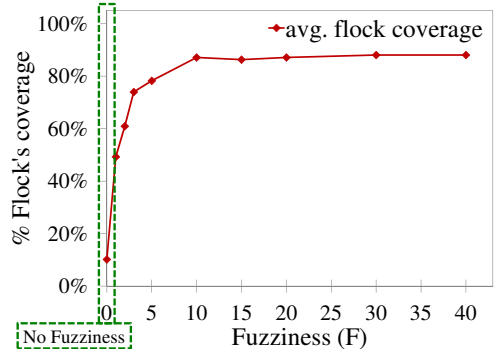}
	\caption{
		Average flock coverage vs. fuzziness for the \hf\ dataset
		(we present only the average as the graphs for the various settings were insignificantly different).
		Compared to the non-fuzzy model, using the fuzzy model notably improves the coverage results.
		Even with the use of $\MF$=1, a substantial improvement is achieved in comparison to traditional methods that
		do not take the fuzziness into account (i.e., the regular {\LC}ing approach).
		In this example, the improvement is from a coverage of 10\% (with $\MF$=0)
		to a coverage of 50\% with $\MF$=1.
		The use of higher fuzziness, e.g., $\MF$=40, brings the coverage up to $\sim$90\%.
	}
	\label{Expr:pigeon-hf-coverage}
\end{figure}
The results show a significant improvement in the accuracy and completeness of the mining process when using
the fuzzy model ($\MF$$\geq$1) in comparison to the non-fuzzy model ($\MF$=0).
Even the use of a fuzziness of a single column (i.e., $\MF$=1) has a notable impact of \textbf{$\times$}5 on the coverage.
The use of high $\MF$, provides coverage of $\sim$90\% of the flocks' members.

To compare the performance of the \ALGFLC\ algorithm, we ran the \ALGDBSCAN\ algorithm~\cite{ester1996density}
with various combinations of distance $Eps$$\in$[0.0001--10000] and density $MinPts$$\in$[2--500000].
\begin{figure}[htb]
	\centering
    \includegraphics[clip=true,trim=0 0 0 0 ,height=5.0cm]{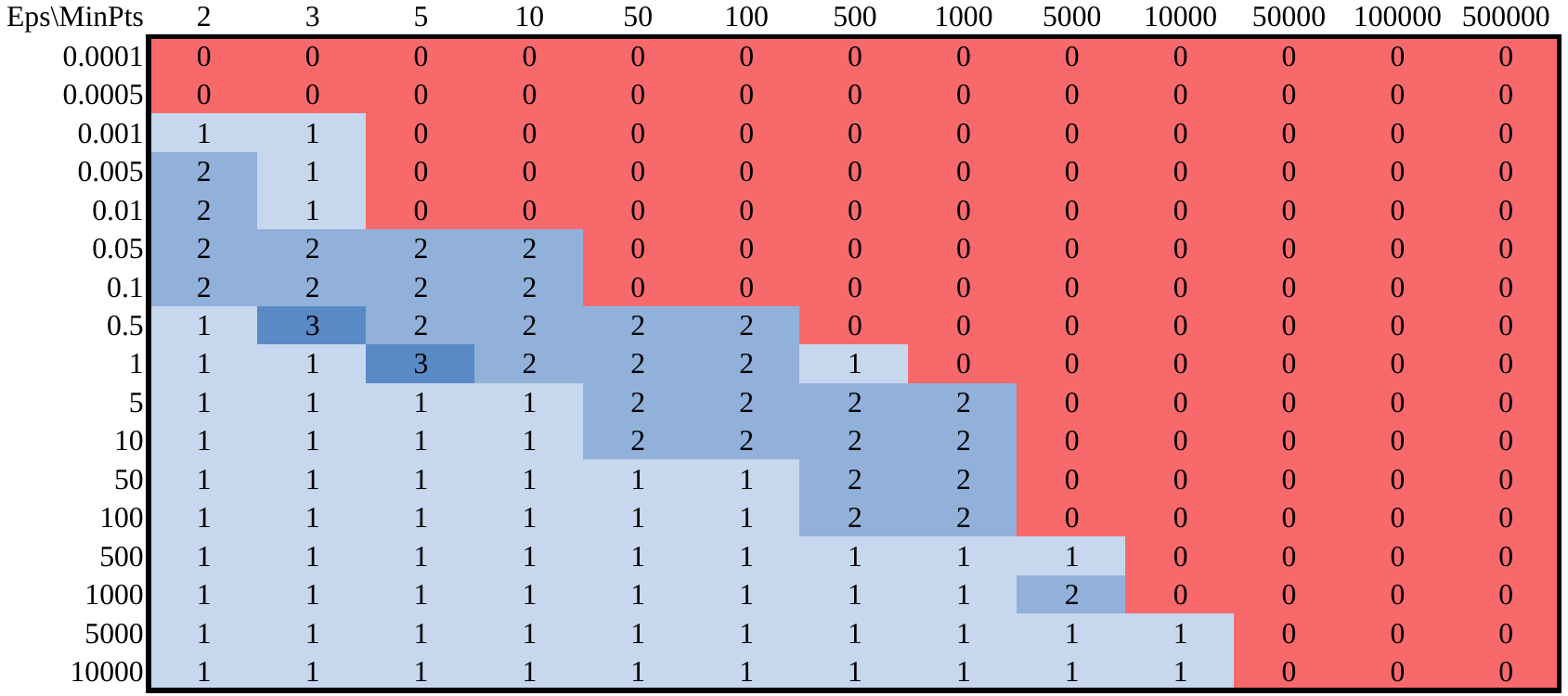}
	\caption{
	  Number of flocks as obtained by the \ALGDBSCAN\ algorithm~\cite{ester1996density}
	  as a function of the distance ($Eps$) and density ($MinPts$) as yielded by the post-process merging stage.
      Actual number of flocks is four.
	}
	\label{Expr:pigeon-hf-dbscan-nb-flock}
\end{figure}
\FIGURE~\ref{Expr:pigeon-hf-dbscan-nb-flock} depicts the
number of flocks as obtained by the \ALGDBSCAN\ algorithm
(as yielded by the post-process stage) as a function of the distance and density used.
For the vast majority of the settings, the \ALGDBSCAN\ algorithm either does not find any clusters (upper right portion -- red color)
or the clusters that it does find contain the entire dataset (lower left portion -- blue color) and are therefore futile.
Only for a negligible number of settings
does the \ALGDBSCAN\ algorithm manage to report three flocks, which even then is not the correct answer of four.
We note that the \ALGDBSCAN\ algorithm results as presented in \FIGURE~\ref{Expr:pigeon-hf-dbscan-nb-flock}
are based on a post-process merging stage, as the results without such a stage were considerably inferior.
On the same range of settings as in \FIGURE~\ref{Expr:pigeon-hf-dbscan-nb-flock},
the obtained average number of flocks without the post-processing stage was $87$ with a standard deviation of $463$,
where less than 0.5\% of the settings yield the correct answer of $4$.

To summarize, mining this dataset is not a trivial task.
This is evidenced by the poor classification results of the
human subjects, the \ALGDBSCAN\ algorithm and the non-fuzzy lagged method.
Furthermore, when using a non-fuzzy mining model, the obtained results are characterized by a low \FOne\ score and low coverage.
On the other hand, the {\FLCA} algorithm performed well in mining {\FLC}s on various trajectories.
It achieved a high \FOne\ score and high coverage, doing so with only a small number of artifact (inter-flock) clusters.
Although prior domain knowledge can be useful in configuring the miner's parameters, such knowledge is not mandatory.
\SUBSECTION~\ref{subsection:algorithm:Proofs} provides default values for setting
the parameters $|\DS|$ and $\LOOPS$ (see \THEOREM~\ref{proofs:theorem:d3} and \ref{proofs:theorem:rc}, respectively).
In order to choose an appropriate value of error, one can adopt any of the methods suggested
for the non-fuzzy {\LC}ing model \cite{shaham2011sc}, e.g., gradual increase, starting from a relatively small error.
Setting the minimum cluster dimensions (i.e., $\MI$ and $\MJ$) to relatively small values
would suffice for mining accurate clusters (e.g., $\MI$=2 and $\MJ$=10\% as depicted by \FIGURE~\ref{Expr:pigeon-hf-confusions}).
When setting the fuzziness, the use of small values            (e.g., $\MF$=1,2 as depicted by \FIGURE~\ref{Expr:pigeon-hf-nb-flock} and \ref{Expr:pigeon-hf-coverage}) already effects a considerable improvement in the clustering results over the non-fuzzy ones.
%
In conclusion, the significant improvement in mining presented by the {\FLCA} algorithm, in comparison to the non-fuzzy algorithm,
demonstrates the importance of including the fuzzy aspect in the model.
The {\FLCA} algorithm can thus be used as a classifier in this domain.


\section{Related Work} \label{sec:Related Work}

With the vast amount of routinely collected data, the need for clustering as a mining tool emerges in many fields:  biology, physics, economics and computer science are but a short list of domains with a wealth of research in this direction \cite{jain1999data,laxman2006survey}.
A typical mining problem is the extraction of patterns from a dataset, where the rows represent objects, the columns represent attributes and the data entries are the measurements of the objects over the attributes \cite{jain1999data,jiang2004cag}.

Simple mining techniques look for a fully dimensional cluster: a subset of the objects over \textit{all} attributes (or vice versa)
\cite{jiang2004cag,plerou1999universal,erdal2004tsa, shapira2009index}.
These techniques have several inherent vulnerabilities, e.g., difficulty in handling the common presence of irrelevant, noisy or missing attributes and inaccuracy due to the ``curse of dimensionality'' \cite{bellman1966dp,moise2009subspace,shi2011coid,beyer1999nearest,kriegel2009clustering}.
All these may be counter-productive as they increase background noise \cite{jiang2004cag,madeira2004bab}.

Cheng and Church \cite{cheng2000biclustering}, in their seminal work in the field of gene expression data, introduced a mining technique which focus on mining biclusters (also known as {\CCS}s or co-regulations): a \emph{subset} of the objects over a \emph{subset} of the attributes. Their approach was followed by many researchers
(see surveys by \cite{madeira2004bab,tanay2005bas,moise2009subspace,berkhin2006survey,jiang2004cag,kriegel2009clustering}),
using various models (additive vs. multiplicity, axis alignment, rows over columns preferment, cluster scoring function, overlapping, etc.), and applying various algorithmic strategies:
greedy \cite{cheng2000biclustering,ayadi2011bicfinder},
divide-and-conquer \cite{hartigan1972direct},
projected clustering \cite{lonardi2006fbr,procopiuc2002mca},
exhaustive enumeration \cite{tanay2002dss},
spectral analysis \cite{kluger2003sbm,takacs2010spectral},
CTWC \cite{getz2000ctw},
bayesian networks \cite{barash2002context},
etc.

Due to the importance of datasets having a temporal nature (i.e., sequences of time series), specific efforts have been directed at utilizing the continuous nature of time as a natural order \cite{jiang2003iec,barjoseph2002naa,jiang2004mcg,mollerlevet2005cos}, surveyed by \cite{roddick2002survey,warren2005clustering}.
%
In particular, some have considered the delay (lag) between the object's behaviors and suggested different approaches for mining {\LC}s. These include dynamic-programming and hierarchical-merging (with pruning) \cite{yin2007mining,wang2010efficiently}, polynomial time Monte-Carlo strategies to mine {\LC}s which encompass the optimal {\LC} \cite{shaham2011sc}, and the reduction to some finite alphabet \cite{gonccalves2013heuristic}.
Despite the efficiency of these approaches in mining {\LC}s, they become ineffective when the lagged pattern is fuzzy.
This is due to their underlying assumption of non-noisy (i.e., fixed) lags.

Existing methods that are inherently designed for mining {\FLC}s can be categorized into three types, each imposing a different design limitation.
%
The first is a group of methods designed for mining \emph{pairs} of sequences using variants of the edit distance measures,
such as the Longest Common Sub Sequence (LCSS) measure \cite{vlachos2002discovering}
and others \cite{yi1998efficient,chen2005robust,chen2004marriage},
surveyed in \cite{ishikawa2010data}. 
However, post-processing merging requires a combinatorial solution which is both time consuming and heavily dependent on the closeness of the merit function \cite{jiang2004cag,shapira2009index}.
Furthermore, pairs of objects might lack the transitivity characteristic, e.g., two stocks may appear to be correlated, while in fact the correlation is to the index which dominates them in volatile trading days \cite{shapira2009index}. Investments based on this in nonvolatile times may lead to poor results.
Moreover,
pairs might simply be merged as a mere aggregation of noise and not due to some hidden regulatory mechanism \cite{beyer1999nearest,shaham2011sc}.
%
The second type uses a space reduction approach to some finite alphabet \cite{girardin2009quantifying,giannotti2007trajectory,alvares2007model,pelekis2010unsupervised,pelekis2010clustering,pelekis2009clustering,ji2005identifying}, e.g., each trajectory coordinate is approximated to a grid cell.
The main limitation of such methods is the reduction magnitude.
On one hand, coarse abstraction using a small alphabet may lead to greater errors and finer clusters being missed. On the other hand, using a large alphabet will have
a dramatic influence on the run-time as it is exponentially dependent on the alphabet size.
%
Finally, there are methods that assume sequentiality of the cluster's columns,
e.g., flock mining \cite{lee2007trajectory,benkert2008reporting}.

\label{Related-Work:density-based}
A popular technique for mining clusters of trajectories is the \emph{density-based} approach.
This approach uses the spatial closeness between data points to associate them into clusters \cite{gudmundsson2007efficient,hwang2005mining,pelekis2009clustering,al2007dimensionality}.
A well known representative of this technique is the DBSCAN algorithm \cite{ester1996density}
(followed by derivative algorithms of a non-temporal
\cite{palma2008clustering,ankerst1999optics,ng1994efficient,zhou2007discovering,ma2004adaptive,tang2011exploring}
and temporal
\cite{birant2007st,sander1998density,kisilevich2010event,lee2007trajectory,tang2011exploring}
nature,
surveyed in
\cite{berkhin2006survey,han2001spatial}).
The disadvantage of the density-based approach is its sensitivity to noise (e.g., signal distortion), outliers (e.g., erroneous GPS measurements), missing values (in real-life, devices may be voluntarily disconnected by their owners, or be subject to machine failures or lost signal),
related objects which are spatially distant (e.g., a roaming group where members are far apart from each other) and
crossing trajectories of unrelated objects (e.g., trajectories of different groups interleave).
These algorithms will find the data used in \SUBSECTION~\ref{subsec:Expr:birds} challenging
as trajectories of different groups interleave.
This may cause clusters to contain inter-group trajectories, and thus, fail classification.

We note that all the works cited above were also unable to find substantial previous reference to the \FLCP\
and that state-of-the art algorithms for this problem are either non-fuzzy \cite{wang2010efficiently,shaham2011sc}
or limited to data points which are spatially close \cite{ester1996density}.

On the application side, the last few years have witnessed an increasing interest in {\FLC}ing.
This is attributed to the dramatic increase in location-aware devices (e.g., cellular, GPS, RFID).
Such devices leave behind \ST\ electronic trails.
A dataset of such trajectories is of a \textit{fuzzy lagged} nature.
Commercial services such as Foursquare, Google Latitude, Microsoft GeoLife, and Facebook Places, use such data to maintain location-based social networks (LBSN), later used for personal marketing purposes.
Another use of such data is to extract patterns
(see surveys in \cite{koperski1996spatial,han2001spatial,antunes2001temporal})
which may suggest Places Of Interest (POI) \cite{palma2008clustering,kisilevich2010event,kisilevich2010novel,girardin2009quantifying,girardin2008digital}.
This is mostly used for the purposes of
tourism \cite{girardin2008leveraging,girardin2008digital,girardin2009quantifying,palma2008clustering,asakura2007analysis},
urban planning \cite{girardin2008digital,girardin2009quantifying,gudmundsson2008movement},
crowd control \cite{girardin2009quantifying,gudmundsson2008movement},
traffic management \cite{palma2008clustering,pelekis2010unsupervised,pelekis2010clustering,gudmundsson2008movement} 
and behavioral sciences \cite{forsyth2009group,wang2006efficient,lauw2005mining}. 

\section{Discussion, Conclusions and Future Work} \label{sec:Conclusions}

The importance of clustering is unquestionable and has been thoroughly discussed and demonstrated in cited prior work.
Similarly, the extensive {\CC} literature includes many examples of the benefit of mining {\CCS}s as opposed to traditional approaches.
The {\FLC} model generalizes the {\LC} model, enabling the inclusion of an additional important dimension, a \textbf{fuzzy aspect}, in the regulatory paradigm.
The results reported in the previous section
not only corroborate the algorithm's ability to efficiently mine relevant and accurate {\FLC}s,
but also illustrate the importance of including fuzziness in the lagged-pattern model.
With the fuzziness dimension, a significant improvement is achieved in both \emph{coverage} and \emph{\FOne} measures in comparison to using the regular {\LC}ing model.
One important strength of the new model relates to the chance of mining artifacts. 
In order to enlarge the dimensions of the mined clusters, 
traditional non-fuzzy methods tend to increase the error which in turn increases the risk of mining artifacts.
The fuzzy model provides the user with the capability of keeping a low level of error, while improving the achieved performance, without introducing artifacts.
%

As proved in \SUBSECTION~\ref{subsection:NP-completeness},
the complexity of mining {\FLC}s is {\NPC} for most
interesting optimality measures. 
Thus the importance of the algorithm presented in this paper lies in promising the probability of mining an optimal {\FLC} and a theoretical bound to the polynomial number of iterations it will take. In addition, the algorithm demonstrates a set of important capabilities such as handling noise, missing values, anti-correlations and overlapping patterns. Moreover, even if lagged clusters with no fuzziness at all need to be mined, 
the {\FLCA} algorithm has a better run-time in comparison to former algorithms inherently designed for such cases \cite{yin2007mining,wang2010efficiently} (including the Monte-Carlo based algorithms \cite{shaham2011sc}).
It is notable that due to the Monte-Carlo nature of the \FLCA\ algorithm, its iterations (and therefore, the mined clusters) are independent of each other.
The algorithm can thus be implemented to take advantage of parallel computing or special hardware in a straightforward manner.

The experiments using an artificial environment (reported in the previous section) reveal actual performance which is far better, in terms of accuracy and efficiency, than the theoretical bounds.
In addition, they supply default values for the various configurable parameters of the algorithm, releasing the user from this burden.
%
When used on real-life datasets, the {\FLCA} algorithm was demonstrated to mine precise, coherent and relevant {\FLC}s in a practicable run-time and with almost no artifacts.
This is in contrast to inferior results obtained by using a non-fuzzy model and
despite the fact the datasets were large, highly noisy, contained many missing values, and were rich in overlapping clusters.
In addition, the {\FLCA} algorithm presented classification capabilities which were superior to the ones presented by the non-fuzzy lagged model, those of human subjects and to the \ALGDBSCAN\ algorithm.
This encouraging result is important in the sense of model validation and suggests great potential for mining {\FLC}s in many other fields of science, business, technology and medicine.

As in the non-fuzzy lagged model, the ability of the {\FLCA} algorithm to mine {\LC}s offers important functionalities such as forecasting. However, when mining {\FLC}s, one may have to choose between possibly intersecting columns (see \SUBSECTION~\ref{algorithm:impl-note-Japaneses-Bridges-problem-solution}).
One can utilize the intersecting mechanism to place weights on the matrix 
columns so as to enable the mining of more ``recent'' clusters. 
It is reasonable to assume that the latest (up-to-date) columns will contribute more to the accuracy of the forecast than old and possibly irrelevant ones.
We believe there is far more that can be developed in this aspect in terms of future research.

	\bibliographystyle{abbrv}
	\bibliography{_NOT_SEND_FOR_SUBMISSION/Main}



\section*{Author Biographies}
\leavevmode

\def \WIDTH	    {60pt}  
\def \VSPACEA   {-10pt}  
\def \VSPACEB   {50pt}  
\def \VSPACEC   {60pt}  

\vbox{%
\begin{wrapfigure}{l}{80pt}
{\vspace*{\VSPACEA}\fbox{
\includegraphics[angle=0, clip=true,trim= 30 97 420 72 ,width=\WIDTH]{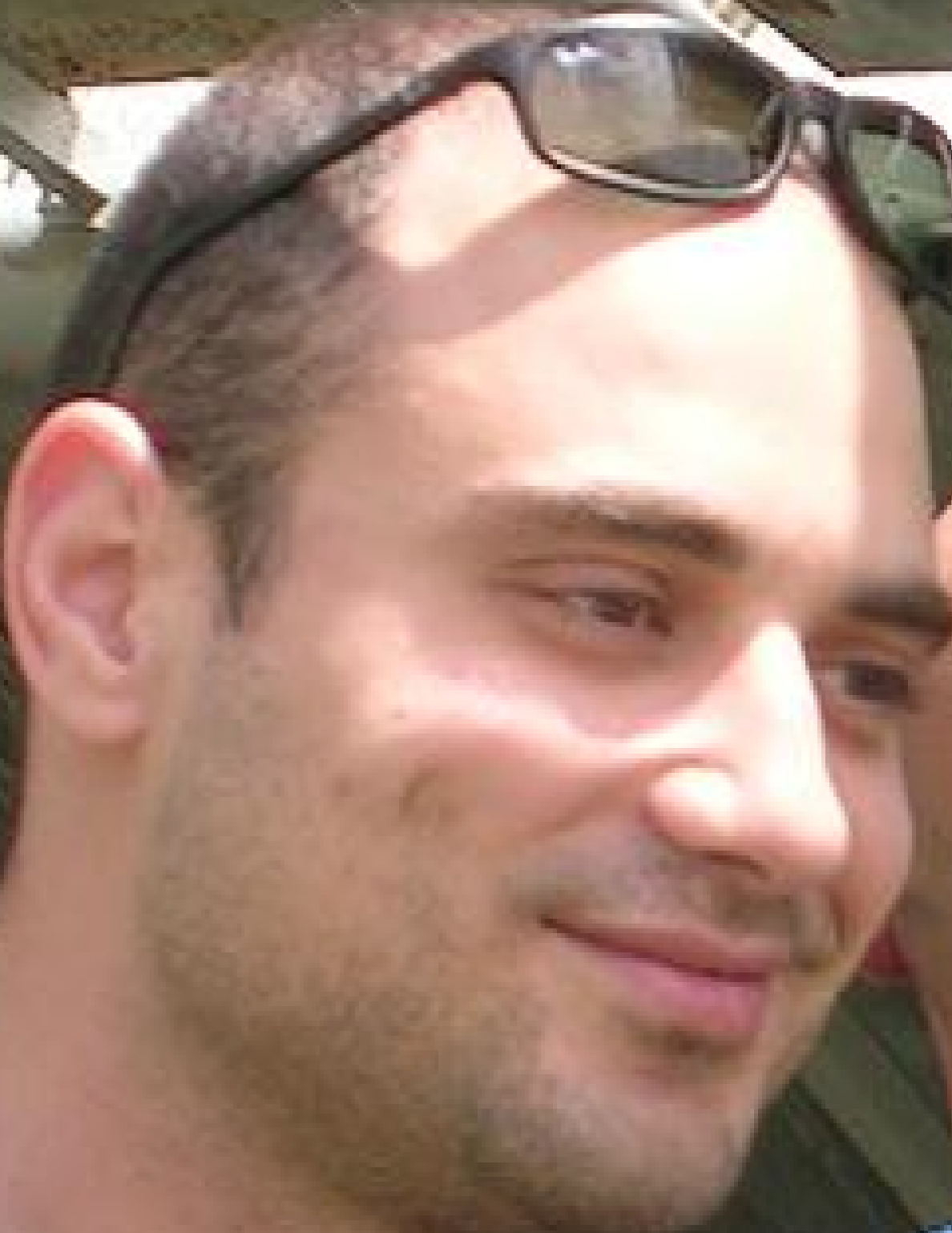}
}\vspace*{\VSPACEB}}%
\end{wrapfigure}
\noindent\small
{\bf Eran Shaham}
In 1998, Eran Shaham received his B.Sc. degree in Mathematics and Computer Science from Ben-Gurion University, Israel.
From 1999 to 2001, he worked at Parametric Technology Corporation (PTC), Israel.
In 2004, he received his M.Sc. degree in Computer Science from Ben-Gurion University, Israel.
From 2005 to 2008, he worked at the IBM Haifa Research Lab, Israel.
He is currently a Ph.D. student at the Department of Computer Science, Bar-Ilan University, Israel.
His research interests include data mining in general and its lagged aspects in particular.
\vadjust{\vspace{\VSPACEC}}
}

\vbox{%
\begin{wrapfigure}{l}{80pt}
{\vspace*{\VSPACEA}\fbox{
\includegraphics[angle=0, clip=true,trim= 140 196 480 200,width=\WIDTH]{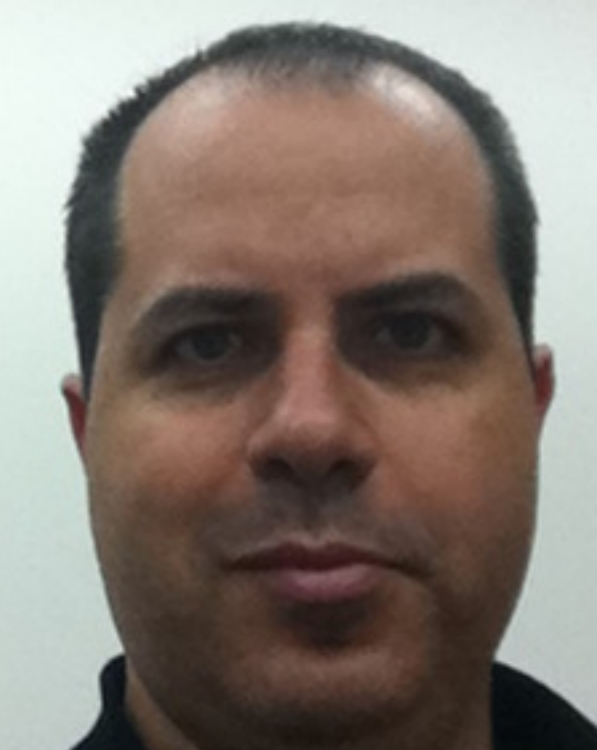}
}\vspace*{\VSPACEB}}%
\end{wrapfigure}
\noindent\small
{\bf David Sarne}
David Sarne is a senior lecturer in the Computer Science department in Bar-Ilan University, Israel.
He received a B.Sc., M.Sc., and a Ph.D. degree in Computer Science from Bar-Ilan University, Israel.
During 2005-2007 he was a post-doctoral fellow at Harvard University.
His research interests include economic search theory, market mechanisms for forming cooperation (mechanism design) and multi-agent systems.
\vadjust{\vspace{\VSPACEC}}
}

\vbox{%
\begin{wrapfigure}{l}{80pt}
{\vspace*{\VSPACEA}\fbox{
\includegraphics[angle=0, clip=true,trim= 37 117 435 72,width=\WIDTH]{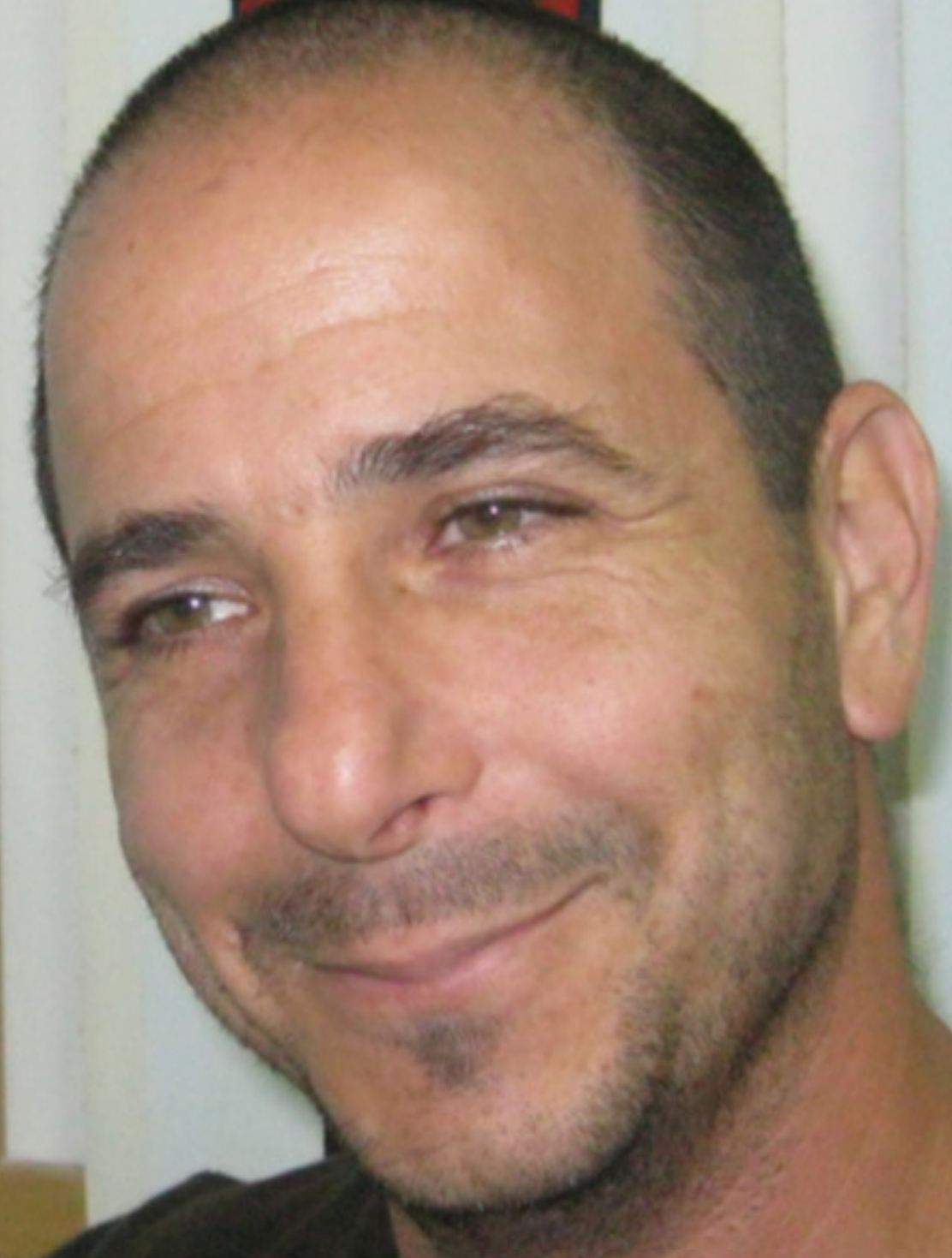}
}\vspace*{\VSPACEB}}%
\end{wrapfigure}
\noindent\small
{\bf Boaz Ben-Moshe}
Boaz Ben-Moshe is a faculty member in the Department of Computer in Ariel University, Israel. He received the B.Sc., M.Sc., and Ph.D. degrees in Computer Science from Ben-Gurion University, Israel. During 2004-2005 he was a post-doctoral fellow at Simon Fraser University, Vancouver, Canada.
His main research areas are: Computational Geometry and GIS algorithms. His research includes Geometric data compression, Optimization of wireless networks, Computing visibility graphs, and Vehicle routing problems.
In 2008 he has founded the Kinematics and Computational Geometry Laboratory with Dr. Nir Shvalb, see: \url{http://www.ariel.ac.il/sites/kcg}.
} 



\appendix \section*{}
\label{sec:Appendix}

%
\begin{table}
\renewcommand\thetable{}
\caption{
	Notations used and their meaning.
}
\label{appendix:notations}
\centering
\begin{tabular}{||l|l||} 
  \noalign{\vspace{2pt}} \cline{0-1} \noalign{\vspace{4pt}} 
	{Notation}&{Meaning} \\	
  \noalign{\vspace{4pt}} \cline{0-1} \noalign{\vspace{2pt}} \cline{0-1} \noalign{\vspace{4pt}}
	    $m$   & number of rows \\
			$n$   & number of columns \\
	    $X$   & real number matrix of size $m \times n$ \\
	    $I$   & a subset of the rows, i.e., $I \subseteq m$  \\
	    $T$   & the corresponding lags of the rows in $I$ ($|T|=|I|$)  \\
	    $J$   & a subset of the columns, i.e., $J \subseteq n$ \\
	    $\MF$   & maximal fuzziness degree  \\
	    $(I,T,J,\MF)$   & a {\FLC} of matrix $X$ \\
	    $\EF_{i,j}$   & the fuzzy alignment of object $i$ to sample $j$,\\
	          & i.e., $-\MF \leq \EF_{i,j} \leq \MF$, for all $i\in I$ and $j\in J$ \\
	    $G_i$   & a latent variable indicating object $i$'s regulation strength\\
	    $H_j$   & a latent variable indicating the regulatory intensity of sample $j$  \\
	    $\eta$   & relative error  \\
	    $A$   & $X$ logarithm transformation, i.e., $A_{i,j}=\log(X_{i,j})$  \\
	    $\varepsilon$   & $\eta$ logarithm transformation, i.e., $\varepsilon=\log(\eta)$  \\
	    $R_i$   & $G_i$ logarithm transformation, i.e., $R_i=\log(G_i)$  \\
	    $C_j$   & $H_j$ logarithm transformation, i.e., $C_j=\log(H_j)$  \\
	    $\TF(I,J)$   & objective function of a cluster \\
	    $\SWTF(I,J)$   & an error of a {\FLC} \\
	    $\MI$   & minimum number of the rows, expressed as a fraction of $m$ \\
	    $\MJ$   & minimum number of the columns, expressed as a fraction of $n$ \\
	    $\Dp$   & discriminating row ($\Dp \in I$) \\
	    $\Ds$   & discriminating column ($\Ds \in J$) \\
	    $\DS$   & discriminating column set ($\DS \subseteq J$) \\
	    $\SFZ$   & a subset of $\DS$ having zero fuzziness over all cluster's rows \\
	    $\LOOPS$   & number of iterations the {\FLCA} algorithm runs \\

  \noalign{\vspace{4pt}} \cline{0-1} 
\end{tabular}
\end{table}

	\correspond{Eran Shaham, Department of Computer Science,
							Bar-Ilan University, Ramat-Gan, 52900 Israel.
						 Email: \href{mailto:erans@macs.biu.ac.il}{erans@macs.biu.ac.il}.}
	\label{lastpage}	
\end{document}